\documentclass{article}

\usepackage{microtype}
\usepackage{graphicx}
\usepackage{subcaption}
\usepackage{booktabs} 

\usepackage{hyperref}
\hypersetup{
    colorlinks=true,
    linkcolor=red
}

\usepackage[preprint]{icml2026}

\usepackage{amsmath}
\usepackage{amssymb}
\usepackage{mathtools}
\usepackage{amsthm}
\usepackage{xcolor}

\usepackage{titletoc}
\usepackage{tocloft}
\usepackage{bbm}
\usepackage{bm}

\usepackage{multirow}
\usepackage[table]{xcolor}
\usepackage{adjustbox}
\usepackage{arydshln} 
\usepackage{pifont}
\usepackage{hyperref}
\usepackage{xcolor}

\definecolor{royalblue}{rgb}{0.21,0.49,0.74}

\definecolor{lightpurple}{HTML}{9999FF}

\usepackage[capitalize,noabbrev]{cleveref}

\theoremstyle{plain}
\newtheorem{theorem}{Theorem}[section]

\newtheorem{definition}[theorem]{Definition}
\newtheorem{assumption}[theorem]{Assumption}
\newtheorem{proposition}[theorem]{Proposition}

\theoremstyle{remark}
\newtheorem{remark}{Remark}[section]

\newcommand{\epsM}{\epsilon_{\mathrm{M}}}
\newcommand{\epsR}{\epsilon_{\mathrm{R}}}
\newcommand{\E}{\mathbb{E}}

\newcommand{\Prob}{\mathbb{P}}
\newcommand{\Correct}{\mathbb{I}}

\usepackage[textsize=tiny]{todonotes}

\icmltitlerunning{Knowing When to Answer: Reliable Audio-Visual Question Answering}

\begin{document}

\twocolumn[
  \icmltitle{Knowing When to Answer: Adaptive Confidence Refinement for \\ Reliable Audio-Visual Question Answering}



  \icmlsetsymbol{equal}{*}
  \icmlsetsymbol{equal2}{$\dagger$}

  \begin{icmlauthorlist}
    \icmlauthor{Dinh Phu Tran}{equal,yyy}
    \icmlauthor{Jihoon Jeong}{equal,comp,comp3}
    \icmlauthor{Saad Wazir}{yyy}
    \icmlauthor{Seongah Kim}{yyy}
    \icmlauthor{Thao Do}{yyy}
    \icmlauthor{Cem Subakan}{equal2,comp,comp3}
    \icmlauthor{Daeyoung Kim}{equal2,yyy}
  \end{icmlauthorlist}

  \icmlaffiliation{yyy}{School of Computing, KAIST, Republic of Korea}
  \icmlaffiliation{comp}{Laval University}
  \icmlaffiliation{comp3}{Mila-Quebec AI Institute}

  \icmlcorrespondingauthor{Dinh Phu Tran}{phutx2000@kaist.ac.kr}

  \icmlkeywords{Machine Learning, ICML}

  \vskip 0.3in
]



\printAffiliationsAndNotice{}  

\begin{abstract}

We present a formal problem formulation for \textit{Reliable} Audio-Visual Question Answering ($\mathcal{R}$-AVQA), where we prefer abstention over answering incorrectly. 
While recent AVQA models have high accuracy, their ability to identify when they are likely wrong and their consequent abstention from answering remain underexplored areas of research. To fill this gap, we explore several approaches and then propose Adaptive Confidence Refinement (ACR), a lightweight method to further enhance the performance of $\mathcal{R}$-AVQA. Our key insight is that the Maximum Softmax Probability (MSP) is Bayes-optimal only under strong calibration, a condition usually not met in deep neural networks, particularly in multimodal models. Instead of replacing MSP, our ACR maintains it as a primary confidence signal and applies input-adaptive residual corrections when MSP is deemed unreliable. ACR introduces two learned heads: i) a Residual Risk Head that predicts low-magnitude correctness residuals that MSP does not capture, and ii) a Confidence Gating Head to determine MSP trustworthiness. Our experiments and theoretical analysis show that ACR consistently outperforms existing methods on in- and out-of-disrtibution, and data bias settings across three different AVQA architectures, establishing a solid foundation for $\mathcal{R}$-AVQA task. The code and checkpoints will be available upon acceptance \href{https://github.com/PhuTran1005/R-AVQA}{at here}

\end{abstract}    
\section{Introduction}

\begin{figure}[t]
    \centering
    \includegraphics[width=1.0\linewidth]{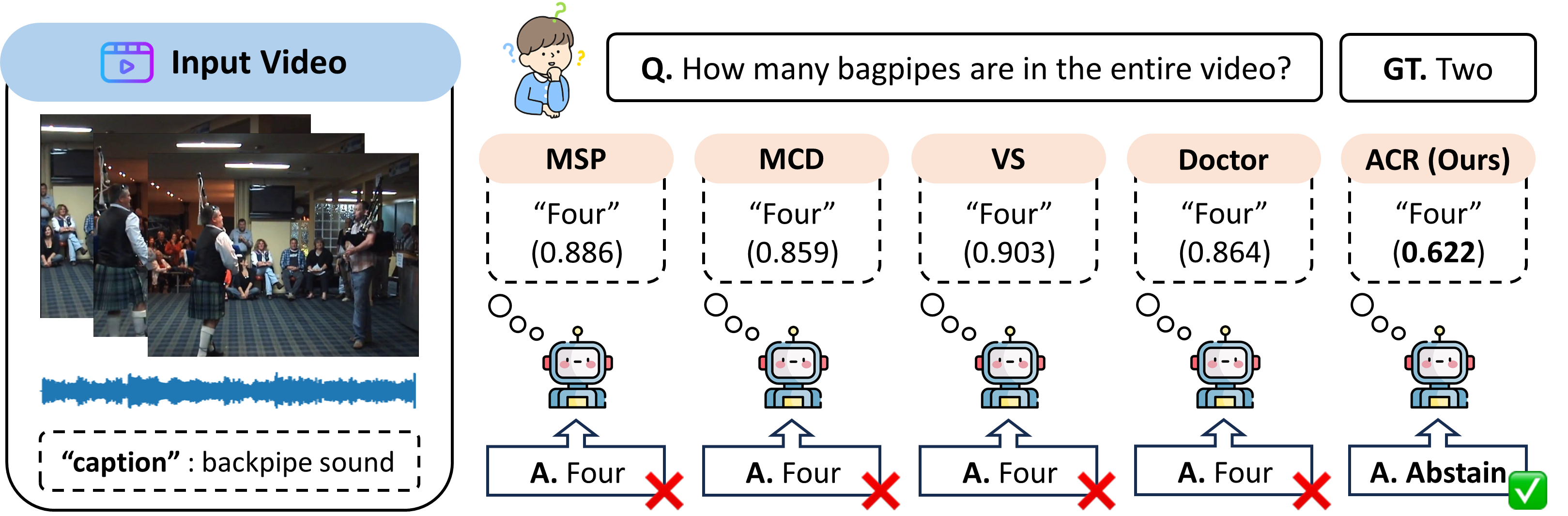}
    \caption{$\mathcal{R}$-AVQA requires knowing \emph{when to answer}. In an illustrative example from MUSIC-AVQA dataset, while standard confidence-based baselines (MSP, MCD, VS, and Doctor) produce incorrect answers with high confidence, our ACR effectively identifies unreliable predictions with a lower confidence score.}
    \label{fig:teaser}
    \vspace{-2mm}
\end{figure}

Audio-Visual Question Answering (AVQA) is a crucial multimodal reasoning task for developing assistants that help users, including those with sensory impairments, in daily activities. \cite{kumar2022deep, patel2025enhancing}. To be truly useful in these contexts, AVQA models must exhibit high reliability \cite{lutkenhoner2013predictive, asan2020artificial}.
Improving accuracy is crucial for reliable AVQA systems; however, accuracy alone is insufficient. All models have inherent limitations and are prone to making incorrect predictions. 
In practice, providing an incorrect answer can have significant consequences, particularly for users with sensory impairments, who cannot verify responses independently.

An approach to avoid giving incorrect answers is to allow the model to abstain from making a prediction, as outlined in the selective prediction studies \cite{geifman2017selective, geifman2019selectivenet, kamath-etal-2020-selective, xin2021art}. As shown in Fig. \ref{fig:teaser}, standard confidence-based baselines (MSP, MCD, VS, Doctor) produce incorrect answers with high confidence, indicating an overconfidence issue. In contrast, our ACR effectively lowers confidence levels to abstain from answering when the model lacks sufficient evidence, such as when counting bagpipes is uncertain, prompting it to respond with ``I don't know," similar to human decision-making.
Moreover, prior works on selective prediction have mainly focused on unimodal \cite{geifman2019selectivenet} or bimodal \cite{whitehead2022reliable, dancette2023improving} settings. In contrast, AVQA involves a trimodal interaction among audio, visual, and linguistic inputs. Each modality may be valid or in-distribution when considered independently, but their combination can be challenging, especially when audio and visual signals are weakly aligned or even conflicting, making $\mathcal{R}$-AVQA more difficult.

In this work, we formalize the concept of reliability in AVQA by framing it as a selective prediction problem \cite{geifman2017selective}, where a model must either provide an answer or choose not to respond. This framework requires two capabilities that remain underexplored in AVQA: 1) the ability to estimate predictive uncertainty accurately, and 2) the effective answer abstention strategies.
Under this framework, we show that existing AVQA models have substantial room for improvement in reliability. Using the widely adopted maximum softmax probability (MSP) \cite{geifman2017selective} method, we observe severely low coverage under strict low-risk constraints (abstain from too many questions to maintain low risk).
For instance, on the MUSIC-AVQA dataset, the state-of-the-art QA-TIGER method can answer only about \textbf{14\%} of questions to maintain a 1\% error rate, despite the overall AVQA accuracy of over \textbf{77\%}. Similar limitations are observed across other selective prediction methods, highlighting the limited practical utility of current AVQA models when reliability is critical.

Finally, we propose \emph{Adaptive Confidence Refinement (ACR)}, a learnable confidence estimation framework for $\mathcal{R}$-AVQA. To the best of our knowledge, this work is the first to propose a learnable confidence estimation scheme in a trimodal AVQA setting, explicitly designed to determine when a model should answer and when should abstain. Unlike prior work on selective prediction that relies on fixed heuristics or post-hoc confidence estimation, ACR learns to refine confidence scores of the MSP baseline by utilizing multimodal features and pre-softmax logits to capture uncertainty signals that MSP often overlooks, enabling more reliable confidence estimation under complex cross-modal interactions. We extensively evaluate ACR against widely used baselines, including MSP \cite{geifman2017selective}, Monte Carlo Dropout (MCD) \cite{gal2016dropout}, calibration-based methods (VS) \cite{guo2017calibration}, and Doctor \cite{granese2021doctor}.
Across three AVQA datasets and three representative backbone models, ACR consistently achieves state-of-the-art risk-coverage trade-offs, showing its effectiveness and robustness for developing reliable AVQA systems.

In summary, our main contributions are:
\begin{enumerate}
    \setlength{\itemsep}{0.001em}
    \item To the best of our knowledge, this is the first study to analyze and evaluate the reliability of AVQA models under a selective prediction framework.

    \item We explore various methods for incorporating abstention, then propose \emph{Adaptive Confidence Refinement (ACR)}, a learnable confidence estimation approach that integrates MSP with a learned residual risk signal through a simple linear fusion, modulated by an input-adaptive weighting mechanism. 


     \item We establish new benchmarks and strong baselines for the $\mathcal{R}$-AVQA task, and demonstrate that ACR consistently outperforms existing methods across three representative AVQA backbone models, both in- and under distribution shifts, including out-of-distribution generalization and data bias scenarios.

    
\end{enumerate}

\section{Related Work}
\vspace{-1mm}

\textbf{Audio-Visual Question Answering}. AVQA has garnered significant attention in recent years with various methods aimed at enhancing answer accuracy \cite{schwartz2019simple,yun2021pano,li2022learning,li2023progressive,li2024boosting,li2024object,kim2025question}. Recent methods \cite{li2024boosting, kim2025question} have shown strong results on several standard benchmarks, including MUSIC-AVQA \cite{li2022learning}, MUSIC-AVQA-R \cite{ma2024look}, and MUSIC-AVQA-v2.0 \cite{liu2024tackling} datasets. However, existing AVQA models have no explicit abstention mechanism and are designed to answer for every input question given audio-visual data.
To fill this gap, in this work we explore and formulate AVQA models through the lens of reliability by framing the task as a selective prediction problem and equipping several representative AVQA models \cite{li2022learning, li2024boosting, kim2025question} with the ability to abstain from providing an answer. Our findings reveal that, despite their high accuracy, recent AVQA models perform poorly in a reliability setting. This highlights substantial room for improving their reliability.

\textbf{Selective Prediction}. In the selective prediction \cite{geifman2017selective, liu2019deep, kamath-etal-2020-selective}, a selection function assigns a confidence score to each model prediction. A common choice for selective prediction is the maximum softmax probability (MSP) \cite{geifman2017selective}. An abstention decision is made by comparing this score to a predefined threshold: predictions with confidence below the threshold are rejected, while those above are accepted.
Selective prediction can be divided into external and integrated methods. External methods create a selector on top of a fixed predictor, such as auxiliary prediction head \cite{mielke2022reducing}, LoRA-based adaptations \cite{chen2023adaptation}, or auxiliary vision tools \cite{srinivasan2024selective}. Integrated methods jointly train predictor and selector during shared optimization phases, using mechanisms like SelectiveNet \cite{geifman2019selectivenet} or an abstention class \cite{liu2019deep}. Yet, joint training from scratch often leads to optimization instabilities requiring a specialized training scheme to optimize both objectives. Moreover, AVQA models typically use pretrained encoders for audio, video, and text that make these methods unsuitable and ineffective for $\mathcal{R}$-AVQA.
Unlike prior work, we do not replace MSP with a new confidence estimator. 
We propose an input-adaptive correction mechanism that focuses on refining MSP, which is a strong baseline for selective prediction \cite{jiang2021can}. 
More related work is given in Appendix \ref{section:A_2}.
\vspace{-1mm}
\section{Reliable Audio-Visual Question Answering}
\label{sec:methodology}

\begin{figure}[t]
    \centering
    \includegraphics[width=1.0\linewidth]{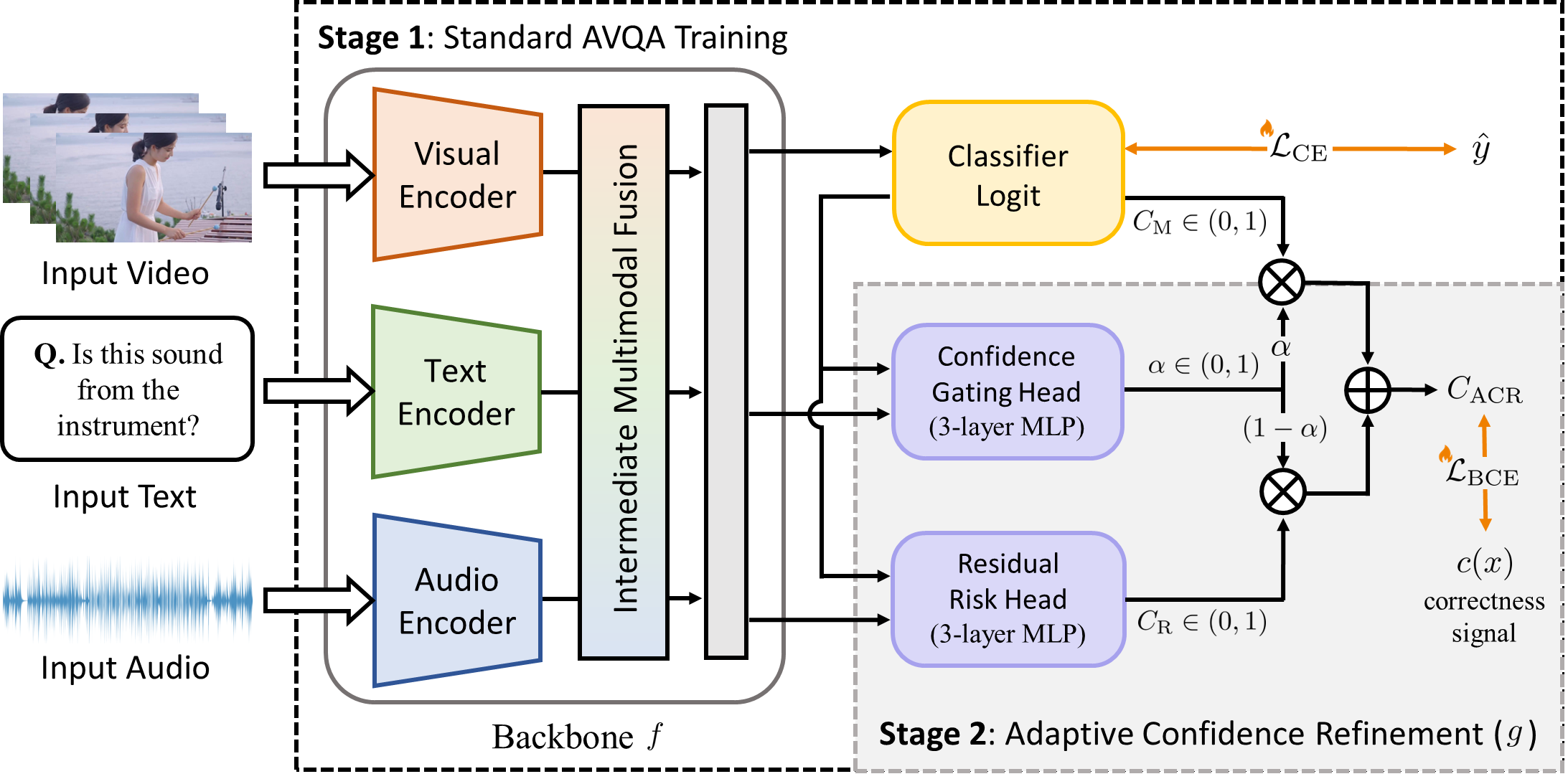}
    \caption{Overview of Adaptive Confidence Refinement framework for $\mathcal{R}$-AVQA. Stage 1 trains a standard AVQA model using cross-entropy loss. Stage 2 freezes the AVQA model and trains two additional heads (\textcolor{lightpurple}{purple blocks}) to achieve a more reliable input-adaptive confidence using binary correctness supervision.} 
    \label{fig:main_arch}
    \vspace{-4mm}
\end{figure}

\subsection{Problem Definition and Notation}

We extend the standard AVQA task to a \textit{reliable} setting, in which a model can either answer or abstain from answering. We call this setting as a \textit{Reliable AVQA} ($\mathcal{R}$-AVQA).

In a standard AVQA task, each input consists of an audio-visual data pair with a natural language question. Formally, let $\boldsymbol{x} = (\boldsymbol{v}, \boldsymbol{a}, \boldsymbol{q}) \in \mathcal{X}$ denote a video $\boldsymbol{v}$, an audio $\boldsymbol{a}$, and a question $\boldsymbol{q}$. The answer space is \(\mathcal{Y} = \{1, \ldots, K\}\), corresponding to a predefined set of candidate answers, and data are drawn from distribution $\mathcal{D}_{\mathcal{X},\mathcal{Y}}$ over $\mathcal{X} \times \mathcal{Y}$. A selective classifier is defined as a pair \((f, g)\), where \(f: \mathcal{X} \rightarrow \mathbb{R}^K\) is a base AVQA model, and \(g: \mathcal{X} \rightarrow \{0, 1\}\) is a selector function that decides whether to predict or abstain. Formally,
\begin{equation}
\label{eq:selective_classifier}
\scalebox{0.9}{$
(f, g)(\boldsymbol{x}) \triangleq 
\begin{cases} 
f(\boldsymbol{x}) & \text{if } g(\boldsymbol{x}) = 1, \\
\text{``abstain"} & \text{if } g(\boldsymbol{x}) = 0.
\end{cases}$}
\end{equation}
When \(g(\boldsymbol{x}) = 0\), the model refrains from answering input questions. This is particularly significant in AVQA, where incomplete cross-modal evidence or temporal ambiguity can result in confidently incorrect answers. In practice, the selector function $g$ is often implemented by applying a tunable threshold \(\gamma \in \mathbb{R}\) to a real-valued confidence measure:
\begin{equation}
\label{eq:selective_threshold}
g_{s, \gamma}(\boldsymbol{x}) \triangleq \mathbb{I}[s(\boldsymbol{x}) > \gamma],
\end{equation}
where \(s: \mathcal{X} \to \mathbb{R}\) is a confidence function that estimates reliability of prediction \(f(\boldsymbol{x})\). An ideal \(s(\boldsymbol{x})\) will yield high scores when \(f(\boldsymbol{x})\) is likely correct and low scores otherwise.

\subsection{Evaluation Metrics}
\label{sec:evaluation_metrics}


\textbf{Risk and Coverage}. \textit{Coverage} is the portion of input questions on which the model chooses to answer, while \textit{risk} quantifies the prediction error on this fraction of questions \cite{geifman2017selective, varshney2022investigating}.
\begin{subequations}
\begin{align}
\begin{split}
\text{Coverage:} \quad \mathcal{C}_{s,\gamma} \triangleq \mathbb{E}_{\boldsymbol{x} \sim \mathcal{D}_{\mathcal{X},\mathcal{Y}}}[g_{s,\gamma}(\boldsymbol{x})], \label{eq:coverage}
 \end{split}\\
\begin{split}
\text{Risk:} \quad \mathcal{R}_{s,\gamma} \triangleq \frac{\mathbb{E}_{(\boldsymbol{x},y) \sim \mathcal{D}_{\mathcal{X},\mathcal{Y}}}[\ell(f(\boldsymbol{x}), y) g_{s,\gamma}(\boldsymbol{x})]}{\mathcal{C}_{s,\gamma}}. \label{eq:risk}
\end{split}
\end{align}
\end{subequations}
$\ell(f(\bm{x}), y)$ equals 1 if $f(\boldsymbol{x}) \neq y$ and 0 otherwise, denoting the 0/1 loss \cite{geifman2017selective} to measure the accuracy of predicted $f(\bm{x})$ against true answer $y$. The risk $\mathcal{R}_{s,\gamma}$ represents the conditional error rate among questions the model elects to answer.
The $\mathcal{C}@\mathcal{R}$ metric is defined as the maximum coverage satisfying the risk constraint:
\begin{equation}
\label{eq:coverage_at_risk}
\scalebox{0.9}{$
\mathcal{C}@\mathcal{R} \triangleq \max \left\{ \mathcal{C}_{s,\gamma} \mid \mathcal{R}_{s,\gamma} \leq r \right\}.$}
\end{equation}
Eq. \ref{eq:coverage_at_risk} ensures that we evaluate the best possible operating point for each method at the specified risk tolerance $r$.
In practice, the risk level users are willing to tolerate varies across application contexts. 
Thus, we evaluate models using coverage at target risk levels ($\mathcal{C}@\mathcal{R}$), including 1\%, 5\%, 10\%, and 20\% risk. This metric directly captures the practical utility of a selective model: higher coverage at a fixed risk indicates a more reliable system. We also measure the Area Under Risk-Coverage curve (AURC) \cite{geifman2018bias} to assess performance across all operating points.
To mitigate sensitivity to the threshold $\gamma$ in $g$, we report the \textit{maximum achievable coverage} at each risk level, enabling direct and fair comparison of different selection functions.

\textbf{Expected Calibration Error.}
The goal of selective prediction is to optimize the balance between coverage and risk. This goal heavily relies on the selector $g$ that depends on a well-calibrated  function $s(\bm{x})$ \cite{guo2017calibration}.
To assess calibration, we use the Expected Calibration Error (ECE), which quantifies the discrepancy between a model’s predicted average confidence and its average accuracy. ECE serves as a diagnostic tool for evaluating the reliability of $s(\bm{x})$. When a model is well-calibrated, confidence scores provide reliable estimates of correctness probability, leading to predictable selective risk as a function of coverage and reducing systematic overconfidence.
Given a test set of size $N$, by dividing the predictions into $m$ bins ($m \ll N$) with $B_j$ being the set of indices in bin $j$, ECE is defined as: 
\begin{equation}
    \adjustbox{max width=0.4\textwidth}{$\displaystyle
    \mathrm{ECE} \triangleq \sum_{j=1}^{m} \frac{|B_j|}{N} \Bigg| \underbrace{\frac{\sum_{i \in B_j} \Correct[\hat{y}_f(x_i) = y_i]}{|B_j|}}_{\text{Accuracy}} - \underbrace{\frac{\sum_{i \in B_j} s(x_i)}{|B_j|}}_{\text{Confidence}} \Bigg|
    $}
    \label{eq:ece}
\end{equation}
\vspace{-6mm}
\subsection{Baseline for Selection Functions}

\textbf{Maximum Softmax Probability.}
A widely used approach for selective prediction is the maximum softmax probability (MSP) \cite{geifman2017selective} that is a simple, cost-free method that captures the peak of the predictive distribution. 
MSP confidence score is formulated as follows:
\begin{equation}
    C_{\mathrm{M}}(\boldsymbol{x}) = \max_{k \in \mathcal{Y}} \Prob(Y = k \mid \boldsymbol{x}) = \max_{k} \frac{\exp(l_k)}{\sum \limits_j \exp(l_j)}
\end{equation}
where $l_k$ is the logit for class $k$.
We now establish the theoretical conditions for MSP optimality in selective prediction.

\begin{definition}[Calibration]
\label{def:calibration}
Let $\hat{y} = \mathrm{softmax}(f(\boldsymbol{x}))$ be predicted class, $c(\boldsymbol{x}) = \Correct[\hat{y} = y^*] \in \{0, 1\}$ be correctness:

    \textbf{i) Weak (Marginal) Calibration}: A confidence score function $s(\boldsymbol{x}) \in [0, 1]$ is weakly calibrated if
    \[\Prob[c(\boldsymbol{x}) = 1 \mid s(\boldsymbol{x}) = p] = p, \quad \forall p \in [0, 1].\]
    \emph{Population-level property}: Predictions grouped by confidence $p$ have accuracy $p$ on average.
    
    \textbf{ii) Strong (Pointwise) Calibration}: A confidence score function $s(\boldsymbol{x})$ is strong calibrated if
    \[s(\boldsymbol{x}) =  \Prob[c(\boldsymbol{x}) = 1 \mid \boldsymbol{x}], \quad \forall \boldsymbol{x}.\]
    \emph{Sample-level property}: Confidence reflects the actual probability of correctness for each individual input.
\end{definition}

\begin{proposition}[MSE Reduction Implies Ranking Improvement]
\label{prop:mse_ranking}
Let $g^*(\boldsymbol{x}) = \Prob[c(\boldsymbol{x}) = 1 \mid \boldsymbol{x}]$ be the Bayes-optimal selection function. For any confidence estimator $s(\boldsymbol{x})$, define the mean squared error as $\mathrm{MSE}(s) = \E\bigl[\bigl(s(\boldsymbol{x}) - c(\boldsymbol{x}) \bigl)^2 \bigl]$.
If estimator $s_1$ has lower MSE than estimator $s_2$, then $s_1$ is a closer approximation to $g^*$ in terms of MSE:
\[
\E\bigl[\bigl(s_1\bigl(\boldsymbol{x}) - g^*(\boldsymbol{x}) \bigl)^2 \bigl] < \E\bigl[\bigl(s_2(\boldsymbol{x}) - g^*(\boldsymbol{x}) \bigl)^2 \bigl]
\]
Consequently, ranking by $s_1$ more closely approximates the Bayes-optimal ranking induced by $g^*$.
\end{proposition}

\begin{remark}
Since $g^*(\boldsymbol{x})$ induces the optimal ranking for selective prediction \cite{geifman2017selective}, any estimator that better approximates $g^*$ will produce rankings closer to optimal, thereby improving metrics such as AURC.
\end{remark}


\begin{theorem}[MSP Optimality under Strong Calibration]
\label{thm:msp_optimal}
Consider a classifier that produces class probabilities
$\Prob(Y = k \mid \boldsymbol{x})$ for each input $\boldsymbol{x}$.
If the MSP confidence function $C_{\mathrm{M}}(\boldsymbol{x}) = \max_k \Prob(Y=k|\boldsymbol{x})$ is strongly calibrated (Definition~\ref{def:calibration}), then 1) MSP equals the Bayes-optimal selection function: $C_{\mathrm{M}}(\boldsymbol{x}) = g^*(\boldsymbol{x}) = \Prob[c(\boldsymbol{x}) = 1 \mid \boldsymbol{x}]$. 2) Ranking examples by MSP induces the Bayes-optimal ranking for selective prediction. 3) This ranking minimizes risk at every coverage level, thereby achieving minimum AURC.
\end{theorem}

The proof is provided in Appendix \ref{section:A}. Under strong calibration, MSP coincides with the Bayes-optimal selection function, making it the ideal confidence score for selective prediction.
However, this condition is rarely satisfied in practice because modern neural networks are often poorly calibrated \cite{guo2017calibration}, especially in multimodal models, causing MSP to deviate from $g^*$. This deviation degrades the relative ranking of predictions, placing some incorrect predictions above correct ones and thereby increasing selective risk. Our approach addresses this by learning to \emph{correct} MSP's ranking errors rather than replacing MSP entirely.

\textbf{Monte Carlo Dropout.}
Another common method for selective prediction is Monte Carlo Dropout (MCD) \cite{gal2016dropout}. It measures Bayesian uncertainty by activating dropout at test time. Given an input $\boldsymbol{x}$ and a network $f(\boldsymbol{x}; \boldsymbol{\theta})$, MCD performs $K$ stochastic forward passes by sampling dropout masks $\{\boldsymbol{\theta}^{(k)}\}_{k=1}^{K}$. Each pass yields a predictive distribution $p^{(k)}(\hat{y} \mid \boldsymbol{x})$.
The Bayesian posterior predictive distribution is then approximated by averaging:
\begin{equation}
\label{eq:mc_dropout_mean}
\bar{p}(\hat{y} \mid \boldsymbol{x}) = \frac{1}{K} \sum_{k=1}^{K} p^{(k)}(\hat{y} \mid \boldsymbol{x}).
\end{equation}
Then, the confidence scores of MCD can be derived as 
$C_{\mathrm{MC}}(\boldsymbol{x}) = \max_{y \in \mathcal{Y}} \bar{p}(y \mid \boldsymbol{x})$.
While MCD offers a richer estimate of uncertainty than single-pass methods, its $K$-fold inference is prohibitive for latency-sensitive AVQA. 

\textbf{Calibration.} Calibration adjusts the model's confidence score so that predicted probabilities accurately reflect the likelihood of correctness. Poor calibration can negatively impact confidence ranking \cite{kamath-etal-2020-selective}, thereby degrading selective prediction performance.
Among calibration methods, Platt scaling \cite{platt1999probabilistic}, originally proposed for binary classification, cannot apply to multi-class settings. Temperature scaling \cite{guo2017calibration} is another method; however, it preserves the relative confidence ranking, thereby providing no effect for selective prediction.
Finally, we utilize vector scaling (VS) \cite{guo2017calibration} to adjust confidence values and relative rankings via a class-specific affine transformation of the logits:
\begin{equation}
\label{eq:calibration}
h\bigl(f(\boldsymbol{x}) \bigl) = \mathrm{softmax}\bigl(\mathbf{W} f(\boldsymbol{x}) + \mathbf{b} \bigl).
\end{equation}
where \(\mathbf{W}\) is a diagonal scaling matrix, \(\mathbf{b}\) is a bias term, and \(h\) is the calibration function. We apply MSP to the calibrated outputs for selective prediction. However, 
VS operates on output logits and ignores uncertainty signals in intermediate features. This highlights \textit{the need for methods that utilize multimodal features to better model prediction correctness}.

\textbf{Doctor}. We further employ Doctor \cite{granese2021doctor} as a more recent post-hoc baseline for selective prediction. Doctor identifies uncertain samples by examining the entire predictive distribution rather than just the maximum logit. Given a model's softmax output $\Prob(Y \mid \boldsymbol{x}) = [p_1, \dots, p_K]^\top$, the Doctor confidence score is defined as $s(\boldsymbol{x}) = \sum_{k=1}^{K} p_k(\boldsymbol{x})^2.$
This quadratic formulation, which is closely related to the complement of the Gini impurity, penalizes high-entropy predictions arising from multi-class confusion. This provides a better ranking of risk-coverage trade-offs than the MSP. Similar to MSP, the doctor operates on output logits while ignoring rich uncertainty signals from intermediate features, leading to suboptimal performance.

\subsection{Adaptive Confidence Refinement}
\label{sec:proposed_method}

We propose Adaptive Confidence Refinement (ACR), a lightweight framework that enhances selective prediction by refining MSP-based confidence. 
Instead of replacing MSP, ACR retains its primary signal while selectively correcting its failure cases. ACR 1) models prediction correctness as binary classification task and 2) operates on intermediate features to learn from rich uncertainty signals, while preserving a theoretical guarantee for improvement over MSP.

\begin{definition}[Error Moments]
\label{def:moments}
Let $C_{\mathrm{M}}(\boldsymbol{x})$ denote MSP's output and let $C_{\mathrm{R}}(\boldsymbol{x})$ be the output of the residual risk estimator (RRH head). We define the following errors:
\[
\epsilon_{\mathrm{M}}(\boldsymbol{x}) = C_{\mathrm{M}}(\boldsymbol{x}) - c(\boldsymbol{x}), \quad \epsilon_{\mathrm{R}}(\boldsymbol{x}) = C_{\mathrm{R}}(\boldsymbol{x}) - c(\boldsymbol{x})
\]
and their second moments:
\begin{align}
    \sigma^2_{\mathrm{M}} &= \E[\epsM^2] \tag*{(MSE of MSP)}\\
    \sigma^2_R &= \E[\epsR^2] \tag*{(MSE of Residual Risk Estimator)}\\
    \sigma_{MR} &= \E[\epsM\epsR] \tag*{(Error Cross-Moment)}
\end{align}
\end{definition}
\begin{theorem}[Fusion Benefit for Confidence Estimation]
\label{thm:fusion_benefit}
If the error cross-moment satisfies the following condition:
\begin{equation}
    \sigma_{\mathrm{MR}} < \min(\sigma^2_{\mathrm{M}}, \sigma^2_{\mathrm{R}})
    \label{eq:fusion_condition_main}
\end{equation}
then there exists a unique optimal fixed fusion weight $\bar{\alpha}^* \in (0, 1)$ such that the fused confidence:
\begin{equation}
    C_{\bar{\alpha}^*}(\boldsymbol{x}) = \bar{\alpha}^* C_{\mathrm{M}}(\boldsymbol{x}) + (1 - \bar{\alpha}^*) C_{\mathrm{R}}(\boldsymbol{x}) \nonumber
\end{equation}
achieves strictly lower MSE than either estimator alone. Moreover, by Proposition~\ref{prop:mse_ranking}, $C_{\bar{\alpha}^*}$ provides a closer approximation to the Bayes-optimal selection function $g^*$, thereby improving the relative ranking of predictions compared to using MSP alone.
The optimal weight is given by:
\begin{equation}
    \bar{\alpha}^* = \frac{\sigma^2_{\mathrm{R}} - \sigma_{\mathrm{MR}}}{\sigma^2_{\mathrm{M}} + \sigma^2_{\mathrm{R}} - 2\sigma_{\mathrm{MR}}} \nonumber
    \label{eq:alpha_star}
\end{equation}
\end{theorem}
\vspace{-1mm}
The proof is provided in Appendix \ref{section:A}. The condition in Eq. \ref{eq:fusion_condition_main} ensures that the residual risk estimator provides complementary information to correct errors in regions where MSP fails to accurately rank predictions. This means the errors of the two estimators are sufficiently ``spread apart,'' allowing their combination to mitigate errors and enhance the separation of correct and incorrect predictions. Empirically, our ACR meets this condition, as shown in Appendix \ref{section:A_1}.

\textbf{Design Principle}. ACR augments the MSP baseline with two learned components: a Residual Risk Head (RRH) and a Confidence Gating Head (CGH), as in Fig. \ref{fig:main_arch}. The RRH exploits intermediate multimodal features to model uncertainty beyond MSP, especially in overconfident failure cases. The CGH determines when to apply residual risk corrections, enabling the input-adaptive fusion.
Unlike prior work in unimodal or bimodal settings \cite{dong-etal-2018-confidence, liu2019deep, kamath-etal-2020-selective}, AVQA poses unique challenges due to its trimodal nature, requiring reasoning over interactions among several modalities: video, audio, questions, and candidate answers. Thus, we use the multimodal fusion features produced by the final layer of the backbone $f$ and the pre-softmax logits as inputs to both heads, enabling access to predicted answers and their supporting evidence.
For simplicity and efficiency, both RRH and CGH are used as multi-layer perceptrons (MLP) with sigmoid outputs, constraining predictions to $[0,1]$. Architectural details are given in Appendix \ref{section:H}. The final correctness score is expressed as:
\begin{equation}
\label{eq:final_conf}
C_{\mathrm{ACR}}(\boldsymbol{x}) =  \alpha (\boldsymbol{x})C_{\mathrm{M}}(\boldsymbol{x}) + (1 - \alpha(\boldsymbol{x}))C_{\mathrm{R}}(\boldsymbol{x})
\end{equation}
where $C_{\mathrm{R}}(\boldsymbol{x})$ and $\alpha (\boldsymbol{x})$ are the output of RRH and CGH; $C_{\mathrm{M}}(\boldsymbol{x})$ is the MSP-based confidence score. 
Overall, our ACR defines a comprehensive confidence estimation framework that encompasses both calibration and selector-based strategies.
Specifically, when RRH head is removed, ACR functions solely as a calibration method \cite{platt1999probabilistic}. In contrast, if CGH head is removed, it functions as a selector, as in prior works \cite{geifman2019selectivenet, whitehead2022reliable}. We provide further ablations and analysis in Appendix \ref{section:B}, where we show straightforward calibration or selector-based methods fall short for $\mathcal{R}$-AVQA.

\begin{theorem}[Optimality of Input-Adaptive Confidence Fusion]
\label{thm:adaptive_fusion}
Let the adaptive confidence score be
$
C_{\mathrm{ACR}}(\boldsymbol{x}) = \alpha^{\dagger}(\boldsymbol{x}) C_{\mathrm{M}}(\boldsymbol{x}) + (1-\alpha^{\dagger}(\boldsymbol{x})) C_{\mathrm{R}}(\boldsymbol{x}),$
where $\alpha^{\dagger}(\boldsymbol{x}) \in [0,1]$ is an input-adaptive fusion weight.
Assume $C_{\mathrm{M}}(\boldsymbol{x})$ and $C_{\mathrm{R}}(\boldsymbol{x})$ are measurable functions of $\boldsymbol{x}$ and the hypothesis class $\mathcal{H} = \{\alpha^{\dagger}: \mathcal{X} \to [0,1] \mid \alpha^{\dagger} \text{ measurable}\}$ has sufficient capacity to represent the Bayes-optimal fusion weight $\alpha^*(\boldsymbol{x})$.
When training $\mathcal{R}$-AVQA model using binary cross-entropy loss on correctness labels (0/1), any risk minimizer satisfies
\[
C_{\mathrm{ACR}}(\boldsymbol{x}) = \Prob(c(\boldsymbol{x})=1 \mid \boldsymbol{x})
\quad \text{a.e.}
\]
Since $C_{\mathrm{ACR}} = g^*$ a.e., ranking by $C_{\mathrm{ACR}}$ gets closer to the Bayes-optimal ranking for selective prediction, providing lower AURC.
If the pointwise optimal fusion weight $\alpha^*(\boldsymbol{x}) = \arg\min_\alpha \E\bigl[\bigl(C_\alpha(\boldsymbol{x}) - c(\boldsymbol{x}) \bigl)^2 \mid \boldsymbol{x} \bigl]$ varies across inputs (i.e., $\mathrm{Var}[\alpha^*(\boldsymbol{x})] > 0$), then any fixed weight $\lambda \in [0,1]$ to $C_{\lambda}$ is suboptimal in both MSE and ranking.
\end{theorem}

The proof is in Appendix \ref{section:A} and $\alpha$ distribution analysis is in Appendix \ref{section:B_1}. The insight is that input-adaptive fusion expands hypothesis class of confidence estimators, allowing sample-specific corrections to MSP and yielding a closer approximation to the Bayes-optimal correctness probability than any static fusion. Appendix \ref{section:B_1} shows that \textbf{no distribution collapses} to fixed values (i.e., $\alpha = \{0, 1\}$), confirming genuine input-adaptive fusion instead of trivial solutions.

\begin{figure*}[ht]
    \centering
    \includegraphics[width=0.90\linewidth]{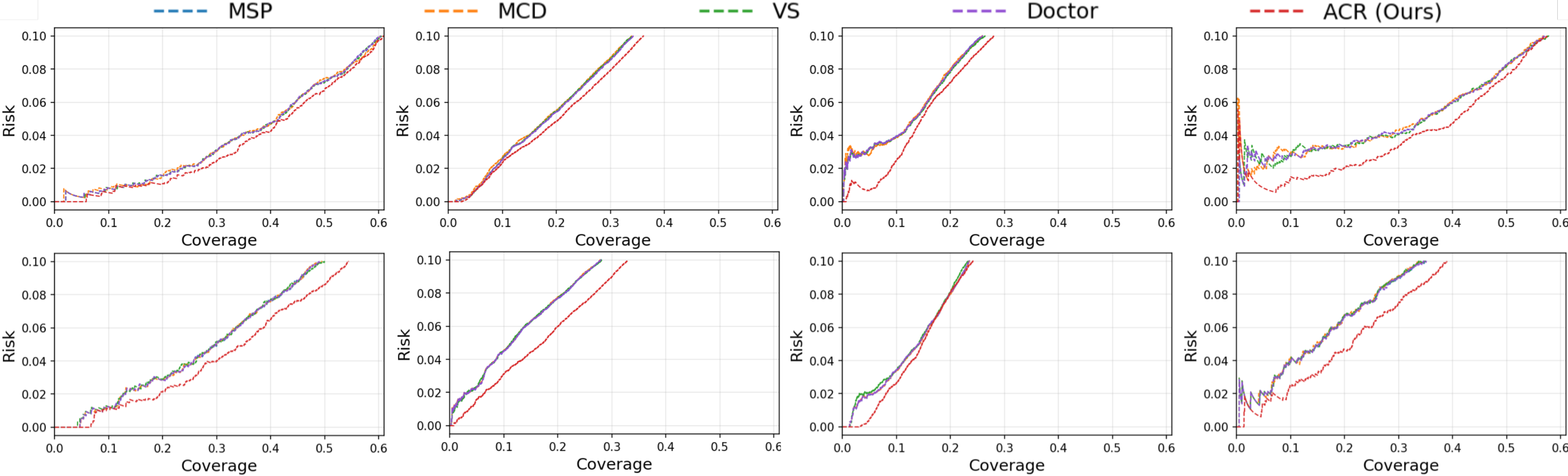}
    \vspace{-2mm}
    \caption{Risk-coverage curves of various selection functions for QA-TIGER (top row) and TSPM (bottom row) up to 10\% risk.} 
    \label{fig:aurc_curve}
    \vspace{-2mm}
\end{figure*}

\textbf{Training Scheme}. ACR is designed to improve confidence estimation while preserving the predictive behavior via a two-stage training procedure, as shown in Fig. \ref{fig:main_arch}.
In \textit{Stage 1}, we train a base AVQA model using standard cross-entropy loss, then freeze all parameters.
In \textit{Stage 2}, with the backbone frozen, we jointly train the RRH and CGH using a binary cross-entropy loss $\mathcal{L}_{\mathrm{BCE}}$.
Our training objective targets selective prediction by learning a confidence estimator that approximates the posterior probability of correctness $\Prob(\mathrm{correct} \mid \boldsymbol{x})$.
We optimize $\mathcal{L}_{\mathrm{BCE}}$ using correctness signal \(c(\boldsymbol{x})\) (0/1 labels) as a binary reward, aligning our ACR with the Bayes-optimal ranking for selective prediction. 
By freezing the backbone, we decouple confidence estimation from base prediction, enabling selective prediction
optimization without interference from classification gradients.
\begin{equation}
    \scalebox{0.87}{$
    \label{eq:bce_loss}
    \mathcal{L}_{\mathrm{BCE}} = -\mathbb{E} \bigl [c(\boldsymbol{x})\log C_{\mathrm{ACR}}(\boldsymbol{x}) - \bigl(1 - c(\boldsymbol{x}) \bigl)\log\bigl(1 - C_{\mathrm{ACR}}(\boldsymbol{x}) \bigl) \bigl ] \nonumber$
    }
\end{equation}
BCE is a strictly proper scoring rule that minimizes expected risk when predicted confidence matches the conditional probability of correctness \cite{gneiting2007strictly}. 
Optimizing $\mathcal{L}_{\mathrm{BCE}}$ encourages \(C_{\mathrm{ACR}}(\boldsymbol{x})\) to approximate the probability \(\mathbb{P}(c(\boldsymbol{x}) = 1 \mid x)\).
From Eq. \ref{eq:final_conf}, the gradient of $\mathcal{L}_{\mathrm{BCE}}$ with respect to $\alpha(\boldsymbol{x})$ is proportional to:
\begin{equation}
    \scalebox{0.9}{$
    \label{eq:bce_gradient}
    \frac{\partial \mathcal{L}_{\text{BCE}}}{\partial \alpha} \propto \bigl(C_{\mathrm{M}}(x) - C_{\mathrm{R}}(x) \bigl) \cdot \bigl(C_{\mathrm{ACR}}(\boldsymbol{x}) - c(\boldsymbol{x}) \bigl). \nonumber$
    }
\end{equation}
When MSP is accurate, the gradients increase \(\alpha(\boldsymbol{x})\), reinforcing it dominance. If MSP is overconfident yet incorrect, the gradients reduce \(\alpha(\boldsymbol{x})\), activating residual correction.
\section{Experiments}

\subsection{Datasets and Baselines}
\label{sec:dataset_baselines}

We conduct experiments on MUSIC-AVQA \cite{li2022learning}, MUSIC-AVQA-R \cite{ma2024look}, and MUSIC-AVQA-v2.0 \cite{liu2024tackling} datasets. MUSIC-AVQA is a widely used dataset for audio-visual reasoning. 
MUSIC-AVQA-R targets \textit{rare and OOD generalization under distribution shift} samples, while MUSIC-AVQA-v2.0 mitigates dataset biases by enhancing diversity in ensemble scenes, providing both \textit{biased and balanced sets} to assess model robustness.
In our setup, we set the AVQA models to \( f \), and the selection functions to \( g \). We train the AVQA models (\( f \)) on the training sets and use validation sets to evaluate them and train the selection functions (\( g \)). Separate training data for \( g \) is crucial to prevent overfitting on the training set of AVQA models and to preserve generalization.
We split the test sets into two subsets: one (20\%) for validating the selection functions and one held-out subset (80\%) for evaluating the final selective models. More details are in Appendix \ref{section:G}.

We evaluate the selection functions discussed in Sect. \ref{sec:methodology} on representative AVQA baseline models spanning diverse architectural designs. Specifically, we consider QA-TIGER \cite{kim2025question}, ST-AVQA \cite{li2022learning}, and TSPM \cite{li2024boosting} that are widely used in prior AVQA studies.

\vspace{-2mm}
\subsection{Benchmarking Evaluation Metrics}
\label{sec:benchmarking_metric}

In this section, we benchmark maximum coverage at a target risk ($\mathcal{C}@\mathcal{R}$), area under the risk-coverage curve (AURC), and expected calibration error (ECE). All results are evaluated on the new held-out test split defined in Sect. \ref{sec:dataset_baselines}.

\begin{table}[]
\centering
\setlength{\tabcolsep}{3pt}
\caption{Coverage at target risk ($\mathcal{C}$@$\mathcal{R}$) $\uparrow$, AURC $\downarrow$, and ECE $\downarrow$ for all selection functions on MUSIC-AVQA dataset. All in \%.}
\vspace{-1mm}
\label{tab:car_music_avqa}
\scalebox{0.705}{
\begin{tabular}{llcccccc}
\toprule
Model $f$ & Sel. Func. $g$ & $\mathcal{C}$@1\% & $\mathcal{C}$@5\% & $\mathcal{C}$@10\% & $\mathcal{C}$@20\% & AURC & ECE \\
\midrule
\multirow{5}{*}{QA-TIGER}   & MSP                             & 14.29 & 41.53 & 60.47 & 91.83 & 8.67                 & 8.83                 \\
                          & MCD                       & \underline{16.81}      & 41.03      & 60.34       & \underline{91.94}       & 8.69     & 8.83     \\
                          & VS            & 14.23      & \underline{41.57}      & \underline{60.54}       & 91.84       & 8.66     & 8.91     \\
                          & Doctor            & 14.29      & 41.44      & 60.38      & 91.31       & \underline{8.65}     & \underline{4.99}     \\
                          & ACR (Ours) & \textbf{17.21} & \textbf{43.07} & \textbf{61.39}  & \textbf{91.96}  & \textbf{8.41} & \textbf{4.88} \\ \hdashline
                          & \textcolor{gray}{Oracle} & \textcolor{gray}{78.08} & \textcolor{gray}{81.37} & \textcolor{gray}{85.89}  & \textcolor{gray}{96.93}  & \textcolor{gray}{2.80} & \textcolor{gray}{0.00} \\ \midrule
\multirow{5}{*}{ST-AVQA}  & MSP                   & 1.71  & 8.00  & \underline{24.03}  & 65.59  & \underline{16.05} & 6.23     \\
                          & MCD                       &  1.71     &  8.06     &  23.43      & 65.03       & 16.07     & 6.22     \\
                          & VS            &  \underline{1.74}     & \underline{8.11}      &  23.44      &  \underline{65.88}      & 16.07     & 6.18      \\
                          & Doctor                   & 1.71  & 7.98  & 23.86  & 65.80  & 16.08 & \underline{5.01}     \\
                          & ACR (Ours) & \textbf{3.67}  & \textbf{10.65} & \textbf{28.55}  & \textbf{66.03}  & \textbf{15.46} & \textbf{1.35}     \\ \hdashline
                          & \textcolor{gray}{Oracle} & \textcolor{gray}{71.70} & \textcolor{gray}{74.08} & \textcolor{gray}{78.19}  & \textcolor{gray}{87.97}  & \textcolor{gray}{4.90} & \textcolor{gray}{0.00}     \\ \midrule
\multirow{5}{*}{TSPM}     & MSP                   & \underline{8.82}      & \underline{30.04}      & 49.21       & 87.60       & \underline{10.40}     & 4.63     \\
                          & MCD                       & 8.73      & 29.98      & 49.58       & 87.43       & \underline{10.40}     & 4.46     \\
                          & VS            & 6.63      & 29.67      & \underline{50.00}       & \underline{87.87}       & 10.41     & 4.68     \\
                          & Doctor                   & \underline{8.82}      & 29.90      & 49.03       & 86.87       & 10.47     & \underline{4.35}     \\
                          & ACR (Ours) & \textbf{8.89}      & \textbf{35.12}      & \textbf{54.50}       & \textbf{88.66}       & \textbf{9.76}     & \textbf{1.51}     \\ \hdashline
                          & \textcolor{gray}{Oracle} & \textcolor{gray}{76.91} & \textcolor{gray}{80.14} & \textcolor{gray}{84.60}  & \textcolor{gray}{95.17}  & \textcolor{gray}{3.11} & \textcolor{gray}{0.00}     \\
\bottomrule
\end{tabular}}
\vspace{-6mm}
\end{table}

\begin{figure*}[t]
    \centering
    \includegraphics[width=0.86\linewidth]{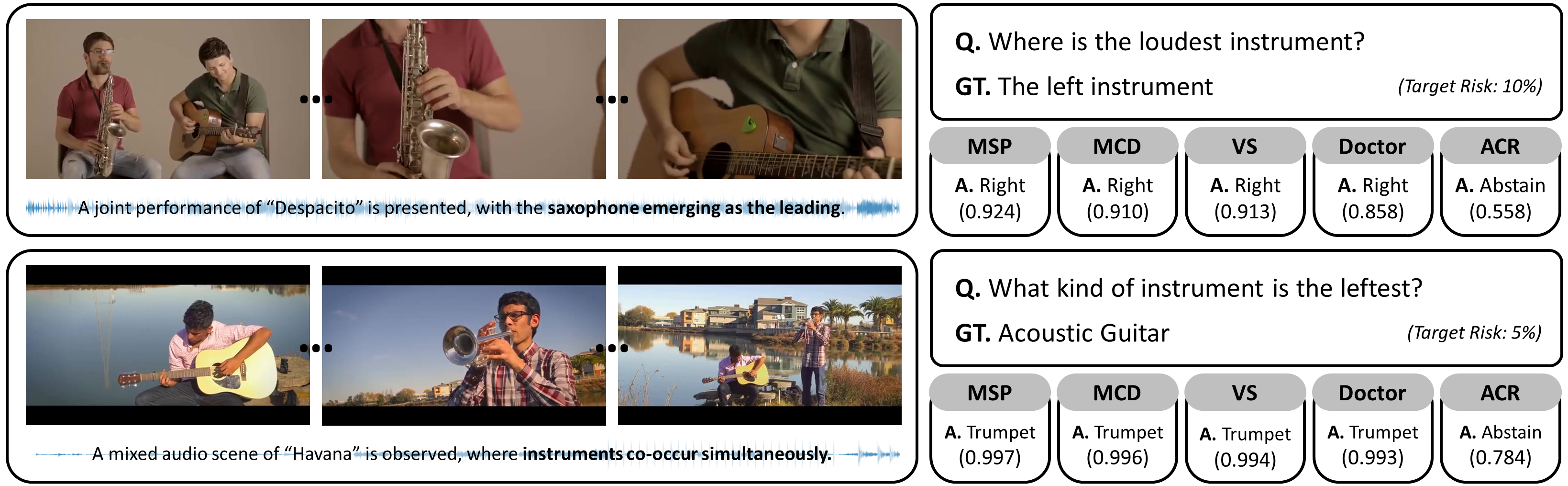}
    \caption{Qualitative examples of selective prediction based on QA-TIGER backbone. Examples in which standard confidence-based methods such as MSP, MCD, and VS produce overconfident but incorrect answers, whereas our proposed ACR correctly abstains by identifying unreliable predictions. Confidence scores at different target risk levels are shown in parentheses.}
    \label{fig:qualitative_result}
    \vspace{-1mm}
\end{figure*}

\begin{table*}[!t]
\centering
\caption{Comparison on MUSIC-AVQA-R and MUSIC-AVQA-v2 datasets. We report Head (H)/Tail (T) results for MUSIC-AVQA-R and Balance (Bal)/Biased (Bias) results for MUSIC-AVQA-v2. Metrics: coverage at target risk ($\mathcal{C}$@$\mathcal{R}$) $\uparrow$, AURC $\downarrow$, and ECE $\downarrow$ (all in \%).}
\label{tab:merged_table23}
\resizebox{0.85\textwidth}{!}{%
\begin{tabular}{cll cccccc cccccc}
\toprule
\multirow{2}{*}{Dataset} & \multirow{2}{*}{Model $f$} & \multirow{2}{*}{Sel. Func. $g$} 
& \multicolumn{2}{c}{$\mathcal{C}$@1\%} & \multicolumn{2}{c}{$\mathcal{C}$@5\%} 
& \multicolumn{2}{c}{$\mathcal{C}$@10\%} & \multicolumn{2}{c}{$\mathcal{C}$@20\%} 
& \multicolumn{2}{c}{AURC} & \multicolumn{2}{c}{ECE} \\
\cmidrule(lr){4-5}\cmidrule(lr){6-7}\cmidrule(lr){8-9}\cmidrule(lr){10-11}\cmidrule(lr){12-13}\cmidrule(lr){14-15}
& & & \small H/Bal & \small T/Bias & \small H/Bal & \small T/Bias & \small H/Bal & \small T/Bias & \small H/Bal & \small T/Bias & \small H/Bal & \small T/Bias & \small H/Bal & \small T/Bias \\
\midrule
\multirow{10}{*}{\rotatebox{90}{MUSIC-AVQA-R}}
& \multirow{5}{*}{QA-TIGER}
& MSP         & 6.04 & 0.23 & 18.55 & 13.48 & \underline{34.25} & 25.96 & 63.56 & 59.06 & 15.43 & 16.87 & 14.25 & 14.16 \\
& & MCD         & 5.65 & \underline{0.47} & 18.09 & 13.12 & 33.96 & 25.99 & \underline{63.63} & 58.57 & 15.44 & 16.94 & 14.08 & 14.13 \\
& & VS          & \textbf{6.12} & 0.24 & 18.39 & \underline{13.55} & 34.07 & \underline{26.43} & 63.30 & \underline{59.52} & 15.49 & \underline{16.77} & 14.25 & 14.07 \\
& & Doctor         & 6.04 & 0.23 & \underline{18.56} & 13.49 & 34.24 & 26.01 & 63.52 & 59.31 & \underline{15.41} & 16.81 & \underline{11.87} & \underline{11.95} \\
& & ACR (Ours) & \underline{6.11} & \textbf{6.46} & \textbf{20.59} & \textbf{14.62} & \textbf{36.16} & \textbf{28.03} & \textbf{63.84} & \textbf{59.98} & \textbf{15.17} & \textbf{16.36} & \textbf{10.86} & \textbf{9.90} \\ \cdashline{2-15}
& & \textcolor{gray}{Oracle} & \textcolor{gray}{68.41} & \textcolor{gray}{68.86} & \textcolor{gray}{71.29} & \textcolor{gray}{71.76} & \textcolor{gray}{75.25} & \textcolor{gray}{75.75} & \textcolor{gray}{84.66} & \textcolor{gray}{85.22} & \textcolor{gray}{5.88} & \textcolor{gray}{5.71} & \textcolor{gray}{0.00} & \textcolor{gray}{0.00} \\
\cmidrule{2-15}
& \multirow{5}{*}{TSPM}
& MSP         & 0.80 & 1.97 & \underline{11.51} & 14.22 & 28.07 & 23.49 & 65.33 & 55.37 & 16.00 & 17.95 & 5.46 & 8.31 \\
& & MCD         & 0.80 & 1.98 & 11.50 & \underline{14.24} & 27.88 & \underline{23.54} & 65.22 & \underline{55.61} & 16.03 & \underline{17.92} & 5.39 & 8.22 \\
& & VS          & \underline{0.89} & \underline{2.05} & 11.35 & 14.08 & \underline{28.18} & 23.28 & \underline{65.66} & 54.76 & \underline{15.93} & 18.02 & \underline{5.34} & 8.22 \\
& & Doctor         & 0.80 & 1.97 & 11.47 & 14.22 & 28.09 & 23.53 & 65.31 & 55.16 & 16.01 & 18.04 & 5.52 & \underline{6.16} \\
& & ACR (Ours) & \textbf{3.55} & \textbf{6.15} & \textbf{17.25} & \textbf{14.75} & \textbf{32.94} & \textbf{24.19} & \textbf{65.86} & \textbf{55.77} & \textbf{15.34} & \textbf{17.62} & \textbf{3.75} & \textbf{5.99} \\ \cdashline{2-15}
& & \textcolor{gray}{Oracle} & \textcolor{gray}{68.62} & \textcolor{gray}{66.60} & \textcolor{gray}{71.72} & \textcolor{gray}{69.40} & \textcolor{gray}{75.71} & \textcolor{gray}{73.26} & \textcolor{gray}{85.17} & \textcolor{gray}{82.41} & \textcolor{gray}{5.72} & \textcolor{gray}{6.60} & \textcolor{gray}{0.00} & \textcolor{gray}{0.00} \\

\midrule[\heavyrulewidth]
\multirow{10}{*}{\rotatebox{90}{MUSIC-AVQA v2.0}}
& \multirow{5}{*}{QA-TIGER}
& MSP         & 2.38 & \underline{1.45} & 34.98 & \underline{35.14} & 56.75 & 57.02 & \underline{88.04} & 88.98 & 9.84 & \underline{9.89} & 13.79 & 13.72 \\
& & MCD         & \underline{2.58} & 0.19 & 35.04 & 34.84 & 56.62 & 56.80 & 87.71 & \underline{89.05} & 9.88 & \underline{9.89} & 13.58 & 13.48 \\
& & VS          & 1.41 & 0.47 & \underline{36.58} & 35.12 & \underline{56.84} & \underline{57.05} & \textbf{88.11} & 88.68 & \underline{9.80} & 9.91 & 13.70 & 13.87 \\
& & Doctor          & 2.38 & \underline{1.45} & 34.98 & \underline{35.14} & 56.75 & 57.02 & \underline{88.10} & 88.64 & 9.84 & \underline{9.89} & \underline{9.89} & \underline{9.87} \\
& & ACR (Ours) & \textbf{10.64} & \textbf{8.88} & \textbf{39.79} & \textbf{40.90} & \textbf{58.63} & \textbf{57.59} & \textbf{88.11} & \textbf{89.53} & \textbf{9.23} & \textbf{9.23} & \textbf{3.52} & \textbf{2.99} \\ \cdashline{2-15}
& & \textcolor{gray}{Oracle} & \textcolor{gray}{76.69} & \textcolor{gray}{76.98} & \textcolor{gray}{79.92} & \textcolor{gray}{80.22} & \textcolor{gray}{84.36} & \textcolor{gray}{84.67} & \textcolor{gray}{94.91} & \textcolor{gray}{95.26} & \textcolor{gray}{3.16} & \textcolor{gray}{3.09} & \textcolor{gray}{0.00} & \textcolor{gray}{0.00} \\
\cmidrule{2-15}
& \multirow{5}{*}{TSPM}
& MSP         & 4.69 & \underline{0.47} & 17.42 & 15.46 & 34.80 & 34.96 & \underline{64.35} & \underline{66.49} & \underline{15.05} & 14.97 & 5.60 & 5.18 \\
& & MCD         & 4.44 & 0.45 & \underline{17.62} & 14.83 & \underline{35.02} & 34.85 & 64.23 & 66.43 & 15.07 & 15.00 & 5.45 & 5.16 \\
& & VS          & \underline{4.70} & 0.45 & 17.38 & 15.46 & 34.43 & 35.05 & 64.19 & 66.41 & \underline{15.05} & 15.00 & 5.58 & 5.05 \\
& & Doctor         & 4.69 & \underline{0.47} & 17.46 & \underline{15.49} & 34.81 & \underline{35.16} & 64.00 & 65.72 & 15.08 & \underline{14.96} & \underline{3.87} & \underline{4.68} \\
& & ACR (Ours) & \textbf{7.12} & \textbf{4.66} & \textbf{24.58} & \textbf{21.48} & \textbf{37.66} & \textbf{38.94} & \textbf{66.84} & \textbf{69.44} & \textbf{14.03} & \textbf{13.84} & \textbf{1.64} & \textbf{1.34} \\ \cdashline{2-15}
& & \textcolor{gray}{Oracle} & \textcolor{gray}{71.12} & \textcolor{gray}{71.95} & \textcolor{gray}{74.12} & \textcolor{gray}{74.98} & \textcolor{gray}{78.24} & \textcolor{gray}{79.15} & \textcolor{gray}{88.02} & \textcolor{gray}{89.04} & \textcolor{gray}{4.89} & \textcolor{gray}{4.60} & \textcolor{gray}{0.00} & \textcolor{gray}{0.00} \\
\bottomrule
\end{tabular}%
}
\vspace{-2mm}
\end{table*}

\textbf{ACR outperforms other methods}. Tab. \ref{tab:car_music_avqa}, \ref{tab:merged_table23} show the results of our ACR with other methods. Top-2 results are in \textbf{bold} and \underline{underline}.
Tab. \ref{tab:merged_table23} reports results only from recent AVQA models (QA-TIGER, TSPM) on MUSIC-AVQA-R and MUSIC-AVQA-v2.0 (trained on balance set) due to page limits; remaining results are in Appendix \ref{section:D}.
ACR consistently surpasses other strong baselines: maximum softmax probability (MSP) \cite{geifman2017selective}, Monte Carlo Dropout (MCD) \cite{gal2016dropout}, vector scaling (VS) \cite{guo2017calibration}, and Doctor \cite{granese2021doctor}.
The gains are most evident in low-risk regimes (e.g., $1\%$, $5\%$), showing strong results in settings that demand highly precise confidence estimation.
ACR achieves the largest gains in $\mathcal{C}@1\%$ on MUSIC-AVQA-v2.0 with QA-TIGER, yielding increases of $\textbf{+8.26\%p}$ on the balance test set and $\textbf{+7.43\%p}$ on the biased test set (Tab. \ref{tab:merged_table23}). ACR achieves markedly lower ECE than other methods, indicating superior calibration.
At low-risk levels, Fig. \ref{fig:aurc_curve} shows that ACR is much lower risk as coverage increases than other methods, proving its advantage in safety-critical conditions. Notably, other methods vary across settings: ID (Tab. \ref{tab:car_music_avqa}), OOD, and biased scenarios (Tab. \ref{tab:merged_table23}), while ACR surpasses these methods on all settings, showing its generalization even under distribution shifts or dataset biases. ACR clearly outperforms MSP, proving our design principle: MSP is a strong baseline, but can be improved further in reliability by correcting MSP failure cases. Appendix \ref{section:B_1} provides a qualitative explanation of our gains: ACR better separates correct and incorrect samples than MSP, particularly near the critical decision boundary at high confidence levels.

\textbf{Correlation between accuracy and reliability}. We observe a strong correlation between standard AVQA accuracy and reliability across Tab. \ref{tab:car_music_avqa}-\ref{tab:merged_table23}: \textit{higher accuracy tends to show better selective prediction behavior}. For example, QA-TIGER is the most accurate model in the standard AVQA and also surpasses TSPM and ST-AVQA in reliability evaluations across settings. This implies that reliability stems from improving either predictor $f$ or selector $g$. In addition, Appendix \ref{section:B_1} shows backbones with higher base AVQA accuracy tend to learn higher $\alpha$ values, indicating greater reliability.
However, in this study, we focus on developing a more effective selector $g$ to directly boost model reliability. 

\textbf{Can thresholds generalize to test-time?} So far, we have reported the maximum achievable coverage at predefined target risk levels, which corresponds to an idealized evaluation where thresholds are selected directly on the test set. In practice, however, abstention thresholds must be selected on a validation set and then fixed during testing. To assess the practicality of this setting, we evaluate how well validation-selected thresholds transfer to the test split by applying them to the test set and measuring the resulting test-time risk using QA-TIGER. As shown in Appendix \ref{section:F}, the discrepancies between the target risk and the realized test risk are consistently small, indicating that the learned thresholds generalize reliably across splits. This behavior aligns with prior findings in selective prediction literature \cite{geifman2019selectivenet, whitehead2022reliable}.

\textbf{Still room for improvement}. 
Tab. \ref{tab:car_music_avqa}, \ref{tab:merged_table23} show that, particularly at low risk levels, coverage remains far from ideal oracle performance. For example, Tab. \ref{tab:car_music_avqa} shows $\mathcal{C}$@1$\%$ gap between selection functions and oracle exceeds 60\%. This gap arises under low-risk constraints, as the oracle has access to ground-truth correctness for sample selection, whereas practical methods rely on imperfect confidence estimates from multimodal answering. As a result, uncertainty estimation errors for cross-modal questions lead to conservative rejection rates and lower coverage. This highlights a substantial room to improve both AVQA models and selection functions, enabling more reliable AVQA systems.

\vspace{-2mm}
\subsection{Qualitative Analysis}

Fig. \ref{fig:qualitative_result} presents qualitative results showing differences in confidence estimates between ACR and other methods.  The key improvement of ACR is enhanced confidence estimates and relative rankings among examples, rather than improved accuracy (stage 2 freezes AVQA models so ACR does not affect accuracy).
ACR assigns much lower confidence scores than other methods to abstain from answering genuinely uncertain or unanswerable questions.
In the first example, many frames show a single performer or instrument, while the audio 
features multiple instruments sounding simultaneously. 
In the second example, ambiguity stems from the audio, where instrumental volumes are similar, while the video offers clearer spatial cues.
This discrepancy between audio and visual cues creates cross-modal ambiguity that can mislead AVQA models into overconfident predictions.
In such cases, the multimodal features convey uncertainty signals that ACR effectively captures and reflects via its reduced confidence. More examples are in Appendix \ref{section:E}.

\vspace{-3mm}
\subsection{Ablation Study}
\label{sec:ablation_study}

In this section, we present ablations of our ACR selection function. More ablations are provided in Appendix \ref{section:B}.

\textbf{Impact of multimodal features}. Tab. \ref{tab:ablation_input_selector} highlights the importance of multimodal features for achieving high performance. Using representations in isolation, multimodal features ($f'(x)$, $avq$) consistently achieve higher coverage and lower AURC than unimodal ones ($a$, $v$, $q$). However, feeding all modalities into ACR (last row) degrades performance due to cross-modal misalignment \cite{ma2023calibrating, cai2025value} that hinders reliable confidence estimation. Therefore, we apply ACR only to the final fused multimodal features that better mitigate misalignment. This result underscores that reliable AVQA is inherently a multimodal problem. Combining $f'(x)$ and $avq$ produces the best performance, so we adopt this setting in all experiments. 

\begin{table}[t]
\centering
\setlength{\tabcolsep}{4.5pt}
\caption{Ablations of the input for ACR with QA-TIGER on the MUSIC-AVQA testset using $\mathcal{C}@\mathcal{R}$ $\uparrow$, AURC $\downarrow$, and ECE $\downarrow$. $f'(x), a, v, q$ and $avq$ are pre-softmax logits, audio, video, question, and fused multimodal representations, respectively.  All in \%.}
\vspace{-1mm}
\label{tab:ablation_input_selector}
\scalebox{0.70}{
\begin{tabular}{lcccccc}
\toprule
Features & $\mathcal{C}$@1\% & $\mathcal{C}$@5\% & $\mathcal{C}$@10\% & $\mathcal{C}$@20\% & AURC & ECE \\
\midrule
$a$ & 10.20 & 40.66 & 60.79 & 90.55 & 8.82 & 5.19 \\
$v$ & 3.57 & 34.15 & 59.35 & 91.02 & 9.17 & 5.01 \\
$q$ & 13.79 & 35.91 & 60.84 & 90.99 & 8.85 & 5.33 \\ \hdashline
$f'(x)$ & 10.98 & 38.77 & 59.39 & 91.63 & 8.96 & 5.53 \\
$avq$ & 12.54 & 41.63 & 60.99 & 90.89 & 8.75 & \underline{4.96} \\ 
($f'(x), a$) & 13.21 & 42.29 & 60.45 & 90.95 & 8.65 & 5.66 \\
($f'(x), v$) & 10.39 &  39.03 & 60.01 & 91.80 & 8.81 & \underline{4.96} \\
($f'(x), q$) & 12.46 & \underline{42.62} & \underline{61.09} & 91.15 & 8.55 & 5.69      \\
($f'(x), avq$) & \textbf{17.21} & \textbf{43.07} & \textbf{61.39} & \textbf{91.96} & \textbf{8.41} & \textbf{4.88}     \\ \hdashline
($f'(x), a, v, q, avq$) & \underline{14.29}  & 41.53 & 60.47  & \underline{91.83}  & \underline{8.52} & 5.14 \\
\bottomrule
\end{tabular}}
\vspace{-6mm}
\end{table}

\textbf{ACR architecture}. Appendix \ref{section:B} presents results obtained using various architectures for our ACR. We use the same architecture for both learned heads for simplicity and find that it generally provides the best performance. We find that simplifying the architectural design can reduce performance, whereas increasing complexity does not always yield significant improvements. When these findings are combined with the results in Tab. \ref{tab:ablation_input_selector}, it is clear that the ACR largely depends on input features and the optimization objective, rather than on the architecture's complexity alone.
\vspace{-3mm}
\section{Discussion and Conclusion}

The standard AVQA formulation does not allow models to abstain when uncertain, despite many real-world applications requiring predictions only at low risk (i.e., $\mathcal{C}$@1$\%$, $\mathcal{C}$@5$\%$). To address this, we present a new formulation for $\mathcal{R}$-AVQA, which prioritizes abstention over incorrect predictions.
We set new benchmarks by evaluating several AVQA models in a selective prediction setting and propose Adaptive Confidence Refinement (ACR) for $\mathcal{R}$-AVQA that substantially outperforms other methods. While MSP is a strong baseline for selective prediction, it remains far from an optimal selection function. 
Thus, ACR enhances confidence estimation by refining MSP rather than replacing it entirely. Our theoretical analysis and empirical results show that ACR consistently improves coverage across risk tolerance levels compared to other methods. 
Overall, our study offers a solid foundation for reliability analysis in multimodal reasoning and promotes AVQA systems that are not only accurate but also self-aware and trustworthy.

\nocite{langley00}

\section*{Acknowledgements}

This work was supported by the National Research Foundation of Korea(NRF) grant funded by the Korea government(MSIT)(RS-2025-00573160); the Technology Innovation Program(RS-2025-25422280) funded By the Ministry of Trade, Industry and Resources(MOTIR, Korea); and the High-Performance Computing Support Project, funded by the Government of the Republic of Korea (Ministry of Science and ICT) (RQT-25-070104).

Cem Subakan is supported by NSERC Discovery Grant RGPIN 2023-05759. This research was also enabled in part by support provided by Calcul Québec (https://www.calculquebec.ca/) and the Digital Research Alliance of Canada (https://alliancecan.ca/en).

The work was also supported by Hyundai Motor Chung Mong-Koo Global Scholarship to Dinh Phu Tran (1st author) and Thao Do (5th author).

\section*{Impact Statement}

This paper presents work whose goal is to advance the field of Machine Learning. There are many potential societal
consequences of our work, none of which we feel must be specifically highlighted here.

\bibliography{example_paper}

@inproceedings{langley00,
 author    = {P. Langley},
 title     = {Crafting Papers on Machine Learning},
 year      = {2000},
 pages     = {1207--1216},
 editor    = {Pat Langley},
 booktitle     = {Proceedings of the 17th International Conference
              on Machine Learning (ICML 2000)},
 address   = {Stanford, CA},
 publisher = {Morgan Kaufmann}
}

@inproceedings{li2022learning,
  title={Learning to answer questions in dynamic audio-visual scenarios},
  author={Li, Guangyao and Wei, Yake and Tian, Yapeng and Xu, Chenliang and Wen, Ji-Rong and Hu, Di},
  booktitle={Proceedings of the IEEE/CVF conference on computer vision and pattern recognition},
  pages={19108--19118},
  year={2022}
}

@article{ma2024look,
  title={Look, listen, and answer: Overcoming biases for audio-visual question answering},
  author={Ma, Jie and Hu, Min and Wang, Pinghui and Sun, Wangchun and Song, Lingyun and Pei, Hongbin and Liu, Jun and Du, Youtian},
  journal={Advances in Neural Information Processing Systems},
  volume={37},
  pages={9507--9531},
  year={2024}
}

@inproceedings{liu2024tackling,
  title={Tackling data bias in music-avqa: Crafting a balanced dataset for unbiased question-answering},
  author={Liu, Xiulong and Dong, Zhikang and Zhang, Peng},
  booktitle={Proceedings of the IEEE/CVF Winter Conference on Applications of Computer Vision},
  pages={4478--4487},
  year={2024}
}

@inproceedings{kim2025question,
  title={Question-Aware Gaussian Experts for Audio-Visual Question Answering},
  author={Kim, Hongyeob and Jung, Inyoung and Suh, Dayoon and Zhang, Youjia and Lee, Sangmin and Hong, Sungeun},
  booktitle={Proceedings of the Computer Vision and Pattern Recognition Conference},
  pages={13681--13690},
  year={2025}
}

@inproceedings{li2024boosting,
  title={Boosting audio visual question answering via key semantic-aware cues},
  author={Li, Guangyao and Du, Henghui and Hu, Di},
  booktitle={Proceedings of the 32nd ACM International Conference on Multimedia},
  pages={5997--6005},
  year={2024}
}

@inproceedings{geifman2019selectivenet,
  title={Selectivenet: A deep neural network with an integrated reject option},
  author={Geifman, Yonatan and El-Yaniv, Ran},
  booktitle={International conference on machine learning},
  pages={2151--2159},
  year={2019},
  organization={PMLR}
}

@inproceedings{kamath-etal-2020-selective,
    title = "Selective Question Answering under Domain Shift",
    author = "Kamath, Amita  and
      Jia, Robin  and
      Liang, Percy",
    editor = "Jurafsky, Dan  and
      Chai, Joyce  and
      Schluter, Natalie  and
      Tetreault, Joel",
    booktitle = "Proceedings of the 58th Annual Meeting of the Association for Computational Linguistics",
    month = jul,
    year = "2020",
    address = "Online",
    publisher = "Association for Computational Linguistics",
    url = "https://aclanthology.org/2020.acl-main.503/",
    doi = "10.18653/v1/2020.acl-main.503",
    pages = "5684--5696"
}

@inproceedings{dong-etal-2018-confidence,
    title = "Confidence Modeling for Neural Semantic Parsing",
    author = "Dong, Li  and
      Quirk, Chris  and
      Lapata, Mirella",
    editor = "Gurevych, Iryna  and
      Miyao, Yusuke",
    booktitle = "Proceedings of the 56th Annual Meeting of the Association for Computational Linguistics (Volume 1: Long Papers)",
    month = jul,
    year = "2018",
    address = "Melbourne, Australia",
    publisher = "Association for Computational Linguistics",
    url = "https://aclanthology.org/P18-1069/",
    doi = "10.18653/v1/P18-1069",
    pages = "743--753"
}

@article{geifman2017selective,
  title={Selective classification for deep neural networks},
  author={Geifman, Yonatan and El-Yaniv, Ran},
  journal={Advances in neural information processing systems},
  volume={30},
  year={2017}
}

@article{platt1999probabilistic,
  title={Probabilistic outputs for support vector machines and comparisons to regularized likelihood methods},
  author={Platt, John and others},
  journal={Advances in large margin classifiers},
  volume={10},
  number={3},
  pages={61--74},
  year={1999},
  publisher={Cambridge, MA}
}

@inproceedings{guo2017calibration,
  title={On calibration of modern neural networks},
  author={Guo, Chuan and Pleiss, Geoff and Sun, Yu and Weinberger, Kilian Q},
  booktitle={International conference on machine learning},
  pages={1321--1330},
  year={2017},
  organization={PMLR}
}

@article{lakshminarayanan2017simple,
  title={Simple and scalable predictive uncertainty estimation using deep ensembles},
  author={Lakshminarayanan, Balaji and Pritzel, Alexander and Blundell, Charles},
  journal={Advances in neural information processing systems},
  volume={30},
  year={2017}
}

@article{geifman2018bias,
  title={Bias-reduced uncertainty estimation for deep neural classifiers},
  author={Geifman, Yonatan and Uziel, Guy and El-Yaniv, Ran},
  journal={arXiv preprint arXiv:1805.08206},
  year={2018}
}

@article{kumar2022deep,
  title={Deep learning based assistive technology on audio visual speech recognition for hearing impaired},
  author={Kumar, L Ashok and Renuka, D Karthika and Rose, S Lovelyn and Wartana, I Made and others},
  journal={International Journal of Cognitive Computing in Engineering},
  volume={3},
  pages={24--30},
  year={2022},
  publisher={Elsevier}
}

@article{patel2025enhancing,
  title={Enhancing accessibility through machine learning: A review on visual and hearing impairment technologies},
  author={Patel, Pal and Pampaniya, Shreyansh and Ghosh, Ananya and Raj, Ritu and Karuppaih, Deepa and Kandasamy, Saravanakumar},
  journal={IEEE Access},
  year={2025},
  publisher={IEEE}
}

@article{asan2020artificial,
  title={Artificial intelligence and human trust in healthcare: focus on clinicians},
  author={Asan, Onur and Bayrak, Alparslan Emrah and Choudhury, Avishek},
  journal={Journal of medical Internet research},
  volume={22},
  number={6},
  pages={e15154},
  year={2020},
  publisher={JMIR Publications Inc., Toronto, Canada}
}

@article{lutkenhoner2013predictive,
  title={Predictive modeling for diagnostic tests with high specificity, but low sensitivity: a study of the glycerol test in patients with suspected Meniere’s disease},
  author={L{\"u}tkenh{\"o}ner, Bernd and Basel, T{\"u}rker},
  journal={PLoS One},
  volume={8},
  number={11},
  pages={e79315},
  year={2013},
  publisher={Public Library of Science San Francisco, USA}
}

@inproceedings{xin2021art,
  title={The art of abstention: Selective prediction and error regularization for natural language processing},
  author={Xin, Ji and Tang, Raphael and Yu, Yaoliang and Lin, Jimmy},
  booktitle={Proceedings of the 59th Annual Meeting of the Association for Computational Linguistics and the 11th International Joint Conference on Natural Language Processing (Volume 1: Long Papers)},
  pages={1040--1051},
  year={2021}
}

@inproceedings{whitehead2022reliable,
  title={Reliable visual question answering: Abstain rather than answer incorrectly},
  author={Whitehead, Spencer and Petryk, Suzanne and Shakib, Vedaad and Gonzalez, Joseph and Darrell, Trevor and Rohrbach, Anna and Rohrbach, Marcus},
  booktitle={European Conference on Computer Vision},
  pages={148--166},
  year={2022},
  organization={Springer}
}

@inproceedings{dancette2023improving,
  title={Improving selective visual question answering by learning from your peers},
  author={Dancette, Corentin and Whitehead, Spencer and Maheshwary, Rishabh and Vedantam, Ramakrishna and Scherer, Stefan and Chen, Xinlei and Cord, Matthieu and Rohrbach, Marcus},
  booktitle={Proceedings of the IEEE/CVF Conference on Computer Vision and Pattern Recognition},
  pages={24049--24059},
  year={2023}
}

@inproceedings{gal2016dropout,
  title={Dropout as a bayesian approximation: Representing model uncertainty in deep learning},
  author={Gal, Yarin and Ghahramani, Zoubin},
  booktitle={international conference on machine learning},
  pages={1050--1059},
  year={2016},
  organization={PMLR}
}

@book{vovk2005algorithmic,
  title={Algorithmic learning in a random world},
  author={Vovk, Vladimir and Gammerman, Alexander and Shafer, Glenn},
  year={2005},
  publisher={Springer}
}

@article{angelopoulos2023conformal,
  title={Conformal prediction: A gentle introduction},
  author={Angelopoulos, Anastasios N and Bates, Stephen and others},
  journal={Foundations and trends{\textregistered} in machine learning},
  volume={16},
  number={4},
  pages={494--591},
  year={2023},
  publisher={Now Publishers, Inc.}
}

@inproceedings{zadrozny2002transforming,
  title={Transforming classifier scores into accurate multiclass probability estimates},
  author={Zadrozny, Bianca and Elkan, Charles},
  booktitle={Proceedings of the eighth ACM SIGKDD international conference on Knowledge discovery and data mining},
  pages={694--699},
  year={2002}
}

@inproceedings{yun2021pano,
  title={Pano-avqa: Grounded audio-visual question answering on 360deg videos},
  author={Yun, Heeseung and Yu, Youngjae and Yang, Wonsuk and Lee, Kangil and Kim, Gunhee},
  booktitle={Proceedings of the IEEE/CVF International Conference on Computer Vision},
  pages={2031--2041},
  year={2021}
}

@inproceedings{li2023progressive,
  title={Progressive spatio-temporal perception for audio-visual question answering},
  author={Li, Guangyao and Hou, Wenxuan and Hu, Di},
  booktitle={Proceedings of the 31st ACM international conference on multimedia},
  pages={7808--7816},
  year={2023}
}

@inproceedings{li2024object,
  title={Object-aware adaptive-positivity learning for audio-visual question answering},
  author={Li, Zhangbin and Guo, Dan and Zhou, Jinxing and Zhang, Jing and Wang, Meng},
  booktitle={Proceedings of the AAAI Conference on Artificial Intelligence},
  volume={38},
  number={4},
  pages={3306--3314},
  year={2024}
}

@inproceedings{schwartz2019simple,
  title={A simple baseline for audio-visual scene-aware dialog},
  author={Schwartz, Idan and Schwing, Alexander G and Hazan, Tamir},
  booktitle={Proceedings of the IEEE/CVF Conference on Computer Vision and Pattern Recognition},
  pages={12548--12558},
  year={2019}
}

@article{freund2004generalization,
  title={Generalization bounds for averaged classifiers},
  author={Freund, Yoav and Mansour, Yishay and Schapire, Robert E},
  year={2004}
}

@inproceedings{varshney2011risk,
  title={A risk bound for ensemble classification with a reject option},
  author={Varshney, Kush R},
  booktitle={2011 IEEE statistical signal processing workshop (SSP)},
  pages={769--772},
  year={2011},
  organization={IEEE}
}

@article{mielke2022reducing,
  title={Reducing conversational agents’ overconfidence through linguistic calibration},
  author={Mielke, Sabrina J and Szlam, Arthur and Dinan, Emily and Boureau, Y-Lan},
  journal={Transactions of the Association for Computational Linguistics},
  volume={10},
  pages={857--872},
  year={2022},
  publisher={MIT Press One Broadway, 12th Floor, Cambridge, Massachusetts 02142, USA~…}
}

@inproceedings{chen2023adaptation,
  title={Adaptation with self-evaluation to improve selective prediction in llms},
  author={Chen, Jiefeng and Yoon, Jinsung and Ebrahimi, Sayna and Arik, Sercan and Pfister, Tomas and Jha, Somesh},
  booktitle={Findings of the Association for Computational Linguistics: EMNLP 2023},
  pages={5190--5213},
  year={2023}
}

@article{srinivasan2024selective,
  title={Selective" selective prediction": Reducing unnecessary abstention in vision-language reasoning},
  author={Srinivasan, Tejas and Hessel, Jack and Gupta, Tanmay and Lin, Bill Yuchen and Choi, Yejin and Thomason, Jesse and Chandu, Khyathi Raghavi},
  journal={arXiv preprint arXiv:2402.15610},
  year={2024}
}

@article{liu2019deep,
  title={Deep gamblers: Learning to abstain with portfolio theory},
  author={Liu, Ziyin and Wang, Zhikang and Liang, Paul Pu and Salakhutdinov, Russ R and Morency, Louis-Philippe and Ueda, Masahito},
  journal={Advances in Neural Information Processing Systems},
  volume={32},
  year={2019}
}

@article{varshney2022investigating,
  title={Investigating selective prediction approaches across several tasks in iid, ood, and adversarial settings},
  author={Varshney, Neeraj and Mishra, Swaroop and Baral, Chitta},
  journal={arXiv preprint arXiv:2203.00211},
  year={2022}
}

@article{jiang2021can,
  title={How can we know when language models know? on the calibration of language models for question answering},
  author={Jiang, Zhengbao and Araki, Jun and Ding, Haibo and Neubig, Graham},
  journal={Transactions of the Association for Computational Linguistics},
  volume={9},
  pages={962--977},
  year={2021},
  publisher={MIT Press One Rogers Street, Cambridge, MA 02142-1209, USA journals-info~…}
}

@article{gneiting2007strictly,
  title={Strictly proper scoring rules, prediction, and estimation},
  author={Gneiting, Tilmann and Raftery, Adrian E},
  journal={Journal of the American statistical Association},
  volume={102},
  number={477},
  pages={359--378},
  year={2007},
  publisher={Taylor \& Francis}
}

@article{glenn1950verification,
  title={Verification of forecasts expressed in terms of probability},
  author={Glenn, W Brier and others},
  journal={Monthly weather review},
  volume={78},
  number={1},
  pages={1--3},
  year={1950},
  publisher={War Department, Office of the Chief Signal Officer}
}

@article{shazeer2017outrageously,
  title={Outrageously large neural networks: The sparsely-gated mixture-of-experts layer},
  author={Shazeer, Noam and Mirhoseini, Azalia and Maziarz, Krzysztof and Davis, Andy and Le, Quoc and Hinton, Geoffrey and Dean, Jeff},
  journal={arXiv preprint arXiv:1701.06538},
  year={2017}
}

@inproceedings{hershey2017cnn,
  title={CNN architectures for large-scale audio classification},
  author={Hershey, Shawn and Chaudhuri, Sourish and Ellis, Daniel PW and Gemmeke, Jort F and Jansen, Aren and Moore, R Channing and Plakal, Manoj and Platt, Devin and Saurous, Rif A and Seybold, Bryan and others},
  booktitle={2017 ieee international conference on acoustics, speech and signal processing (icassp)},
  pages={131--135},
  year={2017},
  organization={IEEE}
}

@inproceedings{gemmeke2017audio,
  title={Audio set: An ontology and human-labeled dataset for audio events},
  author={Gemmeke, Jort F and Ellis, Daniel PW and Freedman, Dylan and Jansen, Aren and Lawrence, Wade and Moore, R Channing and Plakal, Manoj and Ritter, Marvin},
  booktitle={2017 IEEE international conference on acoustics, speech and signal processing (ICASSP)},
  pages={776--780},
  year={2017},
  organization={IEEE}
}

@inproceedings{radford2021learning,
  title={Learning transferable visual models from natural language supervision},
  author={Radford, Alec and Kim, Jong Wook and Hallacy, Chris and Ramesh, Aditya and Goh, Gabriel and Agarwal, Sandhini and Sastry, Girish and Askell, Amanda and Mishkin, Pamela and Clark, Jack and others},
  booktitle={International conference on machine learning},
  pages={8748--8763},
  year={2021},
  organization={PmLR}
}

@article{graves2012long,
  title={Long short-term memory},
  author={Graves, Alex},
  journal={Supervised sequence labelling with recurrent neural networks},
  pages={37--45},
  year={2012},
  publisher={Springer}
}

@article{adam2014method,
  title={A method for stochastic optimization},
  author={Adam, Kingma DP Ba J and others},
  journal={arXiv preprint arXiv:1412.6980},
  volume={1412},
  number={6},
  year={2014}
}

@article{granese2021doctor,
  title={Doctor: A simple method for detecting misclassification errors},
  author={Granese, Federica and Romanelli, Marco and Gorla, Daniele and Palamidessi, Catuscia and Piantanida, Pablo},
  journal={Advances in Neural Information Processing Systems},
  volume={34},
  pages={5669--5681},
  year={2021}
}

@article{cai2025value,
  title={On the Value of Cross-Modal Misalignment in Multimodal Representation Learning},
  author={Cai, Yichao and Liu, Yuhang and Gao, Erdun and Jiang, Tianjiao and Zhang, Zhen and Hengel, Anton van den and Shi, Javen Qinfeng},
  journal={arXiv preprint arXiv:2504.10143},
  year={2025}
}

@inproceedings{ma2023calibrating,
  title={Calibrating multimodal learning},
  author={Ma, Huan and Zhang, Qingyang and Zhang, Changqing and Wu, Bingzhe and Fu, Huazhu and Zhou, Joey Tianyi and Hu, Qinghua},
  booktitle={International Conference on Machine Learning},
  pages={23429--23450},
  year={2023},
  organization={PMLR}
}
\bibliographystyle{icml2026}

\newpage
\appendix
\onecolumn

\begin{center}
    {\Large \textbf{Knowing When to Answer: Adaptive Confidence Refinement for \\ Reliable Audio-Visual Question Answering}}\\
    \vspace{0.5em}{\Large Appendix} \\
    \vspace{0.5em}
    \vspace{1.em}
\end{center}


{\Large \textbf{Index}}

\vspace{2mm}

\textbf{Appendix \ref{section:A}} contains the complete proofs for Theorems \ref{thm:msp_optimal}, \ref{thm:fusion_benefit}, and \ref{thm:adaptive_fusion}.

\textbf{Appendix \ref{section:A_1}} demonstrates the empirical results that satisfy the condition in Theorem \ref{thm:fusion_benefit}, ensuring our method consistently achieves lower MSE and better relative rankings among examples compared to MSP.

\textbf{Appendix \ref{section:A_2}} provides a more detailed review of related work.

\textbf{Appendix \ref{section:B}} provides further discussion on our ablation studies.

\textbf{Appendix \ref{section:B_1}} provides further analysis and discussion regarding the MSP and our ACR with $\alpha$ distribution analysis.

\textbf{Appendix \ref{section:B_1_1}} discusses computational complexity and overhead of our proposed ACR.

\textbf{Appendix \ref{section:B_2}} presents the reproduction results for the three AVQA models (QA-TIGER, ST-AVQA and TSPM) on the original MUSIC-AVQA test set.

\textbf{Appendix \ref{section:C}} presents an experiment with data augmentation for MSP.

\textbf{Appendix \ref{section:D}} reports the additional performance on MUSIC-AVQA-R and MUSIC-AVQA-v2 (trained on the bias dataset and evaluated across both balanced and biased test sets).

\textbf{Appendix \ref{section:E}} provides additional qualitative results.

\textbf{Appendix \ref{section:F}} evaluates the generalization capability of the thresholds.

\textbf{Appendix \ref{section:G}} provides further details on the dataset splits used in our experimental setup.

\textbf{Appendix \ref{section:H}} provides comprehensive details regarding the model architecture, hyperparameters and implementation.

\textbf{Appendix \ref{section:H_1}} reports the mean and standard deviation for the TSPM baseline on the MUSIC-AVQA dataset, computed across five random seeds.

\textbf{Appendix \ref{section:I}} examines the relationship between calibration and conformal prediction within the context of our \textit{Reliable} AVQA framework and recent literature.

\textbf{Appendix \ref{section:I_1}} provides further discussion on the broader implications of our work and outlines potential directions for future research.

\textbf{Appendix \ref{section:J}} outlines the broader impact of our work, including potential social and ethical considerations.

\vspace{60mm}

\section{Proofs}
\label{section:A}

\subsection{Proof of Theorem \ref{thm:msp_optimal}}

In this section, we will provide the full proof of Theorem \ref{thm:msp_optimal} in the main paper (Theorem \ref{thm:msp_optimal_appendix} in the appendix).

\begin{assumption}[MAP Prediction]
\label{assum:calibration}
We assume the predictor outputs the maximum a posteriori (MAP) class:\\
\[\hat{y} = \arg \max\limits_{k \in \mathcal{Y}} \Prob(Y = k \mid \boldsymbol{x}),\]
\\
as is standard for softmax-based classifiers.
\end{assumption}

\begin{theorem}[MSP Optimality under Strong Calibration]
\label{thm:msp_optimal_appendix}
Consider a classifier $f(\boldsymbol{x})$ that outputs class probabilities $\Prob(Y=k|\boldsymbol{x})$ and predicts $\hat{y} = \arg\max_k \Prob(Y=k|\boldsymbol{x})$ (Assumption \ref{assum:calibration}).
If the MSP confidence function $C_{\mathrm{M}}(\boldsymbol{x}) = \max_k \Prob(Y=k|\boldsymbol{x})$ is strongly calibrated (Definition~\ref{def:calibration}), then 1) MSP equals the Bayes-optimal selection function: $C_{\mathrm{M}}(\boldsymbol{x}) = g^*(\boldsymbol{x}) = \Prob[c(\boldsymbol{x}) = 1 \mid \boldsymbol{x}]$. 2) Ranking examples by MSP induces the Bayes-optimal ranking for selective prediction. 3) This ranking minimizes risk at every coverage level, thereby achieving minimum AURC.
\end{theorem}

\begin{proof}
    Selective prediction is equivalent to binary classification: distinguishing correct ($c(\boldsymbol{x})=1$) from incorrect ($c(\boldsymbol{x})=0$) predictions. The Bayes-optimal classifier uses the likelihood ratio, which is monotonically related to the posterior:
\begin{equation}
    s^*(\boldsymbol{x}) = \Prob[c(\boldsymbol{x}) = 1 \mid \boldsymbol{x}] = \Prob[\hat{y} = y^* \mid \boldsymbol{x}]
\end{equation}

Under Strong Calibration (Definition~\ref{def:calibration} and Assumption \ref{assum:calibration}), for a model predicting class $\hat{y} = \arg\max\limits_{k} \Prob(Y=k \mid \boldsymbol{x})$:
\begin{equation}
    C_{\mathrm{M}}(\boldsymbol{x}) = \max_k \Prob(Y = k \mid \boldsymbol{x}) = \Prob(Y = \hat{y} \mid \boldsymbol{x})
\end{equation}

Since $\hat{y}$ is the predicted class, if the model is correct ($\hat{y} = y^*$), then $\Prob(Y = \hat{y} \mid \boldsymbol{x}) = \Prob(Y = y^* \mid \boldsymbol{x})$. Under strong calibration, this equals the true posterior:
\begin{equation}
    C_{\mathrm{M}}(\boldsymbol{x}) = \Prob[\hat{y} = y^* \mid \boldsymbol{x}] = s^*(\boldsymbol{x})
\end{equation}

Since $s^*(\boldsymbol{x})$ equals the true probability of correctness, ranking by $s^*$ perfectly orders samples by their likelihood of being correct. For any pair of samples $\boldsymbol{x}_1, \boldsymbol{x}_2$:
\begin{equation}
    s^*(\boldsymbol{x}_1) > s^*(\boldsymbol{x}_2) \iff \Prob[c(\boldsymbol{x}_1)=1] > \Prob[c(\boldsymbol{x}_2)=1]
\end{equation}
This ranking minimizes risk at every coverage level: selecting the top-$k$ samples by $s^*$ yields the $k$ samples most likely to be correct. Thus, MSP $= s^*$ equals the Bayes-optimal selection function, achieving minimum AURC by definition.
\end{proof}

\subsection{Proof of Theorem \ref{thm:fusion_benefit}}

In this section, we will provide the full proof of Theorem \ref{thm:fusion_benefit} in the main paper (Theorem \ref{thm:fusion_benefit_appendix} in the appendix).

\begin{theorem}[Fusion Benefit for Confidence Estimation]
\label{thm:fusion_benefit_appendix}
If the error cross-moment satisfies the following condition:
\begin{equation}
    \sigma_{\mathrm{MR}} < \min(\sigma^2_{\mathrm{M}}, \sigma^2_{\mathrm{R}})
    \label{eq:fusion_condition_main}
\end{equation}
then there exists a unique optimal fixed fusion weight $\bar{\alpha}^* \in (0, 1)$ such that the fused confidence:
\begin{equation}
    C_{\bar{\alpha}^*}(\boldsymbol{x}) = \bar{\alpha}^* C_{\mathrm{M}}(\boldsymbol{x}) + (1 - \bar{\alpha}^*) C_{\mathrm{R}}(\boldsymbol{x}) \nonumber
\end{equation}
achieves strictly lower MSE than either estimator alone. Moreover, by Proposition~\ref{prop:mse_ranking}, $C_{\bar{\alpha}^*}$ provides a closer approximation to the Bayes-optimal selection function $g^*$, thereby improving the relative ranking of predictions compared to using MSP alone.
The optimal weight is given by:
\begin{equation}
    \bar{\alpha}^* = \frac{\sigma^2_{\mathrm{R}} - \sigma_{\mathrm{MR}}}{\sigma^2_{\mathrm{M}} + \sigma^2_{\mathrm{R}} - 2\sigma_{\mathrm{MR}}} \nonumber
    \label{eq:alpha_star}
\end{equation}
\end{theorem}

\textbf{Proof.} Define the MSE of $C_\alpha = \alpha C_{\mathrm{M}} + (1-\alpha) C_{\mathrm{R}}$ (we ommit input $\boldsymbol{x}$ for simplicity):
\begin{align}
    J(\alpha) &= \E[(C_\alpha - c)^2] = \E[(\alpha \epsM + (1-\alpha)\epsR)^2] \\
    &= \alpha^2 \E[\epsM^2] + (1-\alpha)^2 \E[\epsR^2] + 2\alpha(1-\alpha)\E[\epsM\epsR] \\
    &= \alpha^2 \sigma^2_{\mathrm{M}} + (1-\alpha)^2 \sigma^2_{\mathrm{R}} + 2\alpha(1-\alpha)\sigma_{\mathrm{MR}}
\end{align}

\textbf{First-order condition.} 

Setting $\frac{dJ}{d\alpha} = 0$:
\begin{equation}
    2\alpha \sigma^2_{\mathrm{M}} - 2(1-\alpha)\sigma^2_{\mathrm{R}} + 2(1-2\alpha)\sigma_{\mathrm{MR}} = 0
\end{equation}

Rearranging:
\begin{equation}
    \alpha(\sigma^2_{\mathrm{M}} + \sigma^2_{\mathrm{R}} - 2\sigma_{\mathrm{MR}}) = \sigma^2_{\mathrm{R}} - \sigma_{\mathrm{MR}}
\end{equation}

Thus:
\begin{equation}
    \bar{\alpha}^* = \frac{\sigma^2_{\mathrm{R}} - \sigma_{\mathrm{MR}}}{\sigma^2_{\mathrm{M}} + \sigma^2_{\mathrm{R}} - 2\sigma_{\mathrm{MR}}}
\end{equation}

\textbf{Second-order condition.}
\begin{equation}
    \frac{d^2J}{d\alpha^2} = 2(\sigma^2_{\mathrm{M}} + \sigma^2_{\mathrm{R}} - 2\sigma_{\mathrm{MR}}) = 2\E[(\epsM - \epsR)^2] > 0
\end{equation}

The denominator is strictly positive when estimators differ, confirming $\bar{\alpha}^*$ is a minimum.

\textbf{Interior solution condition.} For $\bar{\alpha}^* \in (0, 1)$:
\begin{itemize}
    \item $\bar{\alpha}^* > 0$: Requires numerator $> 0$, i.e., $\sigma^2_{\mathrm{R}} > \sigma_{\mathrm{MR}}$
    \item $\bar{\alpha}^* < 1$: Requires numerator $<$ denominator, i.e., $\sigma^2_{\mathrm{R}} - \sigma_{\mathrm{MR}} < \sigma^2_{\mathrm{M}} + \sigma^2_{\mathrm{R}} - 2\sigma_{\mathrm{MR}}$, which simplifies to $\sigma_{\mathrm{MR}} < \sigma^2_{\mathrm{M}}$
\end{itemize}

Both conditions hold when $\sigma_{MR} < \min(\sigma^2_M, \sigma^2_R)$.

\textbf{Strict improvement.} At interior minimum $\bar{\alpha}^* \in (0,1)$:
\begin{equation}
    J(\bar{\alpha}^*) < J(0) = \sigma^2_{\mathrm{R}} \quad \text{and} \quad J(\bar{\alpha}^*) < J(1) = \sigma^2_{\mathrm{M}}
\end{equation}

Thus $J(\bar{\alpha}^*) < \min(\sigma^2_{\mathrm{M}}, \sigma^2_{\mathrm{R}})$.

\textbf{Implication for Ranking}. The MSE of any confidence estimator $s$ with respect to correctness $c(\boldsymbol{x})$ admits the decomposition \cite{glenn1950verification}:

\begin{equation}
    \E[(s - c)^2] = \E[(s - s^*)^2] + \E[s^*(1-s^*)]
\end{equation}
where $s^*(\boldsymbol{x}) = \Prob[c(\boldsymbol{x})=1 \mid \boldsymbol{x}]$ is the Bayes-optimal selection function. Since the second term is constant (irreducible uncertainty), the reduction $J(\bar{\alpha}^*) < \sigma^2_{\mathrm{M}}$ implies:
\begin{equation}
    \E[(C_{\bar{\alpha}^*} - s^*)^2] < \E[(C_{\mathrm{M}} - s^*)^2]
\end{equation}
Thus, the fused confidence $C_{\bar{\alpha}^*}$ is a closer $L^2$-approximation to $s^*$ than MSP alone, yielding improved relative ranking for selective prediction.

\subsection{Proof of Theorem \ref{thm:adaptive_fusion}}

In this section, we will provide the full proof of Theorem \ref{thm:adaptive_fusion} in the main paper (Theorem \ref{thm:adaptive_fusion_appendix} in the appendix).

\begin{theorem}[Optimality of Input-Adaptive Confidence Fusion]
\label{thm:adaptive_fusion_appendix}
Let the adaptive confidence score be
$
C_{\mathrm{ACR}}(\boldsymbol{x}) = \alpha^{\dagger}(\boldsymbol{x}) C_{\mathrm{M}}(\boldsymbol{x}) + (1-\alpha^{\dagger}(\boldsymbol{x})) C_{\mathrm{R}}(\boldsymbol{x}),$
where $\alpha^{\dagger}(\boldsymbol{x}) \in [0,1]$ is a input-adaptive fusion weight.
Assume $C_{\mathrm{M}}(\boldsymbol{x})$ and $C_{\mathrm{R}}(\boldsymbol{x})$ are measurable functions of $\boldsymbol{x}$ and the hypothesis class $\mathcal{H} = \{\alpha^{\dagger}: \mathcal{X} \to [0,1] \mid \alpha^{\dagger} \text{ measurable}\}$ has sufficient capacity to represent the Bayes-optimal fusion weight $\alpha^*(\boldsymbol{x})$.
When trained $\mathcal{R}$-AVQA model using binary cross-entropy loss on correctness labels (0/1), any risk minimizer satisfies
\[
C_{\mathrm{ACR}}(\boldsymbol{x}) = \Prob(c(\boldsymbol{x})=1 \mid \boldsymbol{x})
\quad \text{a.e.}
\]

Since $C_{\mathrm{ACR}} = g^*$ a.e., ranking by $C_{\mathrm{ACR}}$ gets closer to the Bayes-optimal ranking for selective prediction, providing lower AURC.
If the pointwise optimal fusion weight $\alpha^*(\boldsymbol{x}) = \arg\min_\alpha \E\bigl[\bigl(C_\alpha(\boldsymbol{x}) - c(\boldsymbol{x}) \bigl)^2 \mid \boldsymbol{x} \bigl]$ varies across inputs (i.e., $\mathrm{Var}[\alpha^*(\boldsymbol{x})] > 0$), then any fixed weight $\lambda \in [0,1]$ to $C_{\lambda}$ is strictly suboptimal in both MSE and ranking. 
\end{theorem}

\begin{proof}
Binary cross-entropy is a strictly proper scoring rule \cite{gneiting2007strictly}.
Therefore, its expected risk is uniquely minimized when the predicted confidence
equals the conditional probability of correctness:
\[
C_{\text{ACR}}(\boldsymbol{x}) = \Prob(c(\boldsymbol{x})=1 \mid \boldsymbol{x}) = s^*(\boldsymbol{x}) \quad \text{a.e.}
\]

Since $C_{\mathrm{ACR}} = s^*$ almost everywhere, ranking by $C_{\mathrm{ACR}}$ gets closer to the ordering as ranking by the Bayes-optimal selection function. By the argument in Theorem~\ref{thm:msp_optimal}, this ranking minimizes risk at every coverage level, achieving a lower AURC that is much closer to the minimum AURC.

Restricting $\alpha^{\dagger}(\boldsymbol{x})$ to be constant constrains the hypothesis class to fixed convex combinations of $C_{\mathrm{M}}$ and $C_{\mathrm{R}}$.
Allowing $\alpha^{\dagger}(\boldsymbol{x})$ to depend on $\boldsymbol{x}$ strictly enlarges this function class.
If the Bayes-optimal fusion weight varies with $\boldsymbol{x}$, no constant $\lambda$ can match the optimal solution almost everywhere, making it a suboptimal solution.

To see why fixed fusion yields suboptimal ranking, consider that when the pointwise optimal weight $\alpha^*(\boldsymbol{x})$ varies across inputs, any fixed $\lambda$ necessarily produces:
\begin{equation}
    C_\lambda(\boldsymbol{x}) \neq s^*(\boldsymbol{x}) \quad \text{on a set of positive measure}
\end{equation}
This discrepancy causes ranking errors: there exist pairs $(\boldsymbol{x}_1, \boldsymbol{x}_2)$ where $C_\lambda$ ranks them incorrectly relative to their true correctness probabilities, increasing AURC above the Bayes-optimal minimum.

Hence, input-adaptive fusion strictly improves expected risk and ranking performance over fixed fusion.
\end{proof}

\section{Empirical Verification for Theorem 3.5}
\label{section:A_1}

\begin{table}
\centering
\caption{Empirical verification for Theorem \ref{thm:fusion_benefit}. Our proposed ACR consistently satisfies the condition in Eq. \ref{eq:condition_3.5_appendix}, thereby ensuring that our method strictly reduces MSE and improves selective prediction.}
\label{tab:veryfied_theorem}
\begin{tabular}{lllccc} 
\toprule
Dataset                           & Dataset Description                          &     & QA-TIGER & ST-AVQA & TSPM  \\ 
\hline
\multirow{3}{*}{MUSIC-AVQA}       & \multirow{3}{*}{N/A}                         & $\sigma_{\text{M}}^{2}$  & 0.1542        & 0.1864       & 0.1564     \\
                                  &                                              & $\sigma_{\text{R}}^{2}$  & 0.4517        & 0.2225       & 0.1960     \\
                                  &                                              & $\sigma_{\mathrm{MR}}$ & \textbf{0.0966}        & \textbf{0.1648}       & \textbf{0.1384}     \\ 
\hline
\multirow{6}{*}{MUSIC-AVQA-R}     & \multirow{3}{*}{Head}                        & $\sigma_{\text{M}}^{2}$  & 0.2044        & 0.2038       & 0.1833     \\
                                  &                                              & $\sigma_{\text{R}}^{2}$  & 0.3847        & 0.2557       & 0.2323     \\
                                  &                                              & $\sigma_{\mathrm{MR}}$ & \textbf{0.1183}        & \textbf{0.1865}       & \textbf{0.1721}     \\ 
\cline{2-6}
                                  & \multirow{3}{*}{Tail}                        & $\sigma_{\text{M}}^{2}$  & 0.2092        & 0.2375       & 0.1983     \\
                                  &                                              & $\sigma_{\text{R}}^{2}$  & 0.4168        & 0.3574       & 0.2516     \\
                                  &                                              & $\sigma_{\mathrm{MR}}$ & \textbf{0.1176}        & \textbf{0.2166}       & \textbf{0.1850}     \\ 
\hline
\multirow{12}{*}{MUSIC-AVQA-v2.0} & \multirow{3}{*}{Train Balance, Test Balance} & $\sigma_{\text{M}}^{2}$  & 0.1754        & 0.1897       & 0.1848     \\
                                  &                                              & $\sigma_{\text{R}}^{2}$  & 0.7513        & 0.2951       & 0.2318     \\
                                  &                                              & $\sigma_{\mathrm{MR}}$ & \textbf{0.0535}        & \textbf{0.1597}       & \textbf{0.1595}     \\ 
\cline{2-6}
                                  & \multirow{3}{*}{Train Balance, Test Biased}  & $\sigma_{\text{M}}^{2}$  & 0.1737        & 0.1875       & 0.1830     \\
                                  &                                              & $\sigma_{\text{R}}^{2}$  & 0.7535        & 0.3060       & 0.2388     \\
                                  &                                              & $\sigma_{\mathrm{MR}}$ &  \textbf{0.0524}        & \textbf{0.1581}       & \textbf{0.1601}     \\ 
\cline{2-6}
                                  & \multirow{3}{*}{Train Biased, Test Balance}  & $\sigma_{\text{M}}^{2}$  & 0.1912        & 0.2099       & 0.2039     \\
                                  &                                              & $\sigma_{\text{R}}^{2}$  & 0.7433        & 0.2587       & 0.2426     \\
                                  &                                              & $\sigma_{\mathrm{MR}}$ & \textbf{0.0528}        & \textbf{0.1684}       & 0.1636     \\ 
\cline{2-6}
                                  & \multirow{3}{*}{Train Biased, Test Biased}   & $\sigma_{\text{M}}^{2}$  & 0.1676        & 0.1827       & 0.1808     \\
                                  &                                              & $\sigma_{\text{R}}^{2}$  & 0.7763        & 0.2647       & 0.2403     \\
                                  &                                              & $\sigma_{\mathrm{MR}}$ & \textbf{0.0531}        & \textbf{0.1535}       & \textbf{0.1488}     \\
\bottomrule
\end{tabular}
\end{table}

We verify the conditions of Theorem \ref{thm:fusion_benefit} by analyzing the relationship between MSP and Residual Risk Head (RRH) errors. The key condition for fusion benefit is:
\begin{equation}
\label{eq:condition_3.5_appendix}
    \sigma_{\mathrm{MR}} = \mathbb{E}\!\left[\epsilon_{\text{M}} \cdot \epsilon_{\text{R}}\right]
< \min\!\left(\sigma_{\text{M}}^{2}, \sigma_{\text{R}}^{2}\right)
\end{equation}
where $\epsilon_{\text{M}}$ and $\epsilon_{\text{R}}$ denote the errors of MSP and RRH in ranking correctness, respectively, and $\sigma_{\text{M}}^{2}$, $\sigma_{\text{R}}^{2}$ represent their respective error variances.
This condition states that fusion provides benefit when the two estimators errors are less correlated than their individual error magnitudes. In other words, when MSP and RRH make complementary errors (low correlation), their fusion can leverage their respective strengths to improve overall selective prediction performance. Tab. \ref{tab:veryfied_theorem} presents comprehensive verification across three AVQA backbones and three datasets. We find a universal condition satisfaction, in which our proposed ACR consistently satisfies the fusion benefit condition with substantial margins. 
These results provide empirical support for the theoretical analysis in Theorem~\ref{thm:fusion_benefit}, and indicate that the performance gains of ACR stem from a systematic reduction in confidence estimation error through complementary uncertainty signals.

\section{More Related Works}
\label{section:A_2}

\textbf{Calibration}. Calibration is an alternative perspective on uncertainty estimation, which assesses the alignment between a model's confidence levels and its actual accuracy. The distinction between calibration and selective prediction becomes apparent when we consider a model that is accurate for a certain percentage of inputs but assigns the same confidence level to every prediction. For example, if a model is accurate for $x\%$ of inputs but consistently provides an $x\%$ confidence rating regardless of the input, it is perfectly calibrated. However, this model fails to differentiate between correct and incorrect predictions, meaning it offers a less useful signal for abstaining from predictions.

In unimodal classification, techniques such as temperature scaling \cite{guo2017calibration} and Platt scaling \cite{platt1999probabilistic} improve calibration by adjusting classifier outputs to generate calibrated probabilities based on a validation set. In this study, we examine the effectiveness of vector scaling (VS) \cite{guo2017calibration}, which is a multi-class extension of Platt scaling, in enhancing selective prediction performance.

We argue that reliability in multimodal reasoning tasks, such as Audio-Visual Question Answering (AVQA), depends on two essential components: (i) well-calibrated confidence estimates, which involve achieving a lower expected calibration error (ECE), and (ii) the ability to identify and abstain from unreliable predictions, meaning knowing when to provide an answer and when to refrain from answering. Our results demonstrate that our proposed strategy, Adaptive Confidence Refinement (ACR), significantly improves both components compared to other methods, including maximum softmax probability, Monte Carlo dropout, and vector scaling.

\textbf{Monte Carlo dropout}. Ensemble-based techniques have been explored for selective prediction, where decisions to reject predictions are typically based on the uncertainty statistics of the ensemble \cite{freund2004generalization,varshney2011risk, lakshminarayanan2017simple}. However, these methods require multiple pretrained models, making them impractical due to the high cost of training a sufficiently large ensemble.

We propose using Monte Carlo dropout (MCD) \cite{gal2016dropout, geifman2019selectivenet} as an efficient alternative to improve selective prediction performance. MCD estimates predictive uncertainty by conducting multiple stochastic forward passes with dropout enabled during inference. This method perturbs the network's predictions without the need for multiple independently trained models.
Although MCD is widely used for uncertainty estimation problems, it can be readily adapted for this purpose by applying a rejection threshold based on the maximum softmax probability from the predictions. Empirical studies indicate that MCD can enhance calibration and robustness relative to deterministic models, thereby improving selective performance in certain scenarios.
However, while MCD often improves absolute uncertainty estimation, it does not guarantee optimal sample ranking for selective prediction \cite{geifman2017selective, lakshminarayanan2017simple}. Specifically, the ordering produced by MCD uncertainty scores may still be suboptimal for minimizing selective risk, especially when evaluated across the entire risk-coverage spectrum. Consequently, MCD may offer only limited advantages over simpler confidence-based methods, particularly with respect to inference cost and scalability.

In our experiments, we use MCD as a robust uncertainty-based baseline for selective prediction. Following standard practice, we enable dropout layers during inference and aggregate predictions from multiple stochastic passes. Consistent with previous research, we find that while MCD can enhance calibration, its selective performance is often limited by inference costs and imperfect confidence rankings. This limitation encourages us to explore more efficient alternatives that focus on directly refining confidence scores. However, we also note that MCD's performance does not consistently improve across all scenarios, including in-distribution (ID) and out-of-distribution (OOD) cases and in biased settings.

\textbf{Doctor score for selective prediction}. Traditional selective prediction relies on the Maximum Softmax Probability (MSP) as a measure of confidence. However, MSP is often criticized for its susceptibility to overconfidence and its failure to utilize the information contained in the ```tail" of the predictive distribution. To address these limitations, a recent study \citet{granese2021doctor} introduces \textsc{Doctor}, a post-hoc, training-free method that shifts the focus from the maximum logit to the overall distribution of the softmax vector $\mathbf{p}(\boldsymbol{x})$. By defining the confidence score through the quadratic sum as follows:
\begin{equation}
    \mathcal{C}_{\text{DOC}}(\boldsymbol{x}) = \sum_{k=1}^K p_k(\boldsymbol{x})^2,
\end{equation}

the method implicitly leverages the relationship between the Gini impurity and the model's predictive uncertainty. On the probability simplex, this metric acts as a distance measure that penalizes high-entropy distributions, where the model's mass is more aggressively diffused across multiple classes than under MSP. This makes \textsc{Doctor} a robust and computationally efficient baseline for selective classification, particularly in black-box scenarios where internal feature representations or training-time statistics are inaccessible. However, Doctor score shares the same limitation as MSP: it operates solely on answer logits, thereby typically ignoring the valuable uncertainty signal from intermediate representations.

\section{More Discussion on Ablation Study}
\label{section:B}

We conduct extensive ablation studies to validate the effectiveness of our proposed components. We utilize the QA-TIGER backbone for all experiments in the ablation study section, including those in both the main paper and appendix on the MUSIC-AVQA dataset.

\begin{table}[h]
\caption{Different architecture for two learned heads: Residual Risk Head and Confidence Gating Head on MUSIC-AVQA test set with QA-TIGER}
\label{tab:acr_arch}
\centering
\begin{tabular}{lllllll} 
\toprule
Architecture       & $\mathcal{C}$@1\%  & $\mathcal{C}$@5\%  & $\mathcal{C}$@10\% & $\mathcal{C}$@20\% & AURC & ECE   \\ 
\midrule
1-layer Linear     & 10.24 & 38.72 & 59.87 & 90.41 & 8.93 & 6.69  \\
2-layer MLP        & 14.96 & 40.17 & 60.83 & 90.98 & 8.56 & 5.12  \\
3-layer MLP (Ours) & \textbf{17.21} & \textbf{43.07} & \textbf{61.39} & \textbf{91.96} & \textbf{8.41} & \textbf{4.88}  \\
4-layer MLP        & \underline{16.74} & \underline{42.77} & \underline{61.30} & \underline{91.68} & \underline{8.47} & \underline{4.92}  \\
\bottomrule
\end{tabular}
\end{table}

\textbf{ACR architecture}. As discussed in Sect. \ref{sec:ablation_study}, we conducted experiments using various architectures for our ACR with QA-TIGER models on the MUSIC-AVQA dataset. For the sake of simplicity, we employed the same architecture for both learned heads across all our experiments, including the main experiments and ablation studies. We found that this consistency generally yields the best performance. Our ACR utilizes a 3-layer multi-layer perceptron (MLP). In Table \ref{tab:acr_arch}, we observe that using a simpler, 1-layer linear architecture for the ACR leads to significantly degraded performance, particularly at the low-risk level. Conversely, a more complex design featuring a 4-layer MLP results in slightly lower performance compared to our chosen design across all metrics. Considering these results alongside those in Tab. \ref{tab:ablation_input_selector}, we conclude that input representations and training objectives are crucial. Therefore, efforts to improve the performance of the learned selection function should focus primarily on these factors rather than on designing more complex architectures.

\begin{table}
\centering
\caption{Ablation study on ACR learned heads. \textbf{RRH}: Residual Risk Head, \textbf{CGH}: Confidence Gating Head.}
\label{tab:ablation_2_heads}
\begin{tabular}{cccccccc} 
\toprule
\multicolumn{2}{c}{Components} & \multicolumn{4}{c}{$\mathcal{C}@\mathcal{R}$ $\uparrow$}           & \multirow{2}{*}{AURC $\downarrow$} & \multirow{2}{*}{ECE $\downarrow$}  \\ 
\cmidrule(r){1-2} \cmidrule(lr){3-6}
RRH & CGH                      & $\mathcal{C}$@1\%           & $\mathcal{C}$@5\%           & $\mathcal{C}$@10\%          & $\mathcal{C}$@20\%          &                                    &                                    \\ 
\midrule
\ding{55}  & \ding{51}                       & 14.57          & 41.94          & 60.77          & 91.63          & 8.61                               & 7.82                               \\
\ding{51}  & \ding{55}                       & 13.78          & 40.11          & 60.31          & 91.24          & 8.72                               & 9.14                               \\
\ding{51}  & \ding{51}                       & \textbf{17.21} & \textbf{43.07} & \textbf{61.39} & \textbf{91.96} & \textbf{8.41}                      & \textbf{4.88}                      \\
\bottomrule
\end{tabular}
\end{table}

\textbf{Effectiveness of two learned heads}. We conducted experiments to validate our two-head design and the strategy of correcting the MSP-based confidence rather than replacing it entirely. As shown in Tab. \ref{tab:ablation_2_heads}, using both heads results in more optimal selective prediction performance across multiple metrics. These metrics include maximum coverage at target risk levels, area under the recovery curve (AURC), and expected calibration error (ECE). Notably, removing the Residual Risk Head (RRH) leads to a significant decline in performance, highlighting its essential role in correcting error cases associated with the base confidence signal provided by the MSP method.

Overall, our design represents an upgrade and integration of \textbf{calibration} and \textbf{selection} methods, enabling us to harness the strengths of both approaches. Specifically, when the RRH component is removed, ACR functions solely as a calibration method \cite{platt1999probabilistic}. Conversely, if the CGH component is removed, it operates as a selector, as in prior works \cite{geifman2017selective, whitehead2022reliable}. From our ablation study, we derive several insights:
\begin{enumerate}
    \item Removing the Residual Risk Head (RRH) simplifies our framework to a calibration-based approach. While this variant shows a slight improvement over the pure MSP baseline (as reported in Tab. \ref{tab:car_music_avqa}) across all evaluated metrics, the gains are marginal. This suggests that achieving significant improvements using only MSP-based confidence is challenging.
    \item Eliminating the Confidence Gating Head (CGH) transforms our framework into a selector-based method. Interestingly, although previous research \cite{geifman2019selectivenet, whitehead2022reliable} indicates that selector-based designs substantially outperform MSP, our findings demonstrate that this variant may even decrease performance compared to MSP. This discrepancy highlights that techniques effective in other selective prediction scenarios cannot be directly applied to Audio-Visual Question Answering (AVQA). Furthermore, it emphasizes the unique challenges posed by AVQA as a trimodal multimodal reasoning task, which fundamentally differs from unimodal or bimodal settings.
\end{enumerate}

\textbf{Effectiveness of adaptive weighting fusion}. Tab. \ref{tab:effect_fusion} analyzes the impact of fusion weight $\alpha$ on $\mathcal{R}$-AVQA performance by comparing fixed fusion rules with our input-adaptive fusion mechanism. As $\alpha$ increases from 0.1 to 0.8, performance improves consistently across all coverage levels and calibration metrics, indicating that MSP serves as a strong primary confidence signal. However, performance degrades or saturates when $\alpha$ is further increased to 0.9, suggesting that a single global fusion weight cannot simultaneously handle both easy and hard inputs. In contrast, ACR with input-adaptive fusion achieves the best results across all metrics, including $\mathcal{C}$@$\mathcal{R}$, AURC, and ECE. This confirms that the optimal fusion weight varies across inputs, allowing $\alpha (\boldsymbol{x})$ to adapt dynamically, and yields strictly better selective risk behaviour than any fixed fusion strategy.

\textbf{Comparison between ACR and selector-based strategies}. As discussed in Section \ref{sec:proposed_method}, our ACR introduces a comprehensive confidence estimation framework that integrates both calibration and selector-based strategies. In the main paper, we explicitly compare our method to the calibration approach known as vector scaling. In this section, we extend our analysis by comparing our ACR with two well-known selector-based methods: SelectiveNet \cite{geifman2019selectivenet} and Selector \cite{whitehead2022reliable}. 
SelectiveNet is an end-to-end framework that employs a triple-head architecture to jointly optimize features, predictions, and selection scores, while enforcing a specific ``coverage" constraint during training. However, this joint training may not be well-suited for the $\mathcal{R}$-AVQA task since AVQA models generally use pretrained models for audio, video, and text encoders. Consequently, they rarely train everything from scratch. Therefore, applying SelectiveNet, which is designed for training from scratch, could negatively impact selective prediction performance in the context of $\mathcal{R}$-AVQA.
On the other hand, the Selector method acts as an external ``referee," leveraging the alignment between visual and textual features for selective prediction in Visual Question Answering (VQA). This results in a more stable training process compared to SelectiveNet when applied to $\mathcal{R}$-AVQA.

Tab. \ref{tab:selectivenet_ablation} provides a direct comparison of our ACR with SelectiveNet and Selector across all metrics on the MUSIC-AVQA dataset using the QA-TIGER backbone. The results demonstrate that simply implementing SelectiveNet in the $\mathcal{R}$-AVQA context leads to poor selective prediction performance, even worse than the naive MSP approach. While the Selector method performs better than SelectiveNet, it still falls short compared to our ACR method. This suggests that relying on a single additional head, as in Selector, is insufficient to capture the more challenging cross-modal uncertainty signals inherent in the AVQA task. In contrast, our method consistently outperforms both selector-based strategies and the MSP approach, indicating its robustness and suitability for the $\mathcal{R}$-AVQA task.

\begin{table}[]
\centering
\setlength{\tabcolsep}{3pt}
\caption{Ablation for comparison between our ACR and selector-based methods. Metrics: Coverage at target risk ($\mathcal{C}$@$\mathcal{R}$) $\uparrow$, AURC $\downarrow$, and ECE $\downarrow$ for all selection functions on MUSIC-AVQA dataset. All in \%.}
\label{tab:selectivenet_ablation}
\begin{tabular}{llcccccc}
\toprule
Model $f$ & Sel. Func. $g$ & $\mathcal{C}$@1\% & $\mathcal{C}$@5\% & $\mathcal{C}$@10\% & $\mathcal{C}$@20\% & AURC & ECE \\
\midrule
\multirow{3}{*}{QA-TIGER}   & SelectiveNet \cite{geifman2019selectivenet}                             & 9.72 & 37.74 & 56.91 & 87.26 & 9.92                 & 11.31                 \\
                          & Selector \cite{whitehead2022reliable}                       & \underline{12.27}      & \underline{40.31}      & \underline{58.83}       & \underline{89.75}       & \underline{9.02}     & \underline{10.19}     \\
                          & ACR (Ours) & \textbf{17.21} & \textbf{43.07} & \textbf{61.39}  & \textbf{91.96}  & \textbf{8.41} & \textbf{4.88} \\
\bottomrule
\end{tabular}
\end{table}

\begin{table}
\caption{Effectiveness of input-adaptive fusion mechanism. Metrics: Coverage at target risk ($\mathcal{C}$@$\mathcal{R}$) $\uparrow$, AURC $\downarrow$, and ECE $\downarrow$ for all selection functions on MUSIC-AVQA dataset. All in \%.}
\label{tab:effect_fusion}
\centering
\begin{tabular}{lcccccc} 
\toprule
$\alpha$ value                             & $\mathcal{C}$@1\% & $\mathcal{C}$@5\% & $\mathcal{C}$@10\% & $\mathcal{C}$@20\% & AURC & ECE  \\ 
\midrule
0.1                               & 10.69     & 38.77     & 57.41      & 88.94      & 9.83    & 10.64    \\
0.5                               & 12.66     & 40.11     & 59.91      & 89.48      & 9.62    & 9.75    \\
0.7                               & 14.02     & 41.42     & 59.86      & 89.83      & 8.78    & 9.67    \\
0.8                               & \underline{15.73}     & \underline{42.08}     & \underline{61.17}      & \underline{90.82}      & \underline{8.46}    & \underline{6.79}    \\
0.9                               & 14.47     & 41.99     & 60.71      & 90.35      & 8.49    & 8.11    \\
ACR (Ours, input-adaptive fusion) & \textbf{17.21}     & \textbf{43.07}     & \textbf{61.39}      & \textbf{91.96}      & \textbf{8.41}    & \textbf{4.88}    \\
\bottomrule
\end{tabular}
\end{table}

\section{Additional Analysis and Discussion of MSP and ACR}
\label{section:B_1}

We provide a comparative analysis of Maximum Softmax Probability (MSP) and our proposed method, Adaptive Confidence Refinement (ACR), for selective prediction in Audio-Visual Question Answering (AVQA). This comparison focuses on distribution, calibration, the distinction between correct and incorrect predictions, and the quality of selective predictions. We evaluate using the TSPM backbone across all three datasets.

\begin{figure*}[hpt]
    \vspace{-2mm}
    \centering
    \includegraphics[width=0.98\linewidth]{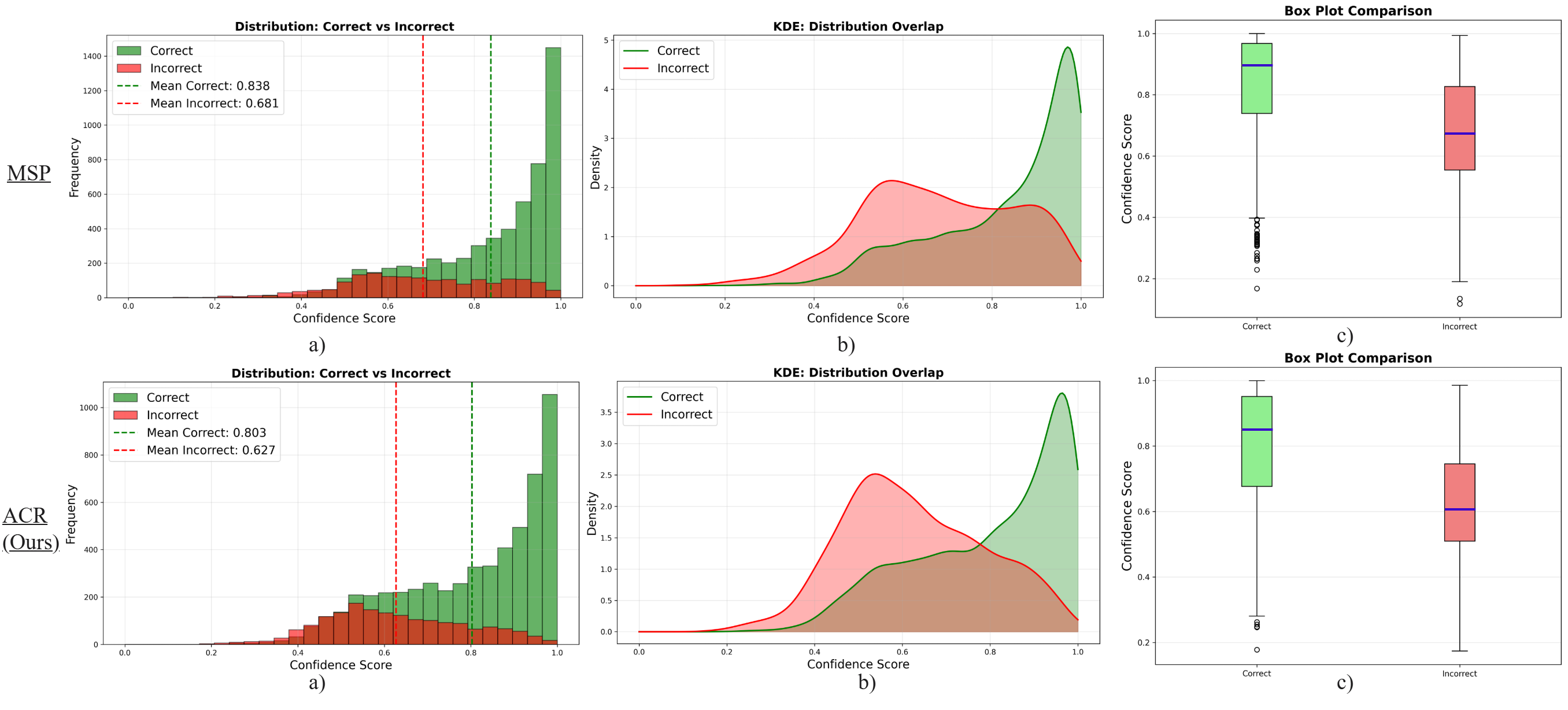}
    \vspace{-4mm}
    \caption{Distribution analysis of MSP and our ACR with TSPM models on MUSIC-AVQA dataset.}
    \label{fig:distribution_analysis}
    \vspace{0.mm}
    \includegraphics[width=0.98\linewidth]{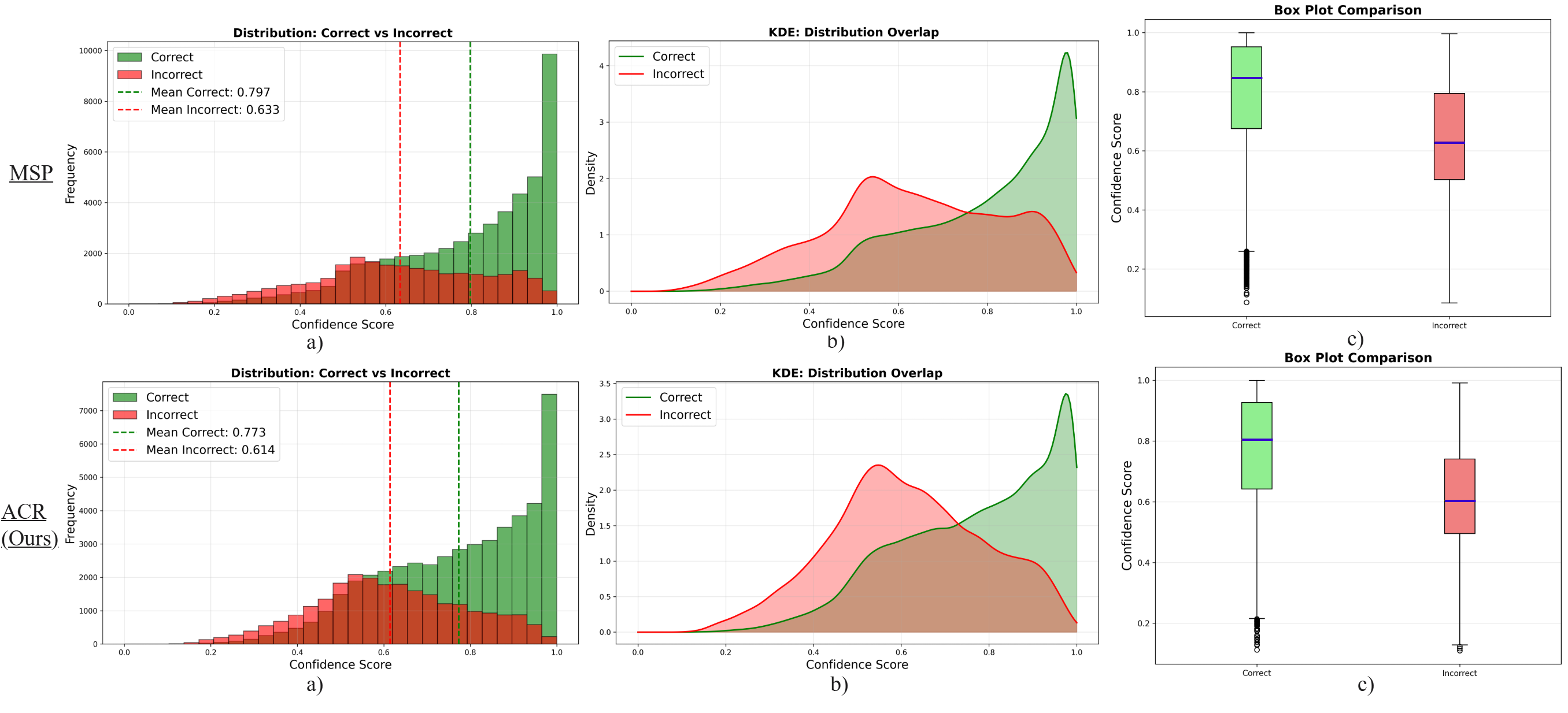}
    \vspace{-4mm}
    \caption{Distribution analysis of MSP and our ACR with TSPM models on MUSIC-AVQA-R dataset (Tail).}
    \label{fig:distribution_analysis2}
    \vspace{0.mm}
    \includegraphics[width=0.98\linewidth]{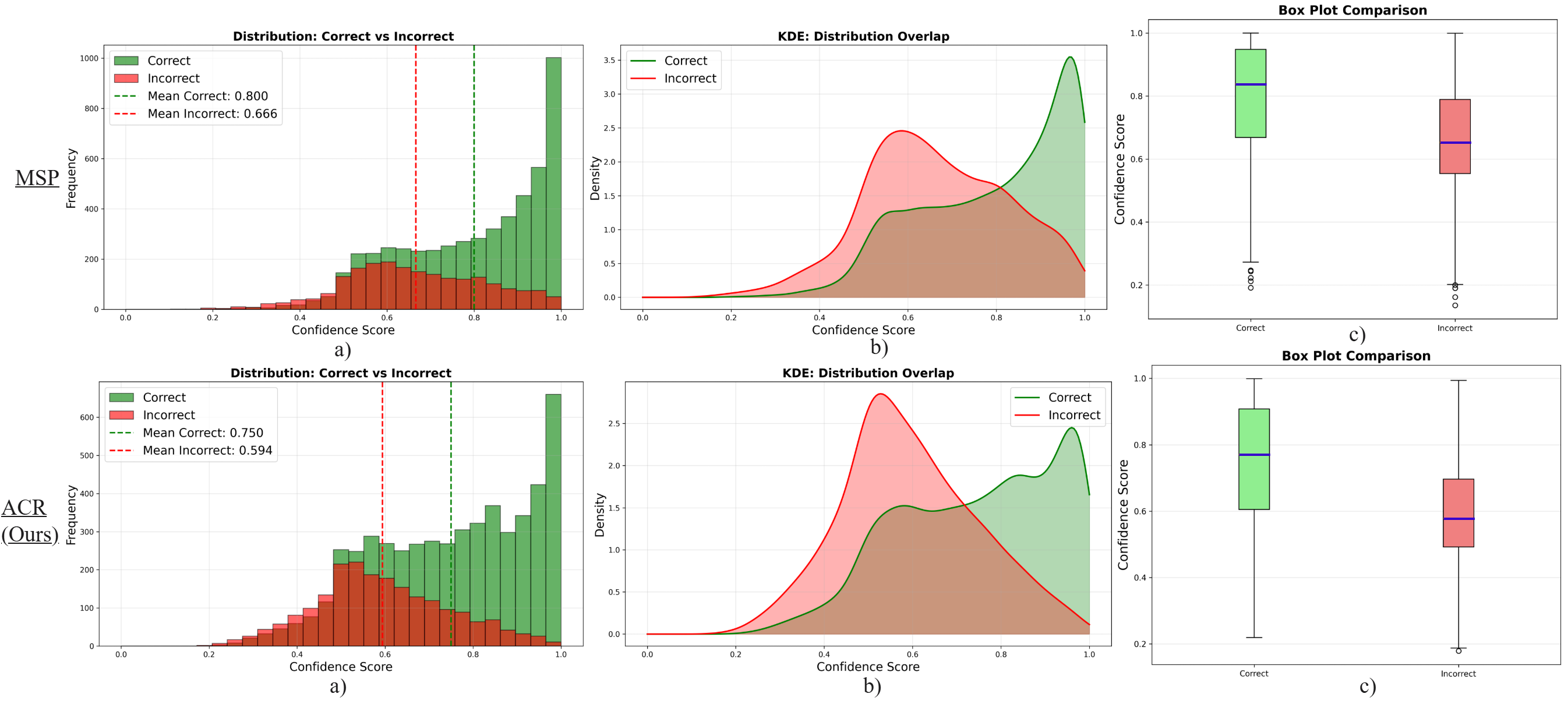}
    \vspace{-4mm}
    \caption{Distribution analysis of MSP and ACR with TSPM models on MUSIC-AVQA-v2.0 dataset (trained on balance, tested on bias).}
    \label{fig:distribution_analysis3}
\end{figure*}




\subsection{Distribution Analysis}

Fig. \ref{fig:distribution_analysis} illustrates the distribution characteristics of confidence scores for correct and incorrect predictions across both methods.

\textbf{Distribution Overlap and Separation}. For the \textbf{MSP baseline (first column)}, the histogram (Fig. \ref{fig:distribution_analysis}a) shows significant overlap between correct predictions (\textcolor{green}{green}) and incorrect predictions (\textcolor{red}{red}). The mean confidence scores are 0.838 for correct predictions and 0.681 for incorrect predictions, resulting in a separation of only 0.16. The KDE plot in Fig. \ref{fig:distribution_analysis}b further illustrates this issue, showing notable overlap in the confidence range from 0.6 to 0.9, indicating difficulty distinguishing between correct and incorrect predictions. This overlapping region contains a substantial portion of both distributions, limiting the effectiveness of MSP as a reliability indicator. Additionally, the box plot in Fig. \ref{fig:distribution_analysis}c confirms this overlap, displaying interquartile ranges that intersect and a number of outliers among the correct predictions that extend down to the 0.2 to 0.4 range. This suggests that MSP can be prone to overconfidence in its subset assessments of incorrect predictions.
\textbf{Our proposed ACR (second column)} demonstrates significantly improved separation characteristics. The histogram in Fig. \ref{fig:distribution_analysis}a shows mean confidence scores of 0.803 for correct predictions and 0.627 for incorrect predictions, resulting in a separation of 0.176. This marks a 12\% improvement over MSP. Notably, the KDE plot in Fig. \ref{fig:distribution_analysis}b shows a more pronounced bimodal distribution. The distribution for correct predictions has a sharper peak near 1.0, while the incorrect predictions show higher density in the lower confidence regions (0.4-0.7). Especially, the reduced overlap in the decision boundary region (0.6–0.8), a confidence range in which selective prediction decisions are most sensitive, indicates enhanced discriminative capability.

\textbf{Variance and Distribution Shape}. The box plot comparison (Fig. \ref{fig:distribution_analysis}c) highlights key differences in the distribution characteristics. MSP exhibits high variance in correct predictions, with the interquartile range spanning approximately 0.75-1.0, and contains numerous low-confidence outliers. This suggests inconsistent confidence calibration. In contrast, ACR exhibits more concentrated confidence levels for correct predictions, with fewer outliers, indicating more reliable confidence estimates. Regarding incorrect predictions, ACR has a lower median ($\sim0.7$ vs. $\sim0.8$), which is beneficial because it yields better separation from correct predictions.

The similar pattern can be observed in MUSIC-AVQA-R (Tail) and MUSIC-AVQA-v2.0 (trained on the balanced set and tested on the biased set) in Fig. \ref{fig:distribution_analysis2} and \ref{fig:distribution_analysis3}. The distribution analysis shows that ACR outperforms MSP in distinguishing between correct and incorrect predictions, achieving a 12\% improvement in mean difference and reducing overlap in the distribution. These enhancements lead to better selective prediction capabilities, allowing for more effective risk-coverage trade-offs in practical video question-answering scenarios. However, challenges persist; both methods still exhibit a significant number of high-confidence incorrect predictions, as indicated by outliers in the box plots. This suggests a need for further improvements in uncertainty quantification. Future work will aim to reduce the ambiguity zone and enhance calibration through advanced techniques in uncertainty quantification.

\begin{figure*}[hpt]
    \centering
    \includegraphics[width=0.88\linewidth]{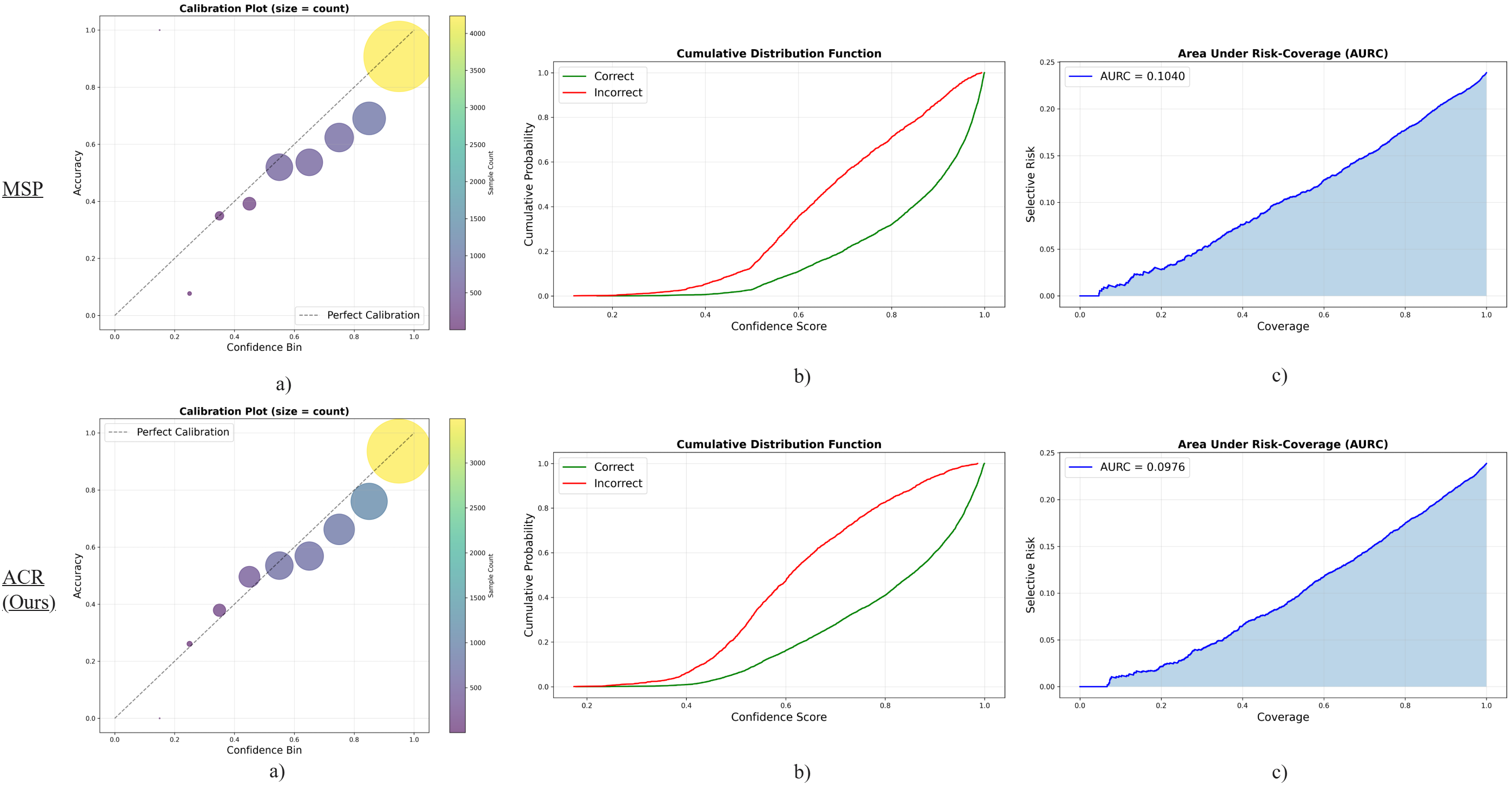}
    \vspace{-5mm}
    \caption{Calibration and AURC analysis of MSP and our ACR with TSPM models on MUSIC-AVQA dataset.}
    \label{fig:calibration_quality}
    \vspace{0.mm}
    \includegraphics[width=0.88\linewidth]{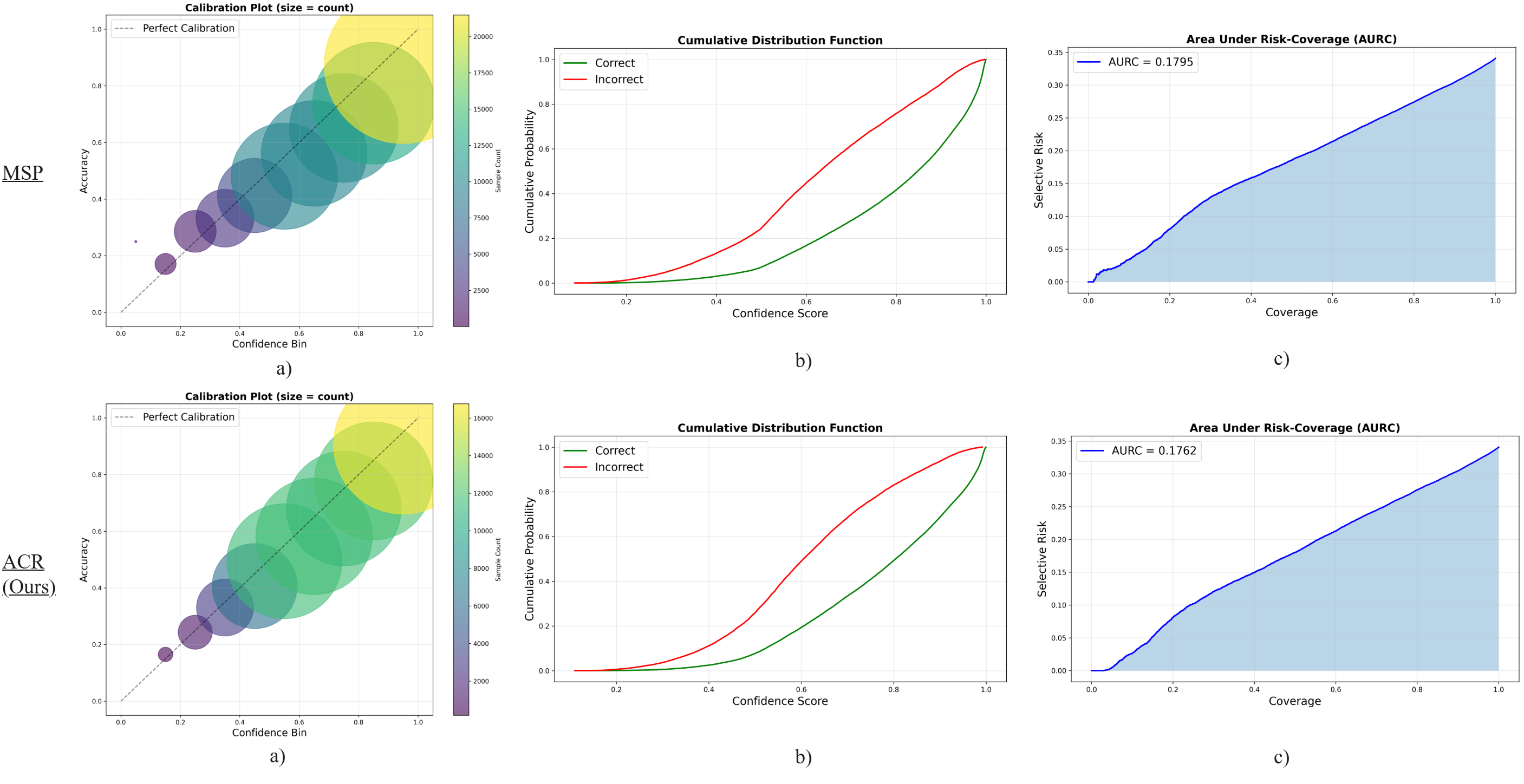}
    \vspace{-5mm}
    \caption{Distribution Analysis of MSP and our ACR with TSPM models on MUSIC-AVQA-R dataset (Tail).}
    \label{fig:calibration_quality2}
    \vspace{0.mm}
    \includegraphics[width=0.88\linewidth]{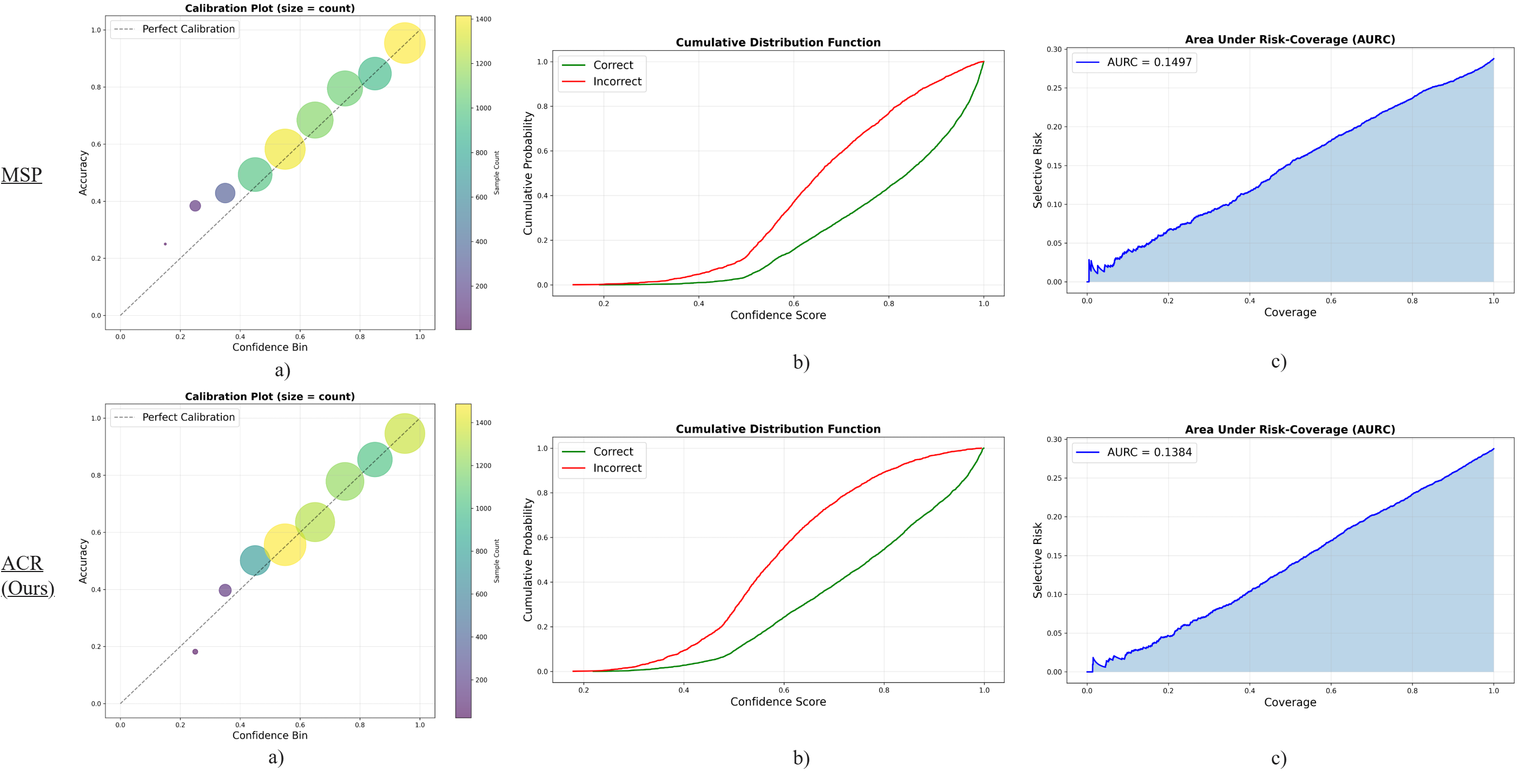}
    \vspace{-5mm}
    \caption{Distribution Analysis of MSP and ACR with TSPM models on MUSIC-AVQA-v2.0 dataset (trained on balance, tested on bias).}
    \label{fig:calibration_quality3}
\end{figure*}




\subsection{Calibration and Selective Prediction Analysis}

Fig. \ref{fig:calibration_quality} provides a detailed evaluation of confidence calibration and selective prediction capabilities, comparing MSP method (first column) with our proposed ACR approach (second column) across three critical dimensions: calibration accuracy, cumulative distribution characteristics, and risk-coverage trade-offs.

\textbf{Calibration Plot Analysis}. In the calibration plot for the \textbf{MSP baseline} (first column) shown in Fig. \ref{fig:calibration_quality}a, we observe systematic miscalibration across different confidence bins, which consistently deviate from the perfect calibration diagonal. While the largest bin (indicated by the yellow circle, containing approximately 4,000 samples) shows near-perfect calibration, this reflects the behavior of the majority class where the model tends to be overly confident. Notably, the mid-range bins (with confidence levels between 0.4 and 0.7) display greater deviations from the diagonal, highlighting poor calibration at crucial points where selective prediction decisions are most important. Additionally, the smaller bins (represented in darker purple) at lower confidence levels contain fewer samples, indicating that MSP rarely assigns low confidence to incorrect predictions.
In contrast, \textbf{our ACR method}, as illustrated in the second column of Fig. \ref{fig:calibration_quality}a, demonstrates improved calibration characteristics. First, the bins are more evenly distributed across the confidence spectrum, with significant sample counts (shown in darker colors) in the 0.5-0.7 range. This suggests that ACR utilizes a broader range of confidence levels instead of clustering predictions. Second, while some slight miscalibration is still observed in the mid-range bins, the deviation from the diagonal is reduced when compared to MSP. The high-confidence bin (again marked by the yellow circle with around 4,000 samples) maintains excellent calibration, but importantly, ACR generates more samples in the lower confidence bins where incorrect predictions are concentrated. This distribution indicates that ACR is better at assigning appropriately lower confidence to uncertain predictions, which is crucial for effective selective prediction.
In summary, while both methods exhibit signs of miscalibration, ACR demonstrates a more uniform distribution of bins and better utilization of the confidence spectrum.

\textbf{Cumulative Distribution Function Analysis}. The Cumulative Distribution Function (CDF) plots, as in Fig. \ref{fig:calibration_quality}b, provide important insights into the ability to distinguish between correct and incorrect predictions. In the case of the \textbf{MSP method}, the red curve (representing incorrect predictions) and the green curve (representing correct predictions) show delayed separation; they remain close together until the confidence level reaches approximately 0.5, after which they begin to diverge. In contrast, our \textbf{ACR} demonstrates earlier and more pronounced separation. The distribution of incorrect predictions (\textcolor{red}{red}) rises more steeply at lower confidence levels, indicating that more incorrect predictions are assigned appropriately low confidence scores.
At a confidence level of 0.6, the MSP method filters out only approximately 36\% of incorrect predictions, while also misclassifying around 10\% of correct ones. In contrast, our ACR approach exhibits significantly stronger separation, with around 50\% of incorrect predictions falling below this threshold, while retaining about 84\% of correct predictions (only 16\% correct predictions fall below a confidence of 0.6). Importantly, ACR maintains this advantage consistently across the entire selection range of the decision-critical zone (0.6-0.9), rather than just at isolated points.
 
\textbf{Area Under Risk-Coverage Curve}. The AURC metric provides a single-number summary of selective prediction performance across all coverage levels, with lower values indicating better performance, as illustrated in Fig. \ref{fig:calibration_quality}c. For the MSP, the AURC is 10.40\%, and the risk-coverage curve shows a gradual reduction in risk as coverage decreases. In contrast, our ACR achieves a lower AURC score of 9.76\%, representing a \textbf{6.15\%} improvement over the MSP. More importantly, the curve exhibits a steeper initial slope, indicating more effective early rejection.

\subsection{Statistical Metrics of Separation Quality}

To thoroughly evaluate the separation quality between correct and incorrect predictions across various confidence estimation methods, we utilize four complementary statistical metrics. 
\begin{enumerate}

    \item Cohen's d: This standardizes the mean difference by the pooled standard deviation, offering a scale-free effect size that considers within-group variability. This allows comparisons across studies, with values above 0.8 indicating a large and practically significant difference.
    
    \item Wasserstein Distance: This metric quantifies the minimum "transportation cost" required to transform one confidence distribution into another. It offers a robust, symmetric distance measure that is directly interpretable in the original confidence scale and is closely related to the area between the cumulative distribution functions (CDFs).

    \item KL Divergence (Kullback-Leibler Divergence): This measures the information-theoretic difference between distributions, capturing how dissimilar they are. It is particularly sensitive to tail behavior but lacks symmetry.

    \item Area Under Risk-Coverage Curve (AURC): This metric quantifies selective prediction performance by measuring the average error rate across all coverage levels; lower values indicate a better ability to reduce risk through confidence-based sample rejection.
\end{enumerate}
These metrics collectively provide a comprehensive framework for assessing the effectiveness of different confidence estimation methods.
We report these metrics for MSP and our ACR selection function in Table \ref{tab:statistical_metric}. As we can see, our ACR consistently outperforms MSP in four representative metrics for separation quality between correct and incorrect predictions, leading to better performance in selective prediction setting.

\begin{table}[h]
\centering
\caption{Summary statistical metrics of separation quality on TSPM backbone.}
\label{tab:statistical_metric}
\scalebox{0.9}{
\begin{tabular}{lccccc} 
\toprule
Dataset                          & Sel. Func. g & Cohen’s d $\uparrow$ & Wasserstein Distance $\uparrow$ & KL Divergence $\uparrow$ & AURC $\downarrow$  \\ 
\midrule
\multirow{2}{*}{MUSIC-AVQ}       & MSP          & 0.9794      & 0.1569                 & 11.11          & 10.40  \\
                                 & ACR          & \textbf{1.0547}      & \textbf{0.1758}                 & \textbf{14.22}          & \textbf{9.76}  \\ 
\midrule
\multirow{2}{*}{MUSIC-AVQ-R}     & MSP      & 0.8827      & 0.1589                 & 8.88          & 17.95  \\
                                 & ACR            & \textbf{0.9074}      & \textbf{0.1640}                 & \textbf{10.22}          & \textbf{17.62}  \\ 
\midrule
\multirow{2}{*}{MUSIC-AVQA-v2.0} & MSP             & 0.8160      & 0.1335                 & 8.15          & 14.97  \\
                                 & ACR     & \textbf{0.9139}      & \textbf{0.1558}                 & \textbf{11.12}          & \textbf{13.84}  \\
\bottomrule
\end{tabular}}
\end{table}

\subsection{Learned $\boldsymbol{\alpha(x)}$ Distribution Analysis and Connection to Theorem \ref{thm:adaptive_fusion}} 
Fig. \ref{fig:alpha_distribution_analysis} illustrates the distribution of the learned adaptive fusion weights, denoted as $\alpha^*(x)$, across six experimental configurations. This analysis reveals systematic patterns that empirically support Theorem \ref{thm:adaptive_fusion}.
In the case of the QA-TIGER backbone (top row), which demonstrates higher base AVQA accuracy, the learned weights often exhibit higher values, with means ranging from $\alpha = 0.858$ to $0.932$. This suggests that the selector primarily relies on the confidence of the MSP while also integrating the learned residual risk signal from the RRH. Conversely, the TSPM backbone (bottom row), which has lower base AVQA accuracy, shows broader weight distributions with lower means ($\alpha = 0.828$ to $0.870$) and significantly higher variance. This is particularly evident in configuration (c), where the weights cover nearly the entire range of $[0.4, 1.0]$.

This behavior, which is dependent on the architecture, aligns precisely with the predictions of Theorem \ref{thm:adaptive_fusion}. When the quality of the MSP varies (higher for QA-TIGER and lower for TSPM), the near-optimal pointwise fusion weight $\alpha^*(x)$ adjusts accordingly. This results in a variance of $\mathrm{Var}[\alpha^{\star}(x)] > 0$ across all settings. Importantly, no distribution collapses to fixed values (i.e., $\alpha = 0$ or $\alpha = 1$), which confirms that the network learns genuine input-adaptive fusion rather than resorting to trivial solutions.
The notably larger variance observed for TSPM further illustrates that when the MSP is less reliable, the selector adaptively downgrades its influence on a per-sample basis, precisely as Theorem \ref{thm:adaptive_fusion} recommends for optimally approximating the Bayes selector $g^*(x)$.

Additionally, the green vertical lines, representing the mean $\alpha$, consistently fall between the red reference line (indicating equal weighting, $\alpha = 0.5$) and the rightmost bins. This shows that the learned policy intelligently balances the calibration strength of the MSP with the complementary risk signals from the RRH, rather than simply averaging or exclusively depending on either component.
Overall, this empirical evidence confirms the non-trivial nature and architecture-awareness of the proposed adaptive fusion mechanism. It directly supports the claim made in Theorem \ref{thm:adaptive_fusion} that adaptive fusion is superior to fixed-weight alternatives when optimal weights vary across different inputs.

\begin{figure*}[hpt]
    \centering
    \includegraphics[width=0.92\linewidth]{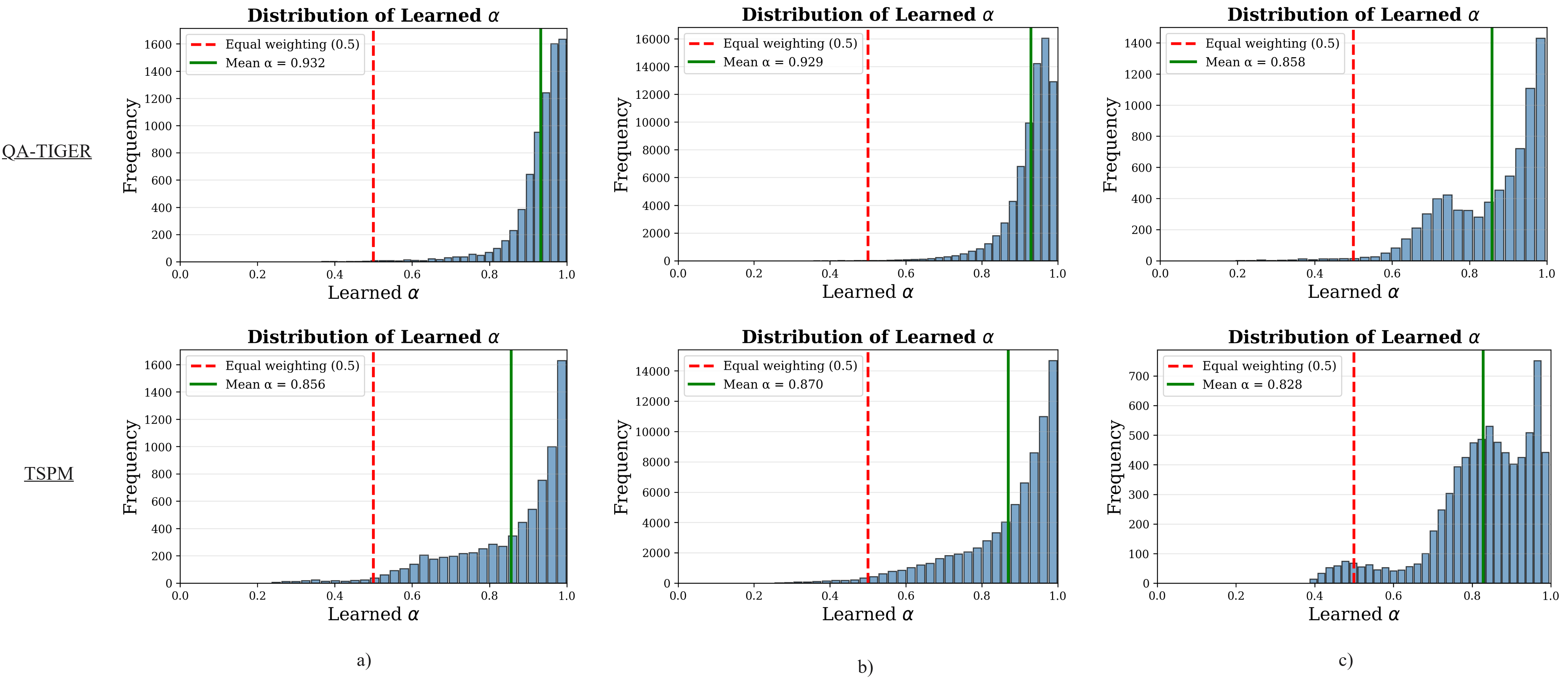}
    \caption{Learned $\alpha$ distribution analysis based on QA-TIGER and TSPM backbone on a) MUSIC-AVQA dataset b) MUSIC-AVQA-R dataset (Tail) and c) MUSIC-AVQA-v2.0 dataset (trained on balance set, tested on biased set).}
    \label{fig:alpha_distribution_analysis}
\end{figure*}

\section{Computational Complexity and Overhead}
\label{section:B_1_1}

Adaptive Confidence Refinement (ACR) is designed to improve selective prediction for standard AVQA task without introducing meaningful computational burden. ACR operates entirely in the confidence estimation stage and does not modify the backbone AVQA model or its inference pathway.

\textbf{Parameter Complexity}. ACR introduces two lightweight components: the Residual Risk Head (RRH) and the Confidence Gating Head (CGH). Both are implemented as three-layer multi-layer perceptrons (MLPs) with a hidden dimension denoted as \(d\). The total number of additional parameters introduced by these two heads is \(\mathcal{O}(d^{2})\).
Most recent Audio-Visual Question Answering (AVQA) models utilize a transformer architecture, which includes multiple layers for visual, audio, and text encoders, as well as multimodal fusion layers (such as cross-attention). For an AVQA model that has \(L\) transformer layers, the total number of parameters for these base AVQA models is \(\mathcal{O}(d^2 L)\), where \(L \gg 6\). As a result, the additional parameters from our new heads are negligible in comparison to those of the base AVQA model. 
For instance, when applying our method to QA-TIGER, the total number of additional parameters constitutes only about \textbf{3\%} of the overall size of the base AVQA model. This confirms that ACR incurs negligible parameter cost relative to the base QAGE architecture.

\textbf{Inference Complexity.} During inference, ACR requires only two forward passes through shallow three-layer perceptrons (MLPs) and a scalar fusion operation. The additional computational cost is $\mathcal{O}(d^2)$ per sample, which is insignificant compared to the $\mathcal{O}(d^2 L)$ complexity associated with multimodal transformer inference. Consequently, ACR introduces negligible latency overhead and does not impact throughput in either real-time or batch inference settings.

\textbf{Training Complexity.} ACR employs a two-stage training strategy. In the first stage, the backbone AVQA model is trained using standard cross-entropy loss, after which it is frozen. In the second stage, only the RRH and CGH components are trained using binary cross-entropy based on correctness labels. This second stage is computationally efficient, converges quickly, and eliminates the need for repeated fine-tuning of the backbone model.

\textbf{Memory Overhead.} ACR only requires storing additional head parameters and the intermediate feature representations produced by the backbone. No extra memory buffers or ensemble models are necessary.

ACR can be implemented in existing AVQA models with minimal computational and memory overhead, without needing retraining or architectural changes. This characteristic makes ACR especially suitable for deployment scenarios that are large-scale or resource-constrained, where reliability is crucial.

\section{Reproduced Results of AVQA Models}
\label{section:B_2}
\begin{table}[h]
\caption{Reproduced accuracy of AVQA models on the original MUSIC-AVQA testset. All values are in $\%$}
\label{tab:reproduced_result}
\centering
\begin{tabular}{lccc} 
\toprule
AVQA Models & Original Accuracy & Reproduced Accuracy & $\Delta$Accuracy  \\ 
\midrule
QA-TIGER \cite{kim2025question}    & 77.62              & 77.52                & -0.10        \\
ST-AVQA \cite{li2022learning}     & 71.52              & 71.26                & -0.26        \\
TSPM \cite{li2024boosting}        & 76.79              & 76.51                & -0.28        \\
\bottomrule
\end{tabular}
\end{table}

In this section, we present the results we reproduced for all AVQA models used in our experiments on the $\mathcal{R}$-AVQA benchmarks, which include QA-TIGER \cite{kim2025question}, ST-AVQA \cite{li2022learning}, and TSPM \cite{li2024boosting}. The summarized results are presented in Tab. \ref{tab:reproduced_result}. We focus on models that have publicly available implementations and open-source data preprocessing pipelines to ensure reproducibility and facilitate fair comparisons. All models were trained and evaluated exactly as described in their original works. As shown in Tab. \ref{tab:reproduced_result}, our reproduced results closely align with the reported performance in the respective papers, confirming the accuracy of our implementations and the reliability of our experimental setup.

\section{Comparing to Data Augmentation}
\label{section:C}

\begin{table}[h]
\centering
\caption{Coverage at target risk ($\mathcal{C}$@$\mathcal{R}$) $\uparrow$, AURC $\downarrow$, and ECE $\downarrow$ for all selection functions on MUSIC-AVQA dataset. All in \%.}
\label{tab:data_augmentation}
\begin{tabular}{llcccccc} 
\toprule
Model $f$                 & Sel. Func. $g$ & $\mathcal{C}$@1\% & $\mathcal{C}$@5\% & $\mathcal{C}$@10\% & $\mathcal{C}$@20\% & AURC          & ECE            \\ 
\midrule
\multirow{3}{*}{QA-TIGER} & MSP            & 14.29             & 41.53             & 60.47              & 91.83              & 8.67          & 8.83   \\
                          & MSP-Augment    & \underline{15.41}     & \underline{41.98}             & \underline{60.77}              & \underline{91.90}      & \underline{8.56}          & \underline{6.23}   \\
                          & ACR (Ours)     & \textbf{17.21}    & \textbf{43.07}    & \textbf{61.39}     & \textbf{91.96}     & \textbf{8.41} & \textbf{4.88}  \\
\bottomrule
\end{tabular}
\end{table}

In our experiments, we set aside a separate validation set to validate the AVQA models and to train the selection functions, referred to as Val-$f$ in Tab. \ref{tab:music-avqa-split} and \ref{tab:music-avqa-v2-split}. An alternative approach is to use these data to further augment the AVQA training set, potentially improving the base model's accuracy. Since the performance of selective prediction is inherently tied to accuracy (as discussed in Sect. \ref{sec:benchmarking_metric}), such augmentation might also benefit selection based on the MSP method. We conduct experiments with QA-TIGER on MUSIC-AVQA dataset to observe the performance change of MSP compared to our ACR approach when we apply above data augmentation. 

Tab. \ref{tab:data_augmentation} compares these two strategies. We find that allocating the additional data to train the selector is consistently more effective for enhancing coverage at low risk levels. While augmenting the AVQA training data does improve overall accuracy, these improvements mainly lead to higher coverage only as the risk tolerance approaches the base model's error rate. In contrast, direct enhancements to the selection function, as our method provides, confer significant advantages in low-risk situations, where reliability is paramount.

\begin{table*}[!t]
\centering
\caption{Comparison on MUSIC-AVQA-R and MUSIC-AVQA-v2 datasets. We report Head (H)/Tail (T) results for MUSIC-AVQA-R and Balance (Bal)/Biased (Bias) results for MUSIC-AVQA-v2. Metrics: coverage at target risk ($\mathcal{C}$@$\mathcal{R}$) $\uparrow$, AURC $\downarrow$, and ECE $\downarrow$ (all in \%).}
\label{tab:merged_table23_appendix}
\resizebox{0.9\textwidth}{!}{%
\begin{tabular}{cll cccccc cccccc}
\toprule
\multirow{2}{*}{Dataset} & \multirow{2}{*}{Model $f$} & \multirow{2}{*}{Sel. Func. $g$} 
& \multicolumn{2}{c}{$\mathcal{C}$@1\%} & \multicolumn{2}{c}{$\mathcal{C}$@5\%} 
& \multicolumn{2}{c}{$\mathcal{C}$@10\%} & \multicolumn{2}{c}{$\mathcal{C}$@20\%} 
& \multicolumn{2}{c}{AURC} & \multicolumn{2}{c}{ECE} \\
\cmidrule(lr){4-5}\cmidrule(lr){6-7}\cmidrule(lr){8-9}\cmidrule(lr){10-11}\cmidrule(lr){12-13}\cmidrule(lr){14-15}
& & & \small H/Bal & \small T/Bias & \small H/Bal & \small T/Bias & \small H/Bal & \small T/Bias & \small H/Bal & \small T/Bias & \small H/Bal & \small T/Bias & \small H/Bal & \small T/Bias \\
\midrule
\multirow{6}{*}{\rotatebox{90}{MUSIC-AVQA-R}}
& \multirow{5}{*}{ST-AVQA}
& MSP        & 0.10 & \underline{0.50} & 0.29 & 1.48 & 14.13 & 2.94 & \underline{44.88} & 7.00 & 21.35 & 36.53 & 5.92 & 6.59 \\[0.5ex]
& & MCD        & 0.10 & 0.44 & \underline{0.42} & 1.39 & \underline{14.19} & 2.74 & 44.85 & 6.76 & \underline{21.34} & 36.68 & 6.02 & 6.56 \\ [0.5ex]
& & VS         & \underline{0.11} & 0.44 & 0.25 & \underline{1.53} & 13.99 & \underline{3.04} & 44.66 & \underline{7.47} & 21.42 & \underline{36.03} & \underline{5.91} & \underline{6.40} \\ [0.5ex]
& & Doctor         & 6.04 & 0.23 & \underline{18.56} & 13.49 & 34.24 & 26.01 & 63.52 & 59.31 & \underline{15.41} & 16.81 & \underline{11.87} & \underline{11.95} \\ [0.5ex]
& & ACR (Ours) & \textbf{0.44} & \textbf{1.03} & \textbf{3.17} & \textbf{2.44} & \textbf{14.80} & \textbf{4.15} & \textbf{44.95} & \textbf{10.03} & \textbf{21.14} & \textbf{35.02} & \textbf{2.60} & \textbf{6.07} \\ [0.5ex] \cdashline{2-15}
& & \textcolor{gray}{Oracle} & \textcolor{gray}{63.45} & \textcolor{gray}{51.33} & \textcolor{gray}{66.12} & \textcolor{gray}{53.50} & \textcolor{gray}{69.79} & \textcolor{gray}{56.47} & \textcolor{gray}{78.52} & \textcolor{gray}{63.56} & \textcolor{gray}{7.98} & \textcolor{gray}{14.78} & \textcolor{gray}{0.00} & \textcolor{gray}{0.00} \\ [0.5ex]
\midrule[\heavyrulewidth]
\multirow{6}{*}{\rotatebox{90}{MUSIC-AVQA v2.0}}
& \multirow{5}{*}{ST-AVQA}
& MSP        & 4.70 & 4.26 & 13.07 & 11.61 & 28.47 & 28.14 & 61.43 & 62.66 & 16.25 & 16.00 & 6.62 & 6.31 \\ [0.5ex]
& & MCD        & \underline{4.76} & 4.24 & 13.23 & 11.29 & 28.69 & \underline{28.67} & 61.29 & 62.80 & 16.28 & 16.01 & 6.32 & 6.11 \\ [0.5ex]
& & VS         & 4.75 & \underline{4.30} & \underline{13.68} & \underline{11.62} & \underline{28.93} & 28.43 & 61.41 & 62.86 & \underline{16.24} & \underline{15.99} & 6.54 & 6.14 \\ [0.5ex]
& & Doctor          & 4.70 & 4.26 & 13.01 & 11.55 & 28.45 & 27.97 & \underline{62.35} & \underline{63.94} & \underline{16.24} & \underline{15.99} & \underline{3.95} & \underline{4.26} \\ [0.5ex]
& & ACR (Ours) & \textbf{7.44} & \textbf{6.29} & \textbf{15.03} & \textbf{15.93} & \textbf{30.53} & \textbf{32.02} & \textbf{62.90} & \textbf{64.49} & \textbf{15.76} & \textbf{15.41} & \textbf{3.90} & \textbf{3.52} \\ [0.5ex] \cdashline{2-15}
& & \textcolor{gray}{Oracle} & \textcolor{gray}{69.30} & \textcolor{gray}{69.80} & \textcolor{gray}{72.22} & \textcolor{gray}{72.74} & \textcolor{gray}{76.23} & \textcolor{gray}{76.78} & \textcolor{gray}{85.76} & \textcolor{gray}{86.38} & \textcolor{gray}{5.54} & \textcolor{gray}{5.36} & \textcolor{gray}{0.00} & \textcolor{gray}{0.00} \\ [0.5ex]
\bottomrule
\end{tabular}%
}
\end{table*}

\begin{table*}[ht]
\caption{Coverage at target risk ($\mathcal{C}$@$\mathcal{R}$) $\uparrow$, AURC $\downarrow$, and ECE $\downarrow$ on MUSIC-AVQA v2.0 (train on biased; test on balance (Bal) and bias (Bias) sets). All in \%}
\label{tab:music_avqa_v2_train_bias_test_bal_bias}
\centering
\scalebox{0.85}{
\begin{tabular}{llcccccccccccc}
\toprule
\multirow{2}{*}{Model $f$} & \multirow{2}{*}{Sel. Func. $g$} & \multicolumn{2}{c}{$\mathcal{C}$@1\%} & \multicolumn{2}{c}{$\mathcal{C}$@5\%} & \multicolumn{2}{c}{$\mathcal{C}$@10\%} & \multicolumn{2}{c}{$\mathcal{C}$@20\%} & \multicolumn{2}{c}{AURC} & \multicolumn{2}{c}{ECE} \\
\cmidrule(lr){3-4}\cmidrule(lr){5-6}\cmidrule(lr){7-8}\cmidrule(lr){9-10}\cmidrule(lr){11-12}\cmidrule(lr){13-14}
& & Bal & Bias & Bal & Bias & Bal & Bias & Bal & Bias & Bal & Bias & Bal & Bias \\
\midrule
\multirow{5}{*}{QA-TIGER}
& MSP & \underline{3.63} & \underline{5.55} & \underline{28.71} & \underline{36.81} & 46.67 & 58.94 & \underline{82.22} & 93.15 & 11.51 & 9.17 & 13.58 & 11.16 \\
& MCD & 3.10 & 4.40 & 27.99 & 36.66 & 46.04 & \underline{58.96} & 81.56 & \textbf{93.48} & \underline{11.50} & 9.20 & 15.02 & 12.40 \\
& VS  & 2.80 & 5.22 & 28.43 & 36.80 & \underline{46.90} & 58.70 & \textbf{82.37} & \underline{93.24} & 11.52 & \underline{9.16} & 15.23 & 12.65 \\
& Doctor  & \underline{3.63} & \underline{5.55} & 28.70 & \underline{36.81} & 46.65 & 58.84 & \textbf{82.37} & 92.72 & 11.51 & 9.17 & \underline{11.26} & \underline{9.14} \\
& ACR (Ours) & \textbf{3.88} & \textbf{12.55} & \textbf{30.13} & \textbf{38.94} & \textbf{47.23} & \textbf{59.58} & 81.52 & 92.34 & \textbf{11.30} & \textbf{8.80} & \textbf{7.56} & \textbf{2.78} \\ \hdashline
& \textcolor{gray}{Oracle} & \textcolor{gray}{75.08} & \textcolor{gray}{78.42} & \textcolor{gray}{78.24} & \textcolor{gray}{81.72} & \textcolor{gray}{82.59} & \textcolor{gray}{86.26} & \textcolor{gray}{92.91} & \textcolor{gray}{97.04} & \textcolor{gray}{3.62} & \textcolor{gray}{2.71} & \textcolor{gray}{0.00} & \textcolor{gray}{0.00} \\
\midrule
\multirow{5}{*}{ST-AVQA}
& MSP & 0.55 & 2.29 & 4.48 & \underline{15.01} & 14.60 & 34.85 & \underline{49.64} & 68.78 & 19.80 & 14.78 & 11.03 & 7.51 \\
& MCD & 0.55 & \underline{2.40} & 4.23 & 14.57 & 14.48 & \textbf{35.36} & 49.56 & \underline{69.06} & 19.82 & 14.80 & 11.03 & 7.61 \\
& VS  & \underline{0.57} & 2.38 & \underline{4.76} & 14.95 & \underline{14.64} & \underline{34.95} & 49.63 & 68.70 & \underline{19.76} & 14.80 & 10.76 & 7.35 \\
& Doctor & 0.55 & 2.26 & 4.47 & 14.97 & 14.60 & \underline{34.95} & 49.53 & 68.75 & 19.78 & \underline{14.76} & \underline{4.89} & \underline{3.06} \\
& ACR (Ours) & \textbf{1.96} & \textbf{5.06} & \textbf{7.91} & \textbf{16.09} & \textbf{17.36} & 34.84 & \textbf{50.02} & \textbf{69.09} & \textbf{19.41} & \textbf{14.42} & \textbf{4.09} & \textbf{2.22} \\ \hdashline
& \textcolor{gray}{Oracle} & \textcolor{gray}{66.40} & \textcolor{gray}{70.80} & \textcolor{gray}{69.19} & \textcolor{gray}{73.78} & \textcolor{gray}{73.04} & \textcolor{gray}{77.88} & \textcolor{gray}{82.17} & \textcolor{gray}{87.62} & \textcolor{gray}{6.69} & \textcolor{gray}{5.00} & \textcolor{gray}{0.00} & \textcolor{gray}{0.00} \\
\midrule
\multirow{5}{*}{TSPM}
& MSP & 1.83 & 1.91 & 9.96 & 15.05 & 21.27 & 38.40 & 54.12 & 71.15 & 18.10 & 14.20 & 9.49 & 6.90 \\
& MCD & \underline{1.85} & \underline{1.95} & 9.95 & 15.43 & 21.18 & \underline{38.67} & 54.10 & 71.16 & 18.12 & 14.21 & \underline{9.45} & 6.81 \\
& VS  & 1.81 & 1.88 & \underline{10.28} & \underline{15.95} & \underline{21.97} & 38.24 & 53.95 & \underline{71.23} & \underline{18.04} & 14.27 & 9.57 & 7.10 \\
& Doctor & 1.83 & 1.91 & 9.89 & 15.21 & 21.15 & 38.53 & \underline{54.25} & 70.45 & 18.14 & \underline{14.11} & \underline{5.51} & \underline{3.58} \\
& ACR (Ours) & \textbf{6.13} & \textbf{4.17} & \textbf{13.64} & \textbf{23.62} & \textbf{26.27} & \textbf{42.43} & \textbf{59.62} & \textbf{72.45} & \textbf{16.98} & \textbf{13.10} & \textbf{4.65} & \textbf{2.40} \\ \hdashline
& \textcolor{gray}{Oracle} & \textcolor{gray}{69.79} & \textcolor{gray}{72.28} & \textcolor{gray}{71.69} & \textcolor{gray}{75.33} & \textcolor{gray}{75.67} & \textcolor{gray}{79.51} & \textcolor{gray}{85.13} & \textcolor{gray}{89.45} & \textcolor{gray}{5.74} & \textcolor{gray}{4.49} & \textcolor{gray}{0.00} & \textcolor{gray}{0.00} \\
\bottomrule
\end{tabular}}
\end{table*}

\section{Additional Quantitative Results}
\label{section:D}

Due to the page limits, we cannot include all table results in the main paper. In this section, we report all remaining table results, including Tab. \ref{tab:merged_table23_appendix} and Tab. \ref{tab:music_avqa_v2_train_bias_test_bal_bias}. Tab. \ref{tab:merged_table23_appendix} presents the performance of ST-AVQA backbone on MUSIC-AVQA-R and MUSIC-AVQA-v2.0 (trained on balance set), while Tab. \ref{tab:music_avqa_v2_train_bias_test_bal_bias} reports performance of all three AVQA backbones on MUSIC-AVQA-v2.0 dataset (trained on biased set).

In accordance with the main paper, we evaluate coverage at various risk tolerance levels, as well as AURC and ECE. Our method, ACR, consistently outperforms all competing approaches across all metrics, demonstrating strong generalization and robustness. Notably, ACR shows particularly significant improvements in low-risk scenarios (e.g., $\mathcal{C}$@1$\%$, $\mathcal{C}$@5$\%$), suggesting that it is more reliable in safety-critical situations where opting for abstention is preferable to making incorrect predictions.

\section{More Qualitative Results}
\label{section:E}

\begin{figure}[ht]
    \centering
    \includegraphics[width=1.0\linewidth]{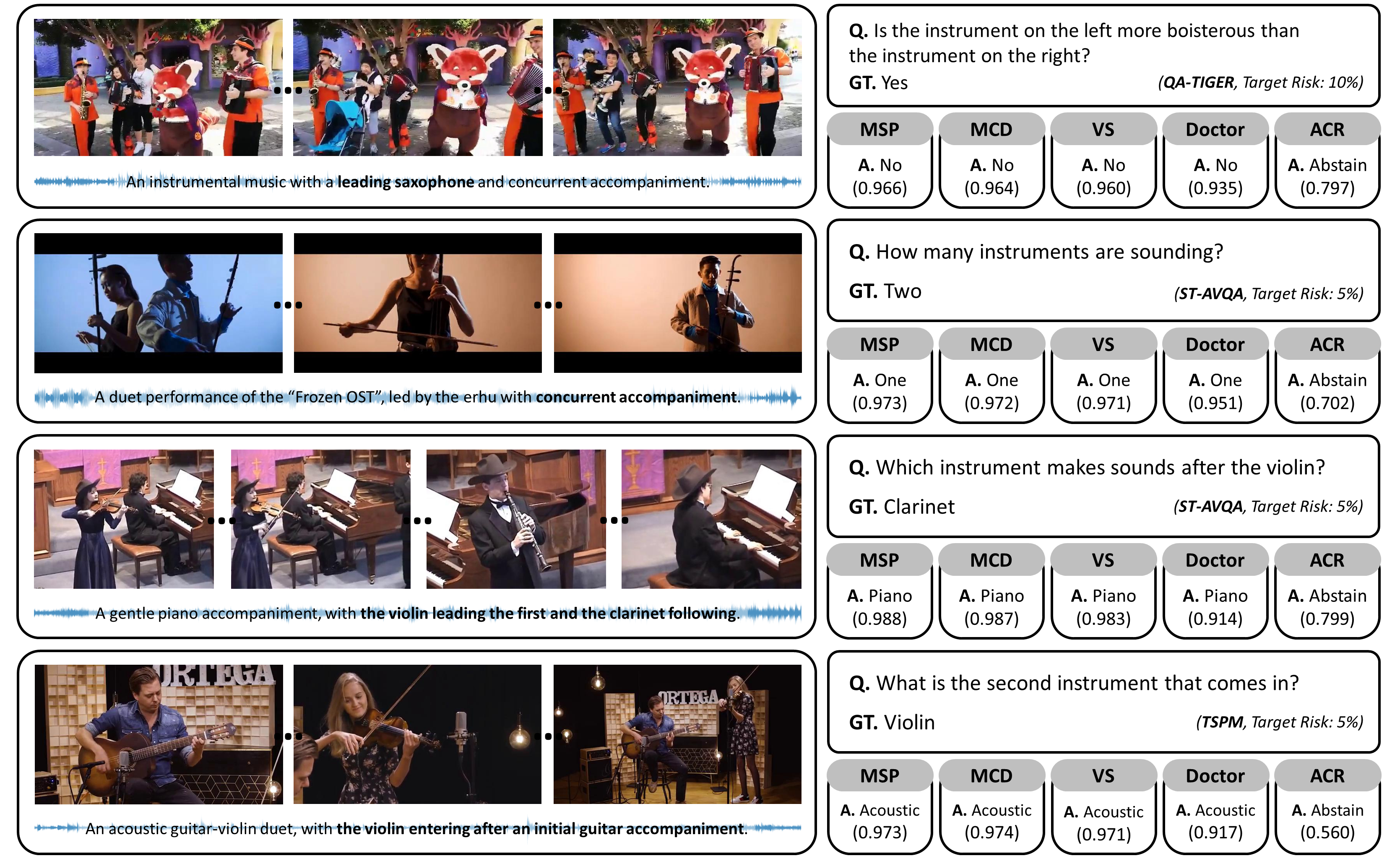}
    \caption{Additional qualitative examples of selective prediction based on various backbone models. For each model baseline (QA-TIGER, ST-AVQA, TSPM), the figure reports the corresponding selective prediction results under different selection strategies, as indicated in the figure.} 
    \label{fig:qualitative_result_appendix}
\end{figure}

Figure~\ref{fig:qualitative_result_appendix} presents additional qualitative examples across different AVQA backbones, including QA-TIGER, ST-AVQA, and TSPM, evaluated under multiple selection strategies.
These examples further illustrate a consistent failure mode of standard confidence-based baselines, which often assign high confidence to incorrect predictions in the presence of multimodal ambiguity.

Across all cases, MSP, MCD, and VS exhibit strong overconfidence, even when the questions involve subtle auditory comparisons (e.g., relative loudness), counting concurrent sound sources, or identifying temporally ordered events.
Such questions require fine-grained reasoning over audio-visual evidence and are particularly sensitive to cross-modal inconsistencies, such as visually salient but acoustically ambiguous cues or overlapping sound sources.

In contrast, ACR consistently assigns lower confidence scores and abstains when the evidence is insufficient or conflicting.
Notably, this behavior is observed regardless of the underlying backbone model, indicating that ACR captures uncertainty signals that are shared across diverse architectures.
These results further support that the primary benefit of ACR lies in reliable confidence estimation and selective abstention.

\section{Threshold Generalization}
\label{section:F}

As discussed in Sect. \ref{sec:benchmarking_metric}, we evaluate how well thresholds selected on the validation set for a specific target risk level transfer to the test split by measuring the resulting test risk. Using QA-TIGER as a representative model, we compare MSP with our adaptive confidence refinement mechanism, ACR. Tab. \ref{tab:test_time_generalization} shows that the resulting risk discrepancies are minimal, with differences of no more than $0.39\%$p for both selection functions on both MUSIC-AVQA and MUSIC-AVQA-v2.0 (train on balance set, test on bias set) datasets. Notably, the risk discrepancies are even smaller at $1\%$ target risk.

\begin{table}[h]
\caption{Generalization of abstention thresholds $\gamma$ from validation to testing using QA-TIGER. $\Delta \mathcal{R}$ and $\Delta \mathcal{C}$ represent the differences in risk and coverage percentage points (\%p) when employing $\gamma$ selected for the target risk $\mathcal{R}$ during validation versus $\gamma$ with maximum $\mathcal{C}@\mathcal{R}$.}
\label{tab:test_time_generalization}
\centering
\begin{tabular}{llccccccccc} 
\toprule
\multirow{2}{*}{Dataset}         & \multirow{2}{*}{Sel. Func. $g$} & \multirow{2}{*}{$\mathcal{R} =$~} & \multicolumn{4}{c}{$\Delta \mathcal{R}$}                                             & \multicolumn{4}{c}{$\Delta \mathcal{C}$}                                                                                                                                                          \\
                                 &                               &                       & 1\%   & 5\%                                       & 10\%  & 20\%  & 1\%                                       & 5\%                                       & 10\%                                      & 20\%                                       \\ 
\midrule
\multirow{2}{*}{MUSIC-AVQA}      & MSP                           &                       & +0.19 & +0.15                                     & +0.37 & +0.36 & +0.13                                     & +0.20                                     & +\textcolor[rgb]{0.122,0.122,0.122}{1.86} & +\textcolor[rgb]{0.122,0.122,0.122}{1.12}  \\
                                 & ACR (Ours)                    &                       & -0.06 & +0.29                                     & +0.39 & -0.04 & +0.28                                     & +0.25                                     & +1.78                                     & -0.24                                      \\ 
\midrule
\multirow{2}{*}{MUSIC-AVQA-v2.0} & MSP                           &                       & +0.08 & +0.24                                     & +0.34 & -0.24 & +0.18                                     & +\textcolor[rgb]{0.122,0.122,0.122}{2.69} & \textcolor[rgb]{0.122,0.122,0.122}{+2.74} & -0.53                                      \\
                                 & ACR (Ours)                    &                       & +0.07 & +0.\textcolor[rgb]{0.122,0.122,0.122}{25} & +0.33 & -0.38 & \textcolor[rgb]{0.122,0.122,0.122}{+0.24} & +\textcolor[rgb]{0.122,0.122,0.122}{1.38} & +2.03                                     & -0.66                                      \\
\bottomrule
\end{tabular}
\end{table}

We also observe differences in coverage achieved using the threshold selected from the validation set compared to the maximum achievable coverage at the same risk level, which we denote as $\Delta C$. These results suggest that risk-based thresholds generalize well across different splits, although they do not allow for a direct comparison of coverage at precisely the same risk level.

\section{Additional Data Split Details}
\label{section:G}

\begin{table}[h]
\caption{Question Statistics Comparison of MUSIC AVQA Datasets}
\label{tab:music_avqa_stats}
\centering
\renewcommand{\arraystretch}{1.3}
\begin{tabular}{lccc p{4.8cm}}
\toprule
Dataset & \#QA & \#Type & Key Features \\
\midrule
MUSIC-AVQA   & 45.8K & 33 & 9 question types \\
MUSIC-AVQA2  & 54K   & 33 & Balanced templates \\
MUSIC-AVQA-R & 211K  & 33 & $25\times$ rephrased; 465 vocab. \\
\bottomrule
\end{tabular}
\end{table}

We conduct experiments using the MUSIC-AVQA \cite{li2022learning}, MUSIC-AVQA-R \cite{ma2024look}, and MUSIC-AVQA-v2.0 \cite{liu2024tackling} datasets, all of which contain audio-visual question-answer pairs. Their question-answer statistics comparison is provided in Tab. \ref{tab:music_avqa_stats}. Tab. \ref{tab:music-avqa-split} and \ref{tab:music-avqa-v2-split} summarize the data splits utilized in our experiments.
For the MUSIC-AVQA and MUSIC-AVQA-v2.0 datasets, we further create splits of their original test sets since we need an additional set for validating the selection function \( g \), as well as for evaluating risk and coverage. These splits are designed to ensure that no videos and, therefore, no associated audio streams or question-answer annotations are shared across the different splits. This strategy is aligned with prior works \cite{whitehead2022reliable, dancette2023improving}

It is important to note that the held-out test set (referred to as ``Test" in Tab. \ref{tab:music-avqa-split} and \ref{tab:music-avqa-v2-split}) is never accessed during the training or validation phases of any components, including the base predictor \( f \) and the selection function \( g \). The MUSIC-AVQA-R dataset shares the same training and validation sets as the MUSIC-AVQA dataset; it only includes a new test set. Consequently, we will use the pretrained models from the MUSIC-AVQA dataset to evaluate the test sets, so there is no need to create a new split for this dataset. This test set is used exclusively for the final evaluation. Unless otherwise specified, all reported results are based on this held-out test set.

\begin{table}[h]
\caption{Data split Details for MUSIC-AVQA \cite{li2022learning} dataset}
\label{tab:music-avqa-split}
\centering
\renewcommand{\arraystretch}{1.3}
\begin{tabular}{lccc} 
\toprule
Source                           & Split Name & Usage                              & \% Source  \\ 
\midrule
MUSIC-AVQA train                 & Train      & Train $f$                            & 100\%      \\ 
\midrule
MUSIC-AVQA val                   & Val-$f$      & Validate $f$, Tran $g$                 & 100\%      \\ 
\midrule
\multirow{2}{*}{MUSIC-AVQA test} & Val-$g$      & Validate g, Evaluate risk-coverage & 20\%       \\
                                 & Test       & Held out test set for testing ($f$, $g$)   & 80\%       \\
\bottomrule
\end{tabular}
\end{table}

\begin{table}[h]
\caption{Data split Details for MUSIC-AVQA-v2.0 \cite{liu2024tackling} dataset}
\label{tab:music-avqa-v2-split}
\centering
\renewcommand{\arraystretch}{1.3}
\begin{tabular}{lccc} 
\toprule
Source                           & Split Name & Usage                              & \% Source  \\ 
\midrule
MUSIC-AVQA-v2.0 train                 & Train      & Train $f$                            & 100\%      \\ 
\midrule
MUSIC-AVQA-v2.0 val                   & Val-$f$      & Validate $f$, Tran $g$                 & 100\%      \\ 
\midrule
\multirow{2}{*}{MUSIC-AVQA-v2.0 test} & Val-$g$      & Validate g, Evaluate risk-coverage & 20\%       \\
                                 & Test       & Held out test set for testing ($f$, $g$)   & 80\%       \\
\bottomrule
\end{tabular}
\end{table}

\section{Additional Model Details}
\label{section:H}

In this section, we provides comprehensive details regarding the model architecture, hyperparameters and implementation which are used for our experiments.

\subsection{AVQA Models}

\begin{table}[t]
\centering
\caption{Hyperparameters of each model used in our experiments. All these hyperparameters are reported directly from their original works.}
\label{tab:hyperparameters}
\begin{tabular}{lccc} 
\toprule
Hyperparameters         & QA-TIGER \cite{kim2025question} & ST-AVQA \cite{li2022learning} & TSPM \cite{li2024boosting}  \\ 
\midrule
Batch Size              & 32        & 64       & 64     \\
Hidden Size             & 512        & 512       & 512     \\
Optimizer               & Adam \cite{adam2014method}       & Adam \cite{adam2014method}       & Adam \cite{adam2014method}     \\
Adam $\epsilon$            & 1$e$-8        & 1$e$-8       & 1$e$-8     \\
Adam $\beta_1$              & 0.95        & 0.90       & 0.90     \\
Adam $\beta_2$              & 0.999        & 0.999       & 0.999     \\
Learning Rate           & 1$e$-4        & 1$e$-4       & 1$e$-4     \\
Learning Rate Scheduler & StepLR        & StepLR       & StepLR     \\
Step Size                   & 8        & 8       & 8     \\
Dropout                 & 0.1        & 0.1 \& 0.5       & 0.1     \\
\# Epoch                   & 15        & 30       & 30     \\
\bottomrule
\end{tabular}
\end{table}

We strictly use the original open-source code for all AVQA models used in our experiments. In this work, we select three AVQA architectures (QA-TIGER\footnote{\textcolor{blue}{https://github.com/AIM-SKKU/QA-TIGER}} \cite{kim2025question}, ST-AVQA\footnote{\textcolor{blue}{https://github.com/GeWu-Lab/MUSIC-AVQA}} \cite{li2022learning}, and TSPM\footnote{\textcolor{blue}{https://github.com/GeWu-Lab/TSPM}} \cite{li2024boosting}) to benchmark the $\mathcal{R}$-AVQA benchmark and evaluate our proposed method. We select AVQA models for benchmarking with varied architectures and different multimodal fusion strategies that integrate audio, visual, and textual modalities. This variety enhances the robust evaluation of the effectiveness and generalization of all selection function, particularly our proposed ACR.
For training AVQA models, we use the hyperparameters from their original paper and GitHub repository, as listed in Tab. \ref{tab:hyperparameters}. All models treat AVQA models ($f$) as a classification task and are trained via a cross-entropy loss. We briefly discuss the models and settings used in our experiments.

\textbf{QA-TIGER} \cite{kim2025question}. QA-TIGER is a state-of-the-art method for the AVQA task and has achieved Highlight paper recognition from CVPR 2025. This serves as a general and solid baseline for the AVQA task, employing early and progressive question embedding along with modeling continuous temporal dynamics through a Mixture of Experts (MoE) \cite{shazeer2017outrageously} architecture. Audio embeddings are extracted using a pretrained VGGish model \cite{hershey2017cnn} pretrained on AudioSet \cite{gemmeke2017audio} to generate frame-level features from 1-second segments, while visual embeddings employ CLIP \cite{radford2021learning} for [CLS] token-based frame-level representations. Text embeddings for questions are derived from CLIP's text encoder, producing sentence-level [EOT] tokens and word-level features. We train this model with original training sets from MUSIC-AVQA and MUSIC-AVQA-v2.0.

\textbf{ST-AVQA} \cite{li2022learning}. ST-AVQA is one of the early works in the AVQA task and was selected for oral presentation at CVPR 2022. ST-AVQA adopts late question integration with a two-stage architecture of hierarchical spatial and temporal grounding modules for improved AVQA accuracy. Audio embeddings are derived from 1-second segments using a pre-trained VGGish model, producing 512-dimensional features via a linear projection from the original 128-dimensional outputs. Visual embeddings are extracted from sampled frames (at 1 fps) per segment employing a pre-trained ResNet-18, yielding 512-dimensional feature maps that capture spatial details. Text embeddings for questions are generated through an LSTM network \cite{graves2012long} on 512-dimensional projected word embeddings, with the final hidden state serving as the question representation, which is injected as a query in the temporal module and element-wise multiplied in the fusion stage to guide attention and enhance cross-modal integration. We train this model on similar datasets as QA-TIGER.

\textbf{TSPM} \cite{li2024boosting}. TSPM is a recent AVQA model that achieves promising performance on several datasets, introducing a novel approach to AVQA research. TSPM leverages intermediate, question-guided progressive perception with powerful prompt grounding. Audio embeddings are derived from 1-second segments using a pre-trained VGGish model, producing 128-dimensional features per segment. Visual embeddings are extracted with a frozen CLIP (ViT-L/14) model, generating 512-dimensional frame-level [CLS] tokens from 1 fps-sampled frames and patch-level token features, which are merged to reduce redundancy while retaining key semantics. Text embeddings for questions are processed through CLIP's text encoder on tokenized words, yielding 512-dimensional sentence-level features.

\subsection{Selection Functions}

We detail the vector scaling (VS) and our Adaptive Confidence Refinement (ACR) selection functions here. We do not cover maximum softmax probability (MSP), Monte Carlo Dropout (MCD), and Doctor, as no additional training is required. While training each selection function, we freeze the AVQA model's weights $f$.

\textbf{Calibration}. The inputs to the calibration are the unnormalized answer logits (i.e., pre-softmax or answer representation just before the softmax layer) of the AVQA model, and the outputs are the calibrated logits. Since we use vector scaling \cite{platt1999probabilistic, guo2017calibration}, we feed the AVQA model's logits into a linear layer with a diagonal weight matrix and a bias term. During training, after the linear layer, we apply a sigmoid activation and, in contrast to \cite{guo2017calibration}, use the resulting activations as input to a binary cross-entropy loss with the AVQA labels. We train the linear layer with the Adam optimizer, using a learning rate of 8$e$-5. The hidden dimension is set to 512. We maintain the same hyperparameters as the original AVQA backbone, as shown in Tab. \ref{tab:hyperparameters}. At test time, we use the output of this linear layer as the calibrated logits, apply softmax, and apply the same abstention procedure as MSP.

\textbf{Adaptive Confidence Refinement}. The inputs to our Adaptive Confidence Refinement (ACR) consist of intermediate multimodal representations (i.e., the inputs to the classifier head of the base AVQA models) and unnormalized answer logits (i.e., the answer representation immediately before the softmax layer). The architectural details, forward pass and stage 2 training of our ACR, which features two learned heads (the Residual Risk Head and Confidence Gating Head), are presented in Tab. \ref{tab:acr_head_arch} and Algorithm \ref{alg:acr_forward}, \ref{alg:acr_training}. We provide a high-level PyTorch-style pseudocode for Adaptive Confidence Refinement (ACR).
The backbone AVQA model is frozen during confidence refinement.
For reproducibility, we have also included our actual PyTorch implementation in the supplementary material. 
To ensure simplicity and efficiency, we utilize the same three-layer multi-layer perceptron (MLP) architecture for both heads. We begin training these two heads jointly in Stage 2 using the Adam optimizer with a learning rate of 8e-5. The hidden dimension is set to $d' = $ 512 (the intermediate multimodal features dimension) + 42 (the logits dimension). We retain the same hyperparameters as those of the original AVQA backbone, as shown in Table \ref{tab:hyperparameters}. During testing, we calculate the confidence score of our ACR based on the formulation presented in Equation \ref{eq:final_conf}.

\begin{table}[H]
\centering
\caption{Architecture of Residual Risk Head (RRH) and Confidence Gating Head (CGH). $\mathbf{h}$ and $\mathbf{z}$ are the fused multimodal representations and answer logits, respectively.}
\label{tab:acr_head_arch}
\begin{tabular}{l c}
\toprule
Layer & Configuration \\
\midrule
Input & [$d$-dim feature $\mathbf{h}$, 42-dim feature $\mathbf{z}$] \\
FC-1 & Linear($d' \rightarrow d'$) + ReLU + Dropout($p{=}0.1$) \\
FC-2 & Linear($d' \rightarrow d'$) + ReLU + Dropout($p{=}0.1$) \\
FC-3 (Output) & Linear($d' \rightarrow 1$) + Sigmoid \\
\bottomrule
\end{tabular}
\end{table}

\begin{algorithm}[H]
\caption{Adaptive Confidence Refinement (ACR): Forward Pass}
\label{alg:acr_forward}
\begin{algorithmic}[1]
\REQUIRE Input sample $\boldsymbol{x}=(a,v,q)$
\REQUIRE Frozen AVQA backbone $f(\cdot)$
\REQUIRE Residual Risk Head $\mathrm{RRH}(\cdot)$
\REQUIRE Confidence Gating Head $\mathrm{CGH}(\cdot)$

\STATE Obtain answer logits and fused representation:
\[
(\mathbf{z}, \mathbf{h}) \leftarrow f(\boldsymbol{x})
\]

\STATE Compute MSP confidence:
\[
C_\mathrm{M}(\boldsymbol{x}) \leftarrow \max_k \mathrm{softmax}(\mathbf{z})_k
\]

\STATE Predict residual risk confidence:
\[
C_\mathrm{R}(\boldsymbol{x}) \leftarrow \mathrm{RRH}(\mathbf{h})
\]

\STATE Predict confidence gate:
\[
\alpha(\boldsymbol{x}) \leftarrow \mathrm{CGH}(\mathbf{h})
\]

\STATE Fuse confidences:
\[
C_{\mathrm{ACR}}(\boldsymbol{x}) \leftarrow
\alpha(\boldsymbol{x})\,C_\mathrm{M}(\boldsymbol{x}) + (1-\alpha(\boldsymbol{x}))\,C_\mathrm{R}(\boldsymbol{x})
\]

\STATE \textbf{return} $C_{\mathrm{ACR}}(\boldsymbol{x})$
\end{algorithmic}
\end{algorithm}

\begin{algorithm}[H]
\caption{Training ACR with Binary Cross-Entropy}
\label{alg:acr_training}
\begin{algorithmic}[1]
\REQUIRE Training set $\{(x_i,y_i)\}_{i=1}^N$
\REQUIRE Frozen AVQA backbone $f(\cdot)$
\REQUIRE RRH and CGH parameters $\theta_\mathrm{R}, \theta_\mathrm{G}$
\REQUIRE Binary cross-entropy loss $\mathcal{L}_{\mathrm{BCE}}$

\FOR{each minibatch $\mathcal{B}$}
    \STATE Compute logits and representations:
    \[
    (\mathbf{z}_i, \mathbf{h}_i) \leftarrow f(x_i),
    \quad \forall x_i \in \mathcal{B}
    \]

    \STATE Compute MSP confidence:
    \[
    C_\mathrm{M}(x_i) \leftarrow \max_k \mathrm{softmax}(\mathbf{z}_i)_k
    \]

    \STATE Predict residual confidence and gate:
    \[
    C_\mathrm{R}(x_i) \leftarrow \mathrm{RRH}(\mathbf{h}_i), \quad
    \alpha(x_i) \leftarrow \mathrm{CGH}(\mathbf{h}_i)
    \]

    \STATE Fuse confidence:
    \[
    C_{\mathrm{ACR}}(x_i) \leftarrow
    \alpha(x_i)C_\mathrm{M}(x_i) + (1-\alpha(x_i))C_\mathrm{R}(x_i)
    \]

    \STATE Compute correctness label:
    \[
    c(x_i) \leftarrow \mathbb{I}[\hat y_i = y_i]
    \]

    \STATE Update parameters by minimizing:
    \[
    \mathcal{L}_{\mathrm{BCE}}(C_{\mathrm{ACR}}(x_i), c(x_i))
    \]
\ENDFOR
\end{algorithmic}
\end{algorithm}

\section{Standard Deviation for TSPM}
\label{section:H_1}

In this section, to assess result stability, we present the mean and standard deviation of TSPM on the MUSIC-AVQA dataset. These results slightly differ from those reported in Tab. \ref{tab:car_music_avqa} in the main paper, which were based on a single random seed. Tab. \ref{tab:std_mean} reports performance across five random seeds, showing low variance across selection functions (standard deviation from 0.04 to 0.05). Given the computational cost of multi-seed evaluation, we verified that QA-TIGER and ST-AVQA exhibit similar stability across different seeds, so we report single-run results for these architectures as in Tab. \ref{tab:car_music_avqa}. This ensures our conclusions are not artifacts of random initialization.

\begin{table}[h]
\centering
\setlength{\tabcolsep}{3pt}
\caption{Mean and standard deviations for risk-coverage, AURC, and ECE metrics for different selection functions with TSPM backbone on MUSIC-AVQA dataset. All in \%.}
\label{tab:std_mean}
\scalebox{1.0}{
\begin{tabular}{llcccccc}
\toprule
Model $f$ & Sel. Func. $g$ & $\mathcal{C}$@1\% & $\mathcal{C}$@5\% & $\mathcal{C}$@10\% & $\mathcal{C}$@20\% & AURC & ECE \\
\midrule
\multirow{5}{*}{TSPM}     & MSP                   & 8.60 $\pm$ 0.66     & 31.06 $\pm$ 0.69       & 49.64 $\pm$ 0.27     & 87.56 $\pm$ 0.30       & 10.46 $\pm$ 0.05     & 4.45 $\pm$ 0.23     \\
                          & MCD                       & 8.77 $\pm$ 0.29      & 29.90 $\pm$ 0.46      & 50.25 $\pm$ 0.48       & 87.30 $\pm$ 0.54       & 10.47 $\pm$ 0.05     & 4.20 $\pm$ 0.44     \\
                          & VS            & 6.70 $\pm$ 0.34      & 30.13 $\pm$ 0.69      & 50.42 $\pm$ 0.36       & 87.54 $\pm$ 0.35       & 10.46 $\pm$ 0.05     & 4.02 $\pm$ 0.43     \\
                          & Doctor                   & 8.63 $\pm$ 0.57     & 31.02 $\pm$ 0.61       & 49.57 $\pm$ 0.33     & 87.64 $\pm$ 0.31       & 10.47 $\pm$ 0.05     & 2.19 $\pm$ 0.20     \\ \hdashline
                          & ACR (Ours) & 8.92 $\pm$ 0.48      & 35.53 $\pm$ 0.38      & 54.22 $\pm$ 0.50       & 88.67 $\pm$ 0.34       & 9.74 $\pm$ 0.04     & 1.64 $\pm$ 0.11     \\
\bottomrule
\end{tabular}}
\vspace{-2mm}
\end{table}

\section{Relation to Calibration and Conformal Prediction }
\label{section:I}

\textbf{Relation to Calibration}. Unlike calibration methods \cite{platt1999probabilistic, zadrozny2002transforming, guo2017calibration}, which learn monotonic transformations of logits to minimize the Expected Calibration Error (ECE), Adaptive Confidence Ranking (ACR) directly optimizes selective prediction by training a discriminative selector on intermediate features. Calibration focuses on probability matching, ensuring that samples with a given confidence level are correct with a proportional fraction of the time. In contrast, selective prediction requires optimal ranking, where correct predictions are assigned higher confidence than incorrect ones. These two objectives can diverge: a well-calibrated model may still rank samples suboptimally, whereas ACR can reorder samples by leveraging richer feature representations beyond the output layer.

\textbf{Relation to Conformal Prediction}. Our Adaptive Confidence Refinement (ACR) framework is based on a foundational principle shared with conformal prediction \cite{vovk2005algorithmic, angelopoulos2023conformal}: both methods use a held-out calibration set to improve potentially miscalibrated model outputs and yield more reliable uncertainty estimates. In conformal prediction, nonconformity scores computed from the calibration set are used to construct prediction sets that provide formal coverage guarantees (e.g., the true label is included in the prediction set with probability at least $1 - \alpha$). Similarly, ACR employs a held-out selection set \( D_{\text{select}} \) to train the selector head, ensuring that it learns from the natural error patterns of the base model rather than simply memorizing training examples. 
However, the two approaches differ in their objectives and guarantees. While conformal prediction provides distribution-free coverage guarantees for prediction sets, ACR focuses on learning a discriminative confidence function tailored for ranking-based selective prediction to maximize coverage at a specific risk level. Our approach can be seen as a learned, task-specific alternative that sacrifices formal guarantees in exchange for improved empirical performance in selective prediction. Importantly, ACR complements conformal methods: the refined confidence scores generated by our selector could serve as enhanced nonconformity scores within a conformal framework, offering an intriguing avenue for future research.

\section{Additional Discussion and Future Work}
\label{section:I_1}

Selective prediction aims to enhance the reliability of machine learning systems by allowing models to abstain from making predictions when they are likely to be incorrect. Our work contributes to this goal by exploring confidence estimation amid multimodal uncertainty and proposing an input-adaptive confidence fusion strategy for audio-visual question answering. To the best of our knowledge, we are the first to formalize the \textit{Reliable} Audio-Visual Question Answering problem. By addressing misalignment among modalities heterogeneous confidence signals, our approach improves the trade-offs between selective risk and coverage without altering the underlying task predictor. This is particularly beneficial in safety-critical or user-facing multimodal AI applications, such as assistive systems, educational tools, and multimedia retrieval, where inaccurate answers can mislead users and undermine their trust.

However, selective prediction systems may create information access imbalances if abstentions disproportionately affect certain input types or user groups, especially in multimodal contexts where data quality can vary across modalities. While our method does not eliminate these risks, it aims to reduce arbitrary or overconfident failures by improving confidence ranking in the face of modality misalignment or disagreement. We emphasize the importance of transparency regarding abstention behavior and failure modes in the deployment of selective models. Overall, this work advances the understanding of reliability in multimodal learning and encourages further research on principled uncertainty estimation and the responsible deployment of selective prediction systems.

The current formulation of Adaptive Confidence Refinement (ACR) is designed for classification-style AVQA models with a fixed and finite answer space. Both the Maximum Softmax Probability (MSP) and the Residual Risk Head rely on logits defined over a closed set of candidate answers. However, ACR is implemented for discrete answer selection; however, its fundamental principle, which is input-adaptive confidence refinement, is not limited to classification tasks. A promising avenue for future exploration is extending ACR to support free-form generation. This could involve applying residual confidence correction at the token or sequence level, utilizing generation probabilities, entropy, or self-consistency signals. We will address a thorough investigation of these extensions in future work.

Although ACR substantially improves the separation between correct and incorrect predictions, our analysis in Appendix~\ref{section:B_1} shows that a subset of incorrect predictions still receive high confidence scores. These failure cases often arise from complex multimodal ambiguities, such as conflicting audio-visual cues or temporally subtle requirements, where both the base AVQA model and the selector struggle to detect unreliability. 

An important direction for future work is extending Adaptive Confidence Refinement beyond closed-set AVQA to multimodal large language models (MLLMs) with open-ended generation.
In such settings, confidence estimation cannot rely on a fixed answer space and must instead account for uncertainty in reasoning chains, the confidence at the token or sequence level, and self-consistency signals across generated outputs.




\section{Broader Impact}
\label{section:J}

Reliable uncertainty estimation is crucial for deploying trustworthy AI, particularly in safety-critical applications, where incorrect predictions can have severe consequences. Our Adaptive Confidence Refinement (ACR) framework helps achieve this by enabling models to more accurately identify when their predictions are likely to be incorrect, thereby enabling informed abstention or deferral to human judgment.

In the field of Audio-Visual Question Answering, enhanced selective prediction could greatly benefit accessibility applications, such as assistive systems for users who are visually or hearing-impaired. This ensures that the system provides responses only when it is confident, rather than risking the delivery of misleading information.

The principle of learning to abstain is applicable to various high-stakes domains, including medical diagnosis, autonomous navigation, and content moderation, where the cost of a wrong prediction is significantly greater than that of deferring to human expertise. 
However, we recognize potential risks in this approach. Practitioners may overreli on confidence scores without fully understanding their limitations, and confidence estimators can perform inconsistently across demographic or content subgroups, which could lead to unequal rates of abstention.

We urge practitioners to assess ACR's behavior across various subpopulations before deployment and to consider confidence scores as one factor among many in decision-making, rather than a definitive measure of correctness. Our work does not introduce new capabilities for generating harmful content; instead, it provides tools to make existing AI systems more reliable and aware of their limitations.

\end{document}